\documentclass{article}

\usepackage{microtype}
\usepackage{graphicx}
\usepackage{subfigure}
\usepackage{booktabs} 

\usepackage{hyperref}

\usepackage{url}

\usepackage[utf8]{inputenc} 
\usepackage[T1]{fontenc}    
\usepackage{hyperref}       
\usepackage{url}            
\usepackage{booktabs}       
\usepackage{amsfonts}       
\usepackage{nicefrac}       
\usepackage{microtype}      
\usepackage{xcolor}         

\usepackage{framed}
\usepackage{color}
\usepackage{booktabs}       
\usepackage{amsfonts}       
\usepackage{nicefrac}       
\usepackage{microtype}      
\usepackage{xcolor}         
\usepackage{color, colortbl}
\definecolor{greyC}{RGB}{180,180,180}
\definecolor{greyL}{RGB}{235,235,235}
\usepackage{multicol}
\usepackage{multirow}
\usepackage{makecell}

\usepackage{amsmath}
\usepackage{amssymb}
\usepackage{mathtools}
\usepackage{amsthm}

\usepackage{soul}
\usepackage{microtype}
\usepackage{graphicx}
\usepackage{booktabs}
\usepackage{caption}
\usepackage{threeparttable}
\usepackage{subfigure}
\usepackage{dsfont}
\usepackage{enumerate}
\usepackage{amsmath,amsthm,amssymb}
\usepackage{natbib}

\usepackage{wrapfig}

\usepackage{framed}
\usepackage{color}
\definecolor{shadecolor}{rgb}{0.92,0.92,0.92}

\newcommand{\cX}{\mathcal{X}}
\newcommand{\cY}{\mathcal{Y}}

\newcommand{\bx}{{x}}

\newcommand{\bxtidle}{\tilde{{x}}}

\newcommand{\xadv}{\tilde{{x}}}

\usepackage[accepted]{icml2023}

\usepackage{amsmath}
\usepackage{amssymb}
\usepackage{mathtools}
\usepackage{amsthm}

\usepackage[capitalize,noabbrev]{cleveref}

\theoremstyle{plain}
\newtheorem{theorem}{Theorem}[section]

\theoremstyle{definition}
\newtheorem{definition}[theorem]{Definition}
\newtheorem{assumption}[theorem]{Assumption}
\newtheorem{property}[theorem]{Property}
\theoremstyle{remark}

\usepackage[textsize=tiny]{todonotes}

\icmltitlerunning{Exploring Model Dynamics for Accumulative Poisoning Discovery}

\begin{document}

\twocolumn[
\icmltitle{Exploring Model Dynamics for Accumulative Poisoning Discovery}



\icmlsetsymbol{equal}{*}

\begin{icmlauthorlist}
\icmlauthor{Jianing Zhu}{hkbu}
\icmlauthor{Xiawei Guo}{ali}
\icmlauthor{Jiangchao Yao}{sjtu,lab}
\icmlauthor{Chao Du}{ali}
\icmlauthor{Li He}{ali}
\icmlauthor{Shuo Yuan}{ali}\\
\icmlauthor{Tongliang Liu}{syd,sydc}
\icmlauthor{Liang Wang}{ali}
\icmlauthor{Bo Han}{hkbu}
\end{icmlauthorlist}

\icmlaffiliation{hkbu}{Department of Computer Science, Hong Kong Baptist University}
\icmlaffiliation{ali}{Alibaba Group}
\icmlaffiliation{sjtu}{CMIC, Shanghai Jiao Tong University}
\icmlaffiliation{lab}{Shanghai AI Laboratory}
\icmlaffiliation{syd}{Mohamed bin Zayed University of Artificial Intelligence}
\icmlaffiliation{sydc}{Sydney AI Centre, The University of Sydney}

\icmlcorrespondingauthor{Bo Han}{bhanml@comp.hkbu.edu.hk}
\icmlcorrespondingauthor{Jiangchao Yao}{Sunarker@sjtu.edu.cn}

\icmlkeywords{Machine Learning, ICML}

\vskip 0.3in
]



\printAffiliationsAndNotice{}  

\begin{abstract}
Adversarial poisoning attacks pose huge threats to various machine learning applications. Especially, the recent accumulative poisoning attacks show that it is possible to achieve irreparable harm on models via a sequence of imperceptible attacks followed by a trigger batch. Due to the limited data-level discrepancy in real-time data streaming, current defensive methods are indiscriminate in handling the poison and clean samples. In this paper, we dive into the perspective of model dynamics and propose a novel information measure, namely, \textit{Memorization Discrepancy}, to explore the defense via the model-level information. By implicitly transferring the changes in the data manipulation to that in the model outputs, Memorization Discrepancy can discover the imperceptible poison samples based on their distinct dynamics from the clean samples. We thoroughly explore its properties and propose Discrepancy-aware Sample Correction (DSC) to defend against accumulative poisoning attacks. Extensive experiments comprehensively characterized Memorization Discrepancy and verified its effectiveness. The code is publicly available at: \url{https://github.com/tmlr-group/Memorization-Discrepancy}.
\end{abstract}

\section{Introduction}
\label{sec:intro}

Machine learning models have achieved remarkable performance on a wide range of tasks in computer vision~\citep{he2016deep} and natural language processing~\citep{devlin-etal-2019-bert}. However, due to the lack of strict supervision in crowd-sourcing~\citep{NIPS2010_0f9cafd0}, data from untrusted sources poses huge threats to machine learning services~\citep{biggio2012poisoning, Goodfellow14_Adversarial_examples}. Specifically, some malicious adversaries~\citep{paudice2018label,goldblum2022dataset} hidden in training data can significantly deteriorate the model performance~\citep{feng2019learning, huang2020unlearnable, tao2021better, fowl2021adversarial}, causing concerns~\citep{bommasani2021opportunities} in those safety-critical applications like autonomous driving or medical intelligence.

\begin{figure}[t]
    \centering
    \hspace{-0.05in}
    \includegraphics[scale=0.145]{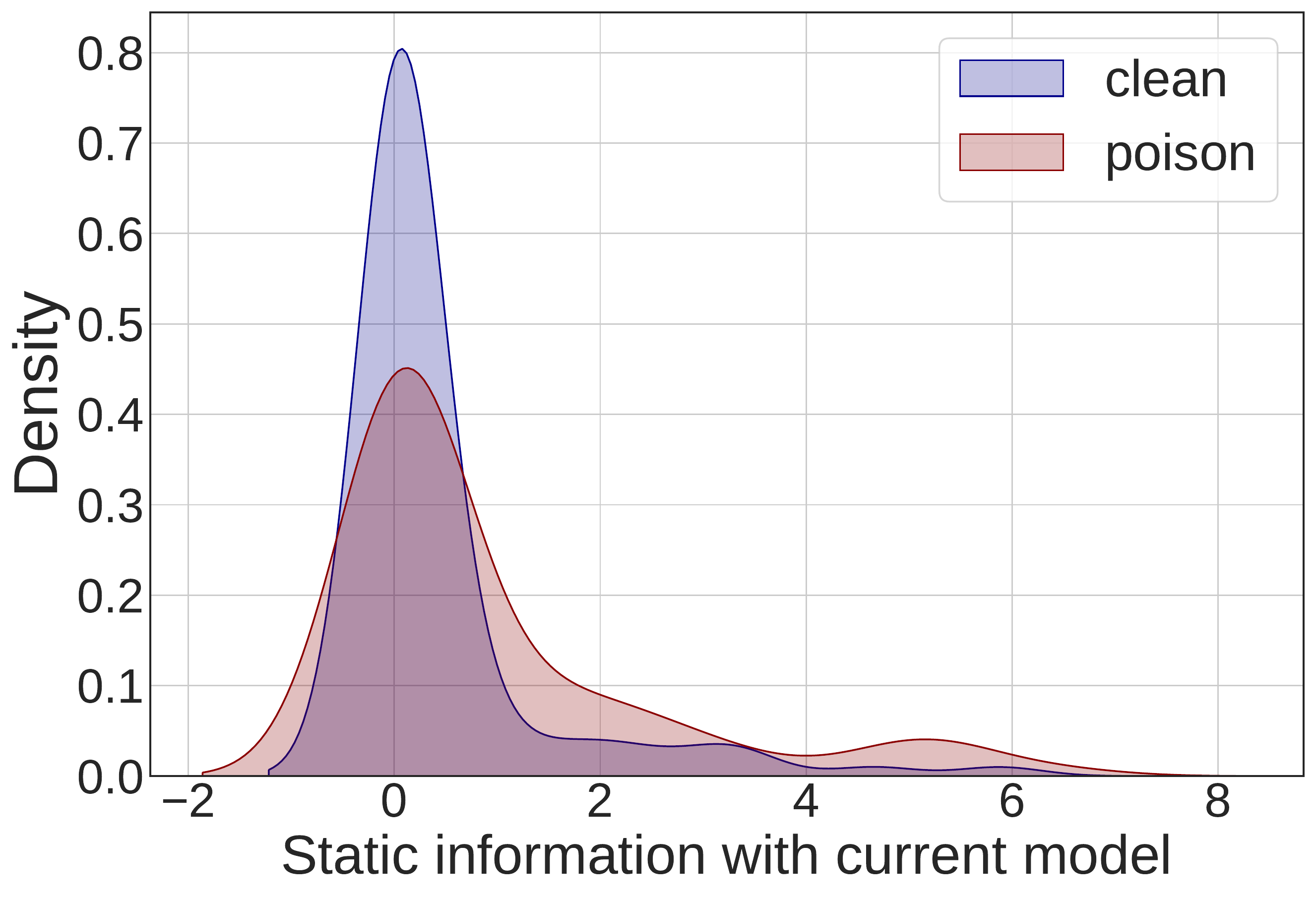}
    \hspace{0.01in}
    \includegraphics[scale=0.145]{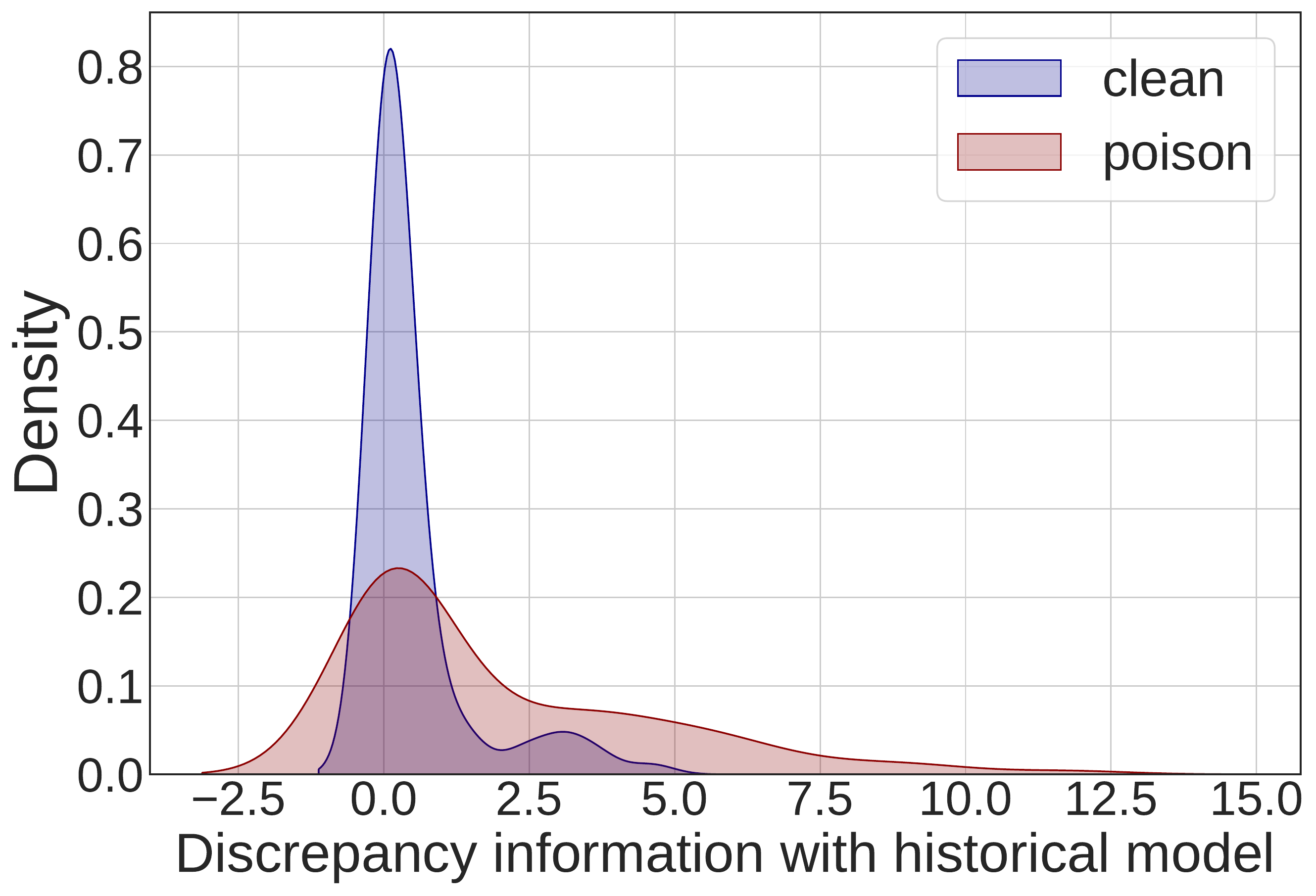} 
    \vspace{-2mm}
    \caption{
    Left: Comparison of the distributions using static information (i.e., the output of the current model); Right: Comparison of the distributions using the discrepancy information (i.e., the output discrepancy of current and historical models). The experiment simulates the accumulative poisoning attack~\citet{pang2021accumulative} in real-time data streaming using CIFAR-10 dataset. The generated poison samples can be better distinguished from clean samples by the discrepancy information, i.e., Memorization Discrepancy. Here the static information is also about the output of the model but is defined as the output difference before and after the model optimized on 1 epoch of data. Considering the interval can be nearly ignored compared with the historical model (before 20 epochs), so termed "static". The detailed operation is illustrated in Figure~\ref{fig:reason}.
    }
    \label{fig:motivation}
    \vspace{-4mm}
\end{figure}

Different from previous well-explored attacks under the offline setting~\citep{li2016data,fowl2021adversarial,goldblum2022dataset}, accumulative poisoning attacks~\citep{pang2021accumulative} are recently proposed and demonstrated to be more imperceptible in real-time data streaming~\citep{wang2018data,zhang2020online}. Employing the newly introduced accumulative batches for pre-poisoning, it will not cause significant harm to the model during the first phase but leverages the trigger batch to induce dramatic degradation of the model performance instantly.
Considering the imperceptibility and the limited knowledge about
accumulative poisoning samples, previous works~\citep{feinman2017detecting,steinhardt2017certified,ma2018characterizing} that depend on the offline data statistics cannot sufficiently handle this type of sneaky adversary. It naturally raises a new challenge: \textit{how can we identify and defend against the imperceptible accumulative poisoning attacks in real-time data streaming?}

Currently, the most possible ways to defend against accumulative poisoning attacks are gradient clipping~\citep{pascanu2013difficulty} and the variants of adversarial training~\citep{tao2021better,geiping2021doesn}, which have both pros and cons. Specifically, although gradient clipping~\citep{pascanu2013difficulty} shows promise to mitigate the poisoning effect, it still can be deceived by samples with small gradient norms in the accumulative phase and has a side-effect on slowing down the training convergence~\citep{pang2021accumulative}. As for adversarial training methods~\citep{Madry_adversarial_training,Zhang_trades}, it has been demonstrated that the natural risk of training with poison samples can be upper bounded by the adversarial risk~\citep{tao2021better}. Therefore, it is natural to adopt the reverse adversarial generation to correct the newly captured samples. Unfortunately, the indiscriminate sample calibration in adversarial training when applying to clean samples is detrimental~\citep{Zhang_trades} to performance (e.g., as illustrated in Figure~\ref{fig:method}) due to the over-correction.

In this paper, we introduce a new measure, termed as \textit{Memorization Discrepancy} (i.e., Eq.~(\ref{eq:memorization_discrepancy}) in Section~\ref{sec:memorization_discrepancy}), which is surprisingly aware of the imperceptible accumulative poisoning samples by backtracking earlier historical model (e.g., as illustrated in Figure~\ref{fig:motivation}). Diving into the model dynamics, we compute the discrepancy by leveraging the historical model's output on the same sample. It can be found in Figure~\ref{fig:reason} that with the increase in the backtracking intervals, poison samples can be more distinguishable from clean samples. The underlying mechanism is to transfer imperceptible manipulation into significant model-level changes (as further explained in Figure~\ref{fig:reason_2}). 
Then, some observed properties (i.e., Properties~\ref{pro:1} and~\ref{pro:2}) like monotonically increasing and the existence of highly discriminative backtracking interval can be used to handle poisoning discovery for the sneaky adversary, which show promising in identifying poison samples with imperceptible constraint from clean samples or other natural samples with distribution shift. 


Based on the above insights, we accordingly design a new defense algorithm, namely, \textit{Discrepancy-aware Sample Correction} (DSC), which incorporates Memorization Discrepancy to selectively calibrate the potential poison samples in real-time data streaming. At the high level, we relax the inner-minimization of reverse adversarial generation (i.e., Eq.~(\ref{eq:objective}) in Section~\ref{sec:proposed_method}) and construct a learning filter capable of calibrating oriented poison samples (as shown in Figure~\ref{fig:method}) to avoid over-calibration. In detail, our DSC employs the early-stopping in sample correction and utilizes the historical model to be an auxiliary inspector for Memorization Discrepancy. Our main contributions are summarized as,
\begin{itemize}
    
    \item We make the first effort to explore identifying the accumulative poisoning attack for the real-time data streaming from the perspective of model dynamics, i.e., considering model changes in poison discovery. 
    
    \item We introduce a novel information measure, i.e., Memorization Discrepancy, to distinguish the imperceptible poison samples by leveraging model-level information from backtracking the historical models. (in Section~\ref{sec:memorization_discrepancy})

    \item We accordingly propose a new learning method, i.e., Discrepancy-aware Sample Correction (DSC), which incorporates the proposed Memorization Discrepancy to selectively calibrate the potential poison samples with only a historical auxiliary model. (in Section~\ref{sec:proposed_method})
    
    \item We conduct extensive experiments to comprehensively characterize the Memorization Discrepancy, and verify the effectiveness of DSC in improving the model robustness against accumulative poisoning attacks using a range of benchmarked datasets. (in Sections~\ref{sec:exp})

\end{itemize}

\section{Backgrounds}
\label{sec:background}

In this section, we briefly review the background of delusive attack and accumulative poisoning attack~\citep{pang2021accumulative}, and discuss some existing defense methods.

\subsection{Delusive Attack}
\label{sec:background_delusive}

Delusive attack~\citep{newsome2006paragraph,feng2019learning} belongs to data poisoning attacks~\citep{barreno2010security,biggio2012poisoning,goldblum2022dataset}, which aim to degrade the model performance via manipulating the training data. The general malicious objective can be formulated as,
\begin{align}
\label{eq:mal_1}
\begin{split}
    \max_{\mathcal{P}}\mathcal{L}(S_{val};\theta^*), s.t.\; \theta^*\in \arg\min_{\theta}\mathcal{L}(\mathcal{P}(S_{train});\theta),
\end{split}
\end{align}
where $S_{train}$ is the training set consisting of natural examples, $S_{val}$ is the validation set, $\mathcal{P}(\cdot)$ denotes the transformation that manipulates $S_{train}$ into a poisoned version and $\mathcal{L}(S;\theta)$ denotes the empirical learning objective of a dataset $S=\{x_i,y_i\}^N_{i=1}$ with the model parameter $\theta$. Specifically, delusive attack targets to deteriorate the overall accuracy of the test data by only manipulating the input feature of the training data~\citep{newsome2006paragraph,barreno2010security,feng2019learning}, instead of attacking the specific class~\citep{koh2017understanding} or triggering the backdoors~\cite{shafahi2018poison}. Generally, the delusive attack can be formulated as the optimization problem through the gradient-based methods (e.g., Project Gradient Decent (PGD)~\citep{Madry_adversarial_training}), and limits the manipulation into a small constraint (e.g., $\ell_\infty$-norm adopted in adversarial attack~\citep{Goodfellow14_Adversarial_examples,kumar2020adversarial}).  

\subsection{Accumulative Poisoning Attack}
\label{sec:background_accumulative}

Different from previous studies which focus on poisoning offline datasets~\citep{feng2019learning, fowl2021adversarial,tao2021better}, \citet{pang2021accumulative} recently proposed the accumulative poisoning attack for the real-time data stream to simulate the poisoning on the online settings~\citep{chechik2010large}. The major difference between this attack from the ordinary delusive attack is that it can interact with the training process and dynamically manipulate the data according to the model status. Through this, it spreads the poisoning effect over multiple learning statuses to further avoid distinct modifications on clean samples. The certain objective for accumulative poisoning attack can be formulated as,
\begin{align}
\label{eq:mal_2}
\min_{\mathcal{P},\mathcal{A}}\nabla_{\theta}\mathcal{L}(S_{val};\mathcal{A}(\theta^T))^\top\nabla_{\theta}\mathcal{L}(\mathcal{P}(S_T);\mathcal{A}(\theta^T)),
\end{align}
where $\mathcal{A}$ denotes an accumulative phase to inject secrete poison samples, $\mathcal{A}(\theta^T)$ denotes the model parameter at round $T$ obtained after the accumulative phase and $\nabla_\theta$ denotes the gradient. Specifically, the whole process can be divided into two parts given a pre-trained burn-in model for several epochs on the data stream. First, the model will be secretly poisoned by the samples in the accumulative phase $\mathcal{A}$, while keeping test accuracy in a heuristically reasonable range of variation. Then a trigger batch $\mathcal{P}(S_T)$ will be fed into the model. By jointly optimizing the accumulative phase and the trigger batch $\mathcal{P}(S_T)$, the accumulative poisoning attack can result in a severe drop in the model performance in a single step (e.g., one batch). More details about the accumulative poisoning attacks can be referred to in Appendix~\ref{app:poison_details}.


\subsection{Existing Defenses}
\label{sec:background_defense}

To combat data poisoning, there are many strategies proposed for defending against poisoning attacks, like detection-based methods~\citep{steinhardt2017certified, collinge2019defending} to find and filter the poison data according to the feature statistics,  robust training methods~\citep{borgnia2021strong,li2021anti} that is designed for targeted or backdoor attacks. Considering the characteristic of the real-time data streaming and the imperceptibility of delusive attack, it is computationally expensive and impractical to analyze the statistics for the incoming data~\citep{pang2021accumulative, kumar2020adversarial}. For the accumulative poisoning attack, except the gradient clipping discussed in~\citet{pang2021accumulative} that constrains the poisoning effect by small gradients, a principled defense~\citep{tao2021better,geiping2021doesn} based on adversarial training can also serve as the major technique to calibrate poison samples. However, both of them are indiscriminate in handling the poison and clean samples. Different from the previous methods, we introduce a novel information measure to discover the imperceptible poison samples by considering the model dynamics.

\begin{figure}[t!]
    \centering
    \includegraphics[scale=0.33]{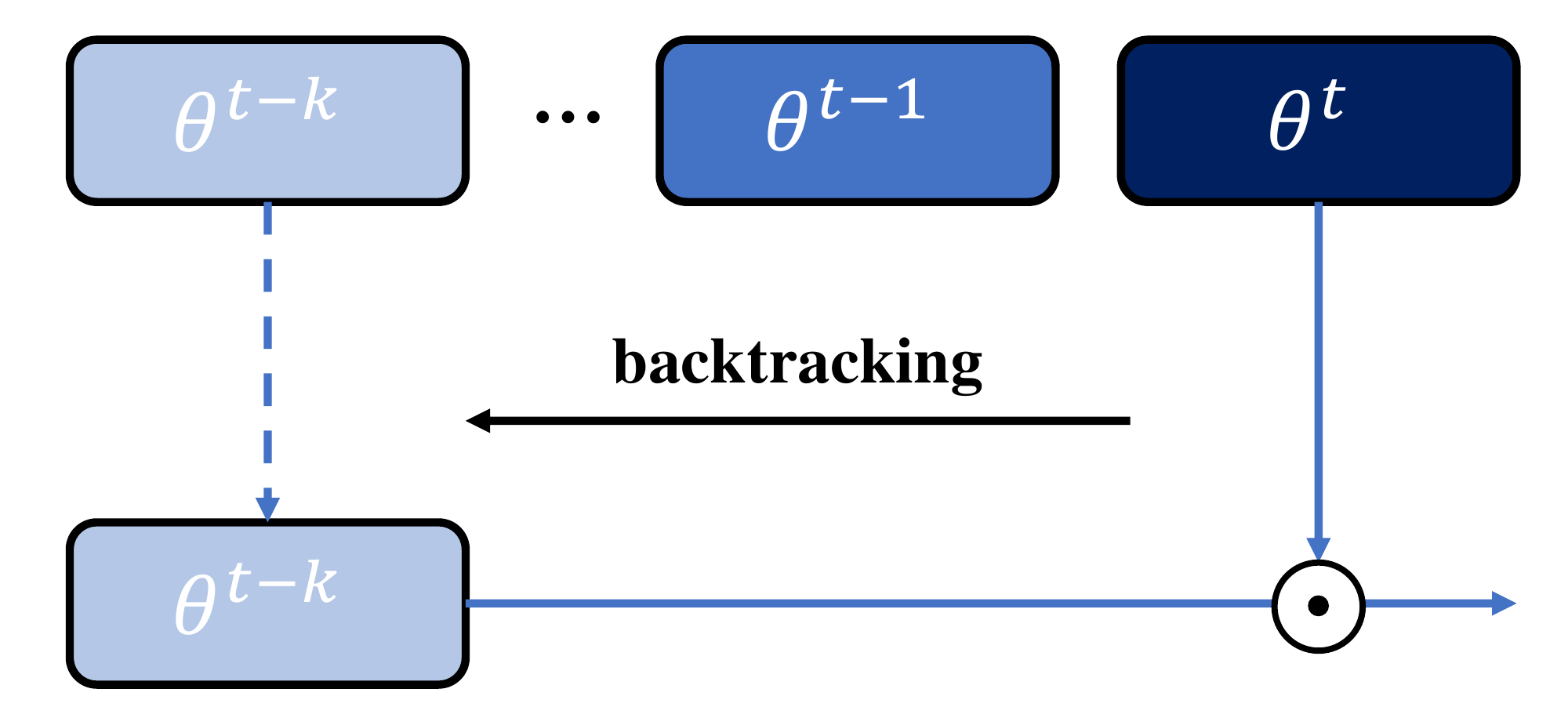}
    \vspace{2mm}
    \\
    \includegraphics[scale=0.19]{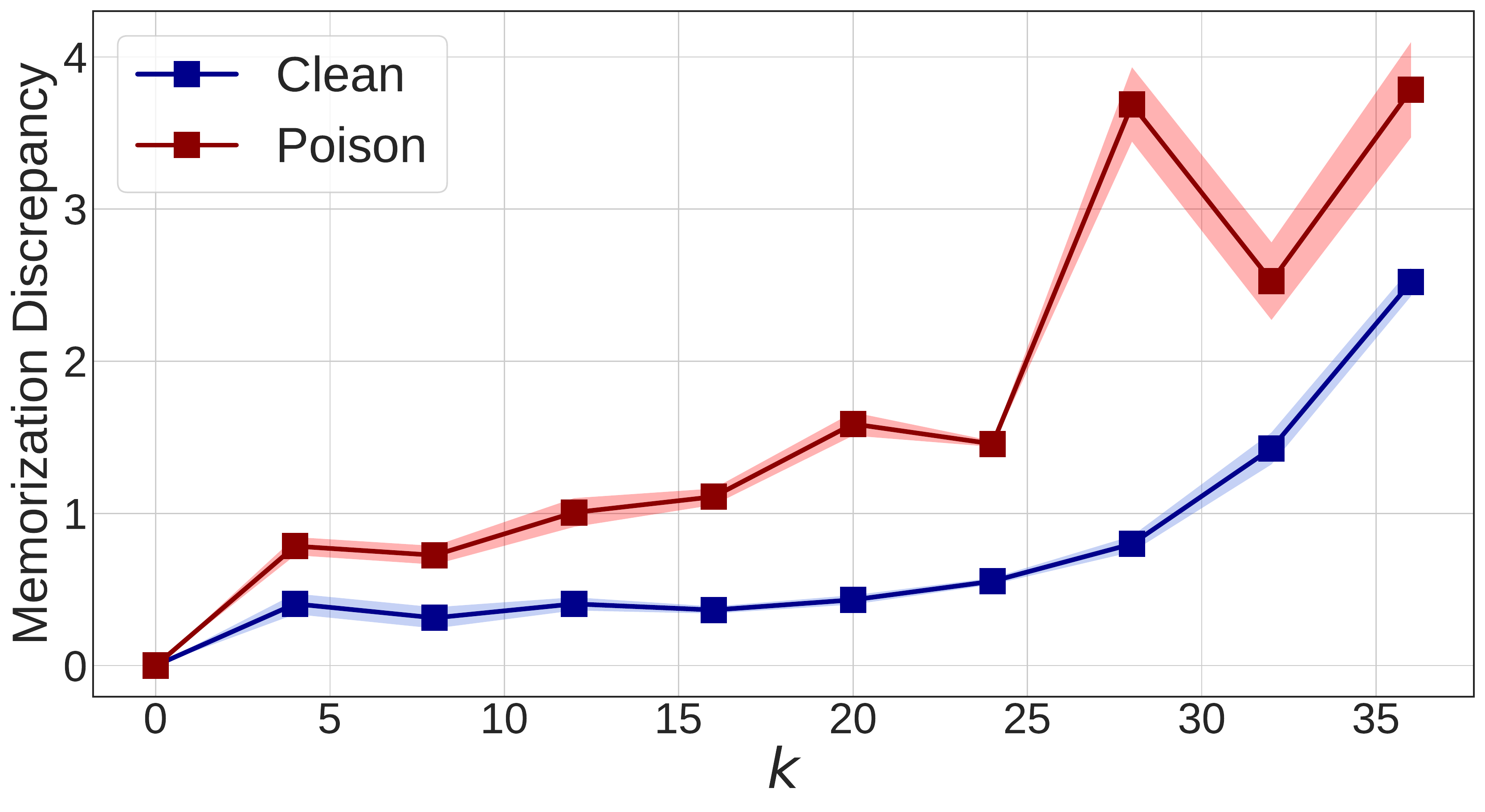}
    \vspace{-1mm}
    \caption{
    Top: illustration of the concrete operation to obtain the discrepancy information, i.e., Memorization Discrepancy. Bottom: the mean values of the Memorization Discrepancy on clean and poisoning batch data w.r.t. the backtracking interval $k$ (epochs). The $\theta^t$ denotes the current model which is used by the attacker to generate poison samples, and $\theta^{t-1}$ to $\theta^{t-k}$ are the historical model we backtracked. The discrepancy is measured by the output of the data using current and historical models. The difference between the Memorization Discrepancy on poison samples from that on clean samples is more distinguishable along with the enlargement of $k$. The underlying mechanism is further elaborated in Figure~\ref{fig:reason_2}.
    }
    \label{fig:reason}
    \vspace{-2mm}
\end{figure}

\section{Memorization Discrepancy}
\label{sec:memorization_discrepancy}
In this section, we present the new information measure \textit{Memorization Discrepancy} to explore the poison sample discovery through the lens of model dynamics during the training process. We first discuss our motivation, and then formally introduce the assumption and the definition of Memorization Discrepancy. Finally, we conduct experiments to empirically explore its corresponding properties.

\subsection{Motivation}
\label{sec:memo_motivation}

Different from the offline poisoning adversaries~\citep{li2016data,fowl2021adversarial}, the accumulative poisoning attack is allowed to interact with the model status to update its poison samples dynamically in the training process. Considering the practical situation, without sufficient knowledge of the original natural sample captured in the data streaming and the imperceptible characteristic of delusive attacks, the static information provided by the model from the single dimension seems to be hopeless to differentiate the poisoning and clean samples (e.g., the left panel of Figure~\ref{fig:motivation}). However, one critical component that is so far overlooked but easily backtracked~\citep{kumar2020adversarial} in training, is the historical model information. Since the accumulative poisoning attack utilize the sequential order property of real-time data streaming, we raise the following question,
\begin{quote}
\textit{Can we also exploit the information of model dynamics to gain some useful clues to identify the imperceptible accumulative poisoning attacks?}
\end{quote}
The answer is affirmative. As shown in the right panel of Figure~\ref{fig:motivation}, we can find the distributions of clean and poison samples are much different compared with the left panel in Figure~\ref{fig:motivation}. Such a significant difference is computed by taking the backtracked historical model into consideration (as illustrated at the top of Figure~\ref{fig:reason}). Intuitively, to achieve a better poisoning effect, the close interaction with the current model~\citep{pang2021accumulative} better optimizes the malicious objective (e.g., Eq.~\eqref{eq:mal_1}) for poisoning than other checkpoints, but it also ignores the changes in the historical model. This motivates us to further explore the poison discovery from the perspective of the dynamic changes in historical models. 

\begin{figure}[t!]
    \centering
    \includegraphics[scale=0.38]{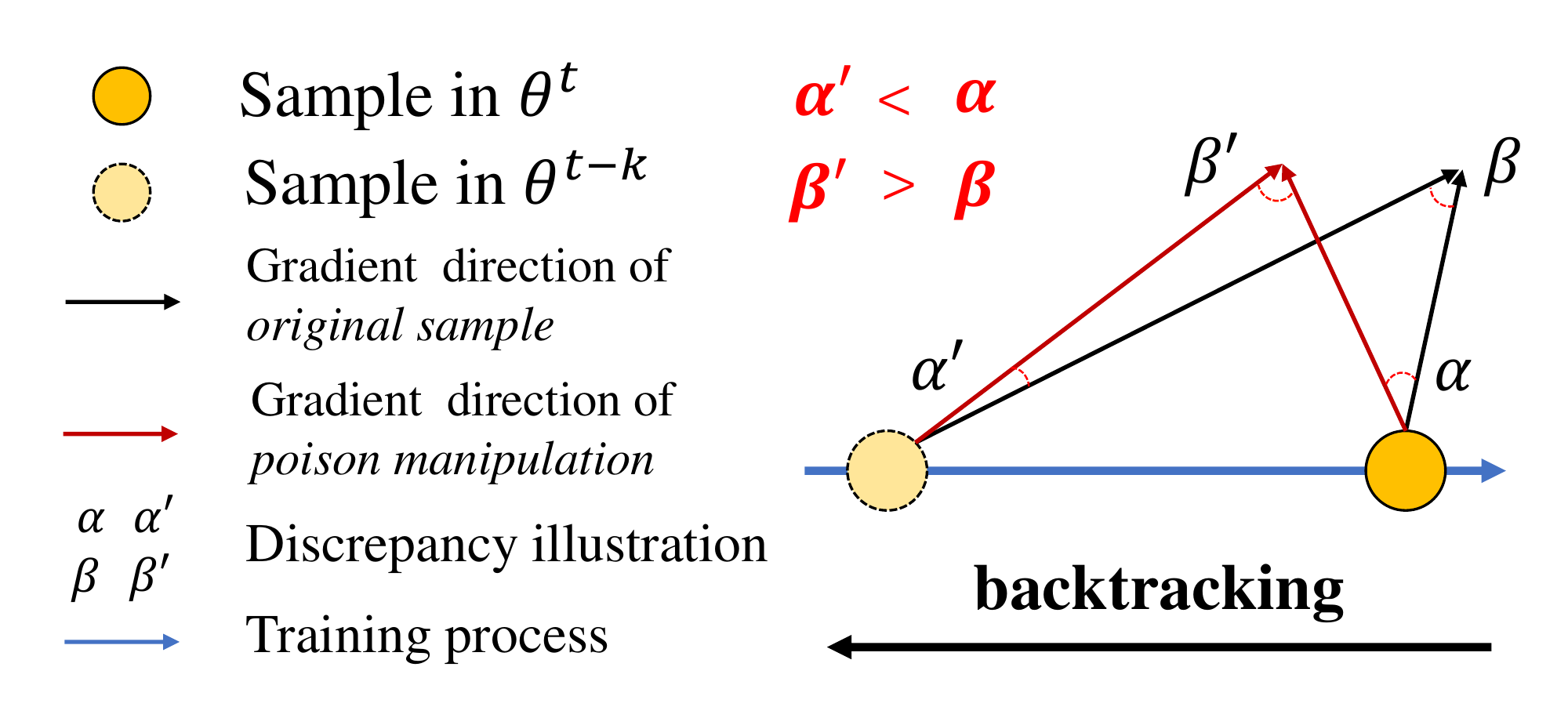}\vspace{2mm}\\
    \includegraphics[scale=0.20]{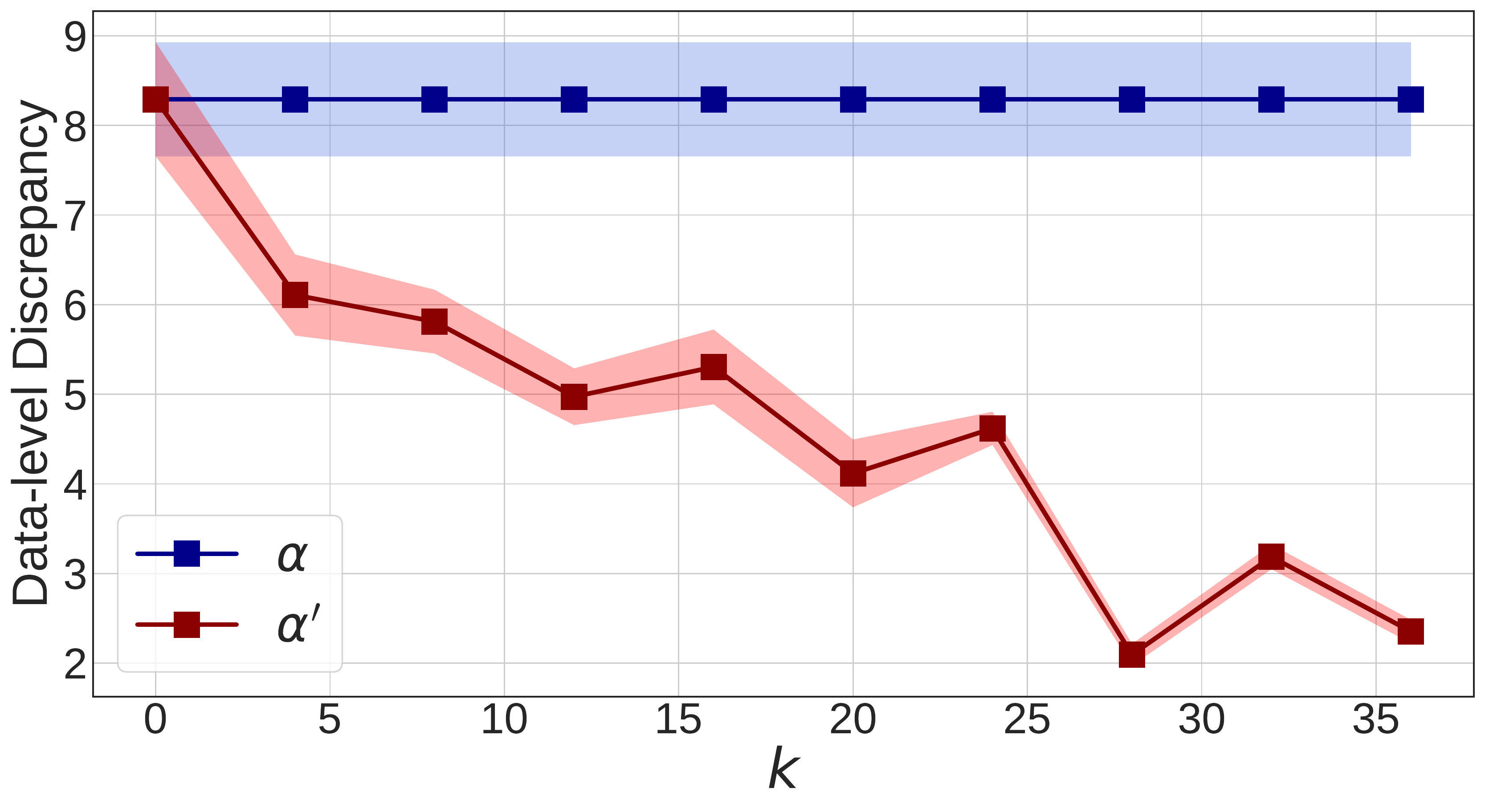}
    \vspace{-1mm}
    \caption{ 
    Top: illustration of model dynamics, which shows different effects (e.g., $\alpha$ and $\alpha'$) of the same poisoning manipulation on different model statuses (on the same original sample). Bottom: empirical verification about the above discrepancy by backtracking the model status. Here the $\alpha$ is the illustration of the discrepancy between two different optimization directions (or the gradient direction of the model $\theta^t$) approximated by using the outputs on clean and poison samples, respectively. And the $\alpha^{'}$ is the illustration of discrepancy on the historical model $\theta^{t-k}$. Models at different statuses will have different output changes for the same data manipulation, the discrepancy can be naturally captured using the historical model backtracked in the training process. The bottom figure empirically justifies that $\alpha'<\alpha$ and $\beta'>\beta$, which explains the underlying mechanism of previous trend in Figure~\ref{fig:reason}. 
    }
    \label{fig:reason_2}
    \vspace{-1mm}
\end{figure}

\subsection{Proposed Definition}
\label{sec:memo_empirical}

As the model is changed along with training on streaming data, it is natural to make the following assumption of model dynamics, which is about different model outputs with poison samples generated based on the victim model.

\begin{assumption}[Model Dynamics]
\label{ass:1}
Let $\theta^t$ and $\theta^{t-k}$ denote the current model at round $t$ and the historical model at round $t-k$, $\hat{x}(\theta^t)$ is the adversarial manipulation from $x$ by the model $\theta^t$, $\mathbb{D}$ indicate the general distribution discrepancy measurement\footnote{Note that, in the most experiments of this paper, we adopt Kullback–Leibler divergence~\citep{Joyce2011} in computation.}. Then, we have the following inequality, 
\begin{equation}\label{eq:assumption_discrepancy}
\begin{split}
    \mathbb{D}(f(\hat{x}(\theta^t); \theta^{t}), &f(x; \theta^{t})) \neq \\  &\mathbb{D}(f(\hat{x}(\theta^t); \theta^{t-k}), f(x; \theta^{t-k})).    
\end{split}
\end{equation}
\end{assumption}
The above inequality indicates that the poison sample generated on the victim model has a different effect on the output changes of a different model. Since the poison manipulation added to the clean samples targets the malicious learning objective that is different from the original one, the left side of Eq.~(\ref{eq:assumption_discrepancy}) actually reflects the difference between poison samples from clean samples. However, considering the practical situation of real-time data streaming, it is impractical to know whether the newly captured training samples are poisoned in advance. This motivates us to introduce another measure to leverage the information characteristic of historical models. Here we draw further theoretical analyses behind the Assumption~\ref{ass:1}, which construct the relationship on the difference between the poisoning objective (e.g. Eq.~\eqref{eq:mal_1}) and the original objective. We leave complete discussion and verification in Appendixes~\ref{app:proof} and~\ref{app:further_discuss}.

\begin{theorem}\label{theorem:correlation}
Let $f(x;\theta^t)$ denote the output about the sample $x$ at epoch $t$, k denotes the interval rounds, and $S$ denotes a clean dataset. Considering the opposite between objective $\min\mathcal{L}(S,\theta^*)$ and the poisoning objective $\max\mathcal{L}(S,\theta^*)$ where $\theta^*$ is the well-trained model respectively, there exists a learning period where we have,
\begin{equation}
\label{eq:theo}
\begin{split}
    \mathbb{D}(f(\hat{x}(\theta^{t}); \theta^{t}), f(\hat{x}(\theta^{t-k}); \theta^{t-k}))-&\\\mathbb{D}(f(x; \theta^{t}), f(x; \theta^{t-k}))  
    \propto \mathcal{L}(S; \theta^{t-k})&-\mathcal{L}(S; \theta^{t}).
\end{split}
\end{equation}
\end{theorem}
The above positive relationship constructs the discrepancy relationship of different samples (e.g., natural sample $x$ and poisoned sample $\hat{x}$) with the model learning dynamics. Due to the different poisoning effect results of sample $\hat{x}$ on different model stages, the previous discrepancy in Assumption~\ref{ass:1} is constructed on the different sample ($x$ and $\hat{x}$) with the same model (either current model at round $t$ or the historical model at round $t-k$). The underlying intuition of Eq.~\eqref{eq:theo} reflects the differences between the natural learning objective and the poisoning objective, where we evaluate their differences in model outputs. Hence, the sample-wise discrepancy can be transferred to the differences of each sample on different models. Since we can not know whether the newly captured samples are poisoned or not in advance, the latter formulation for the information measure on the same sample is more practical for utilization and help us to provide the final definition of Memorization Discrepancy.



\begin{figure}[t!]
    \centering
    \subfigure[Under different poisoning capacities (k=1)]{
    \includegraphics[scale=0.13]{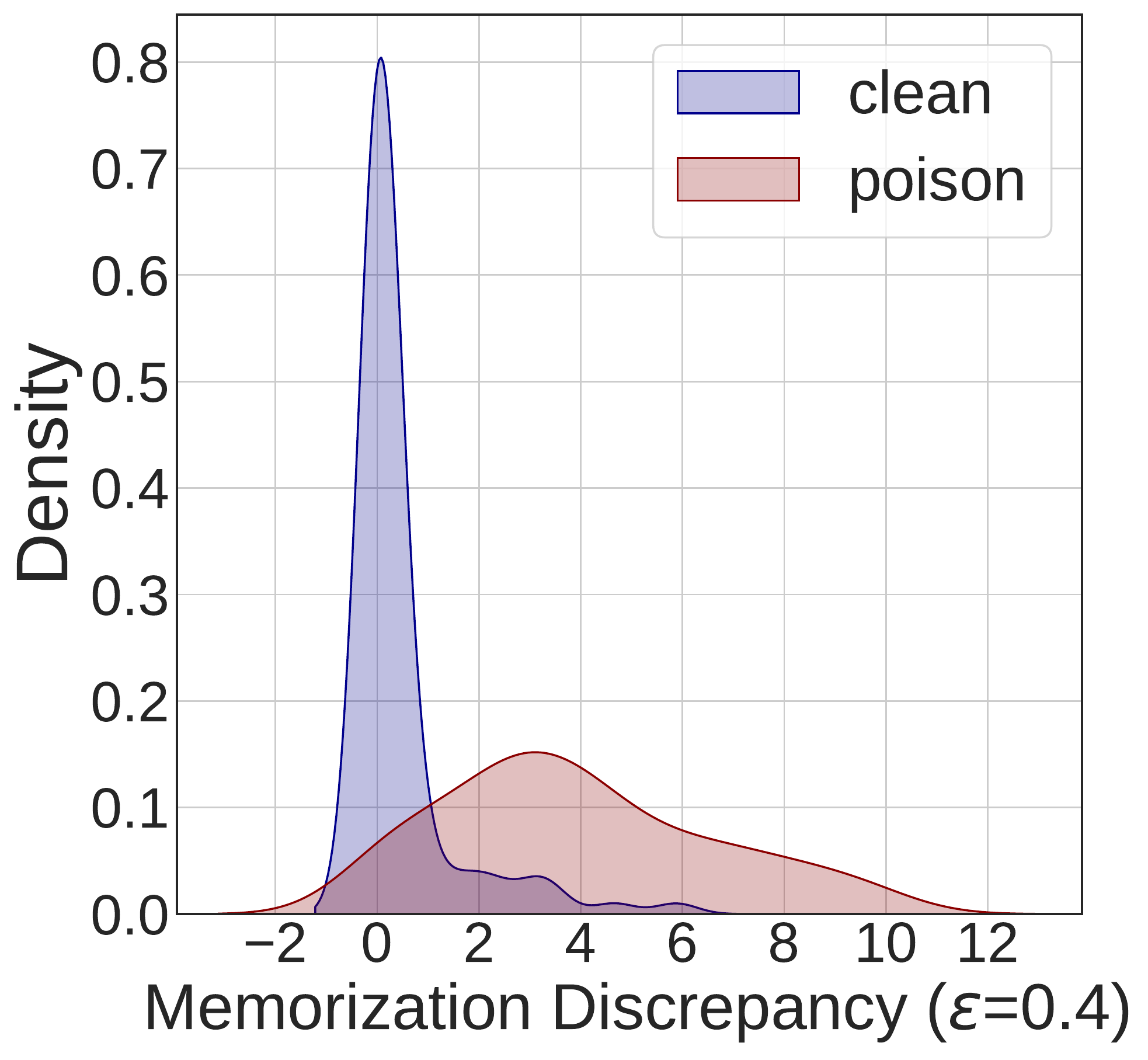}
    \includegraphics[scale=0.13]{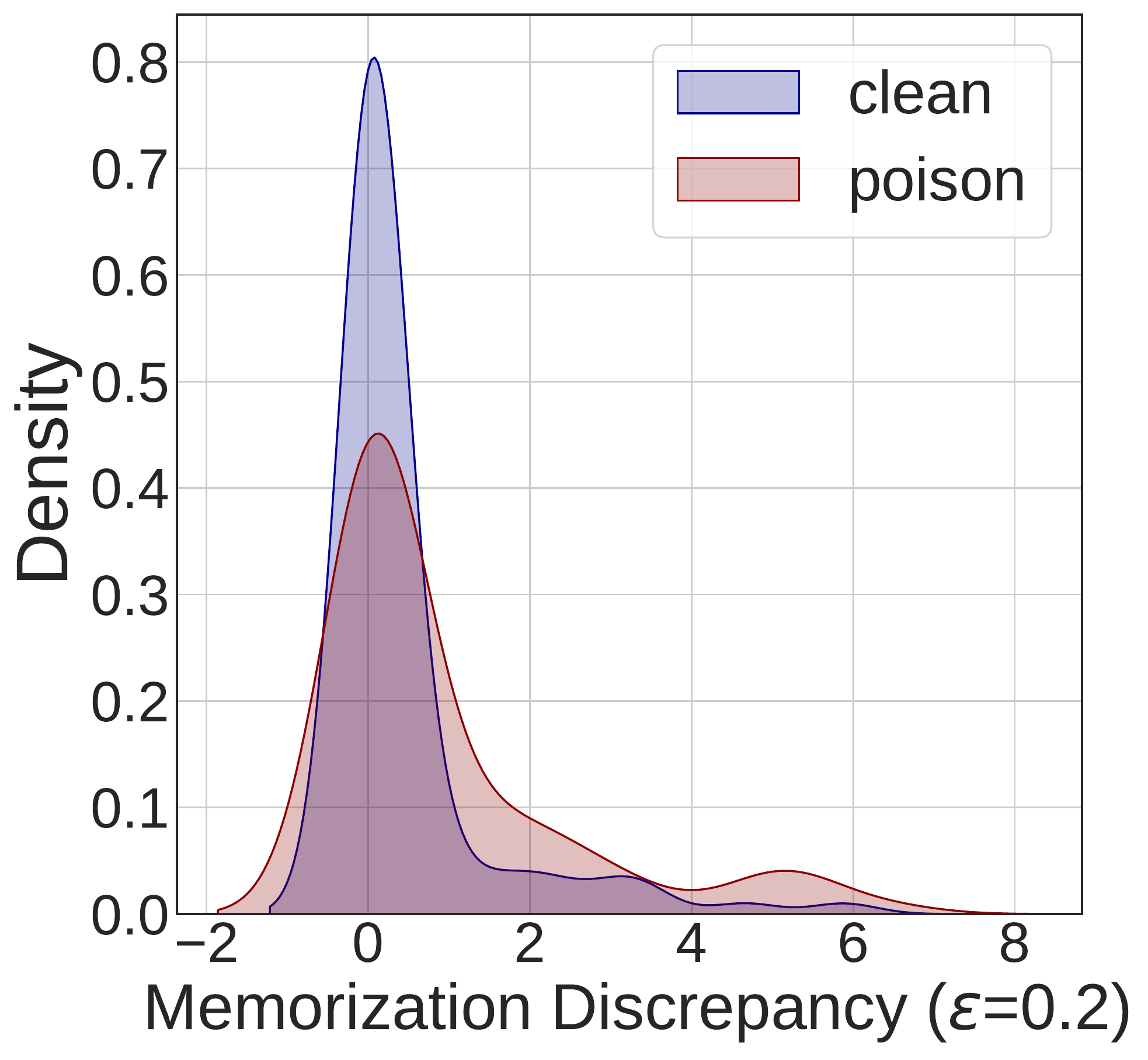}
    \includegraphics[scale=0.13]{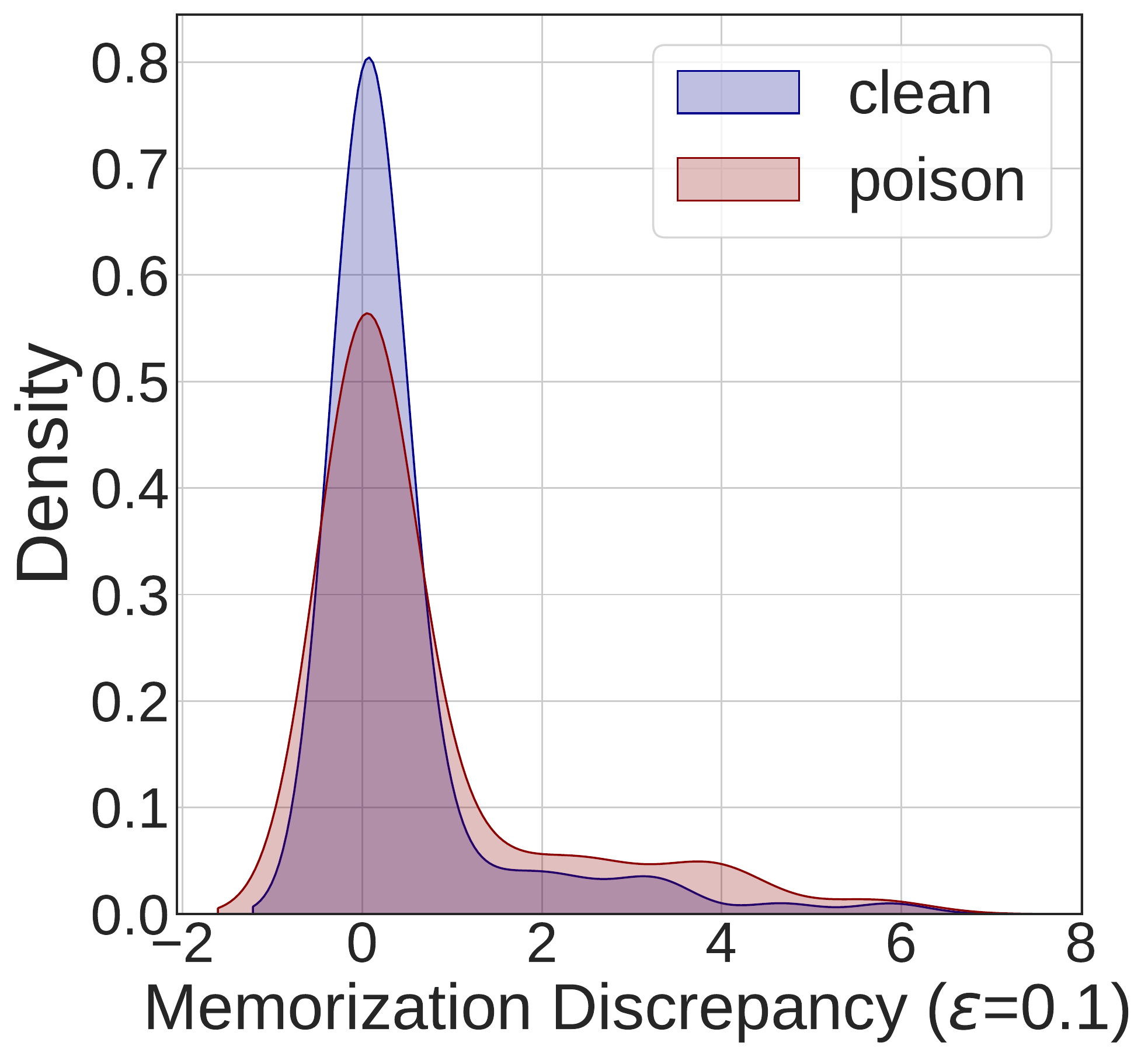}
    \label{fig4:a}
    }\\
    \subfigure[Under different backtracking intervals ($\epsilon$=0.1)]{
    \includegraphics[scale=0.13]{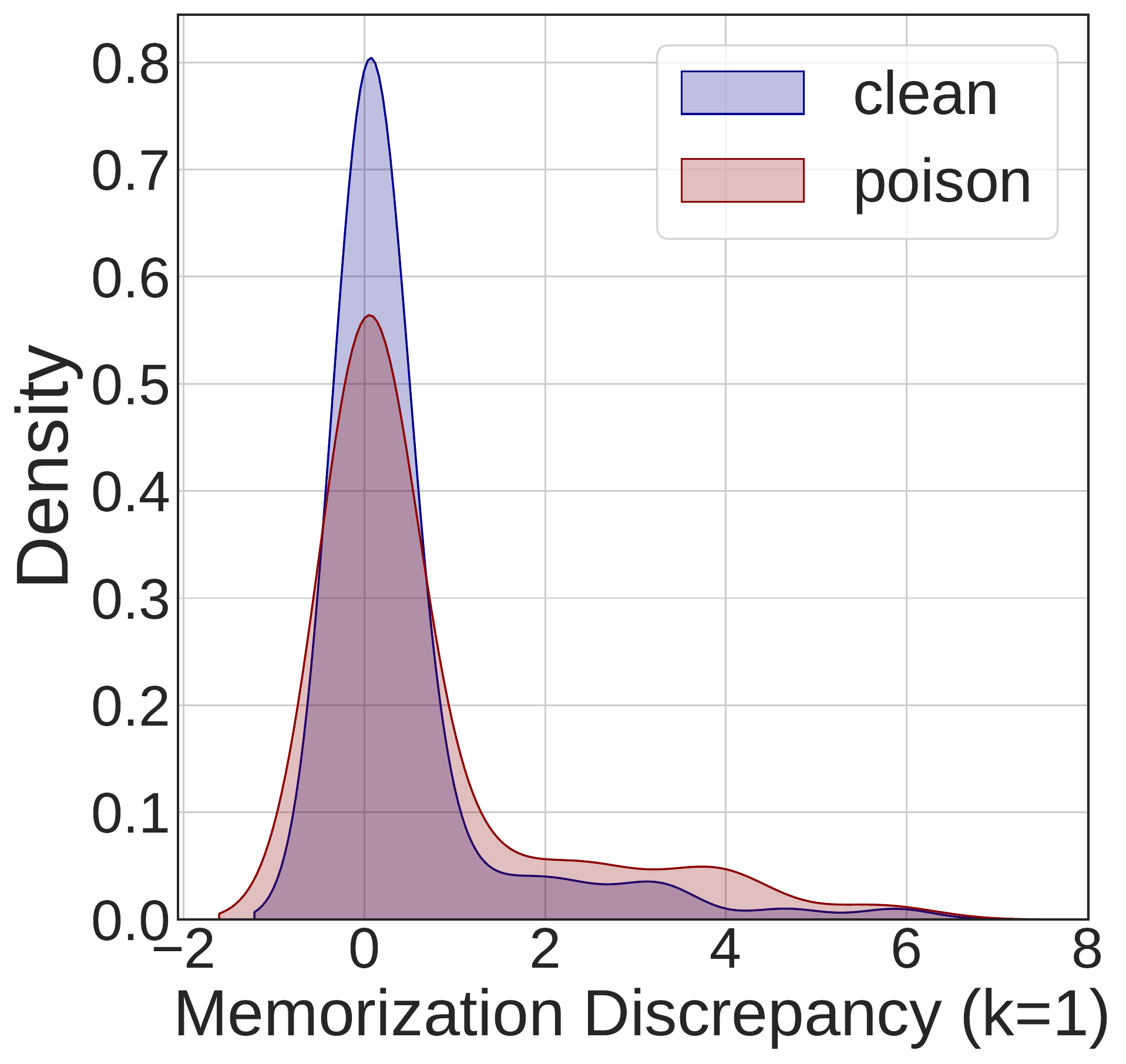}
    \includegraphics[scale=0.13]{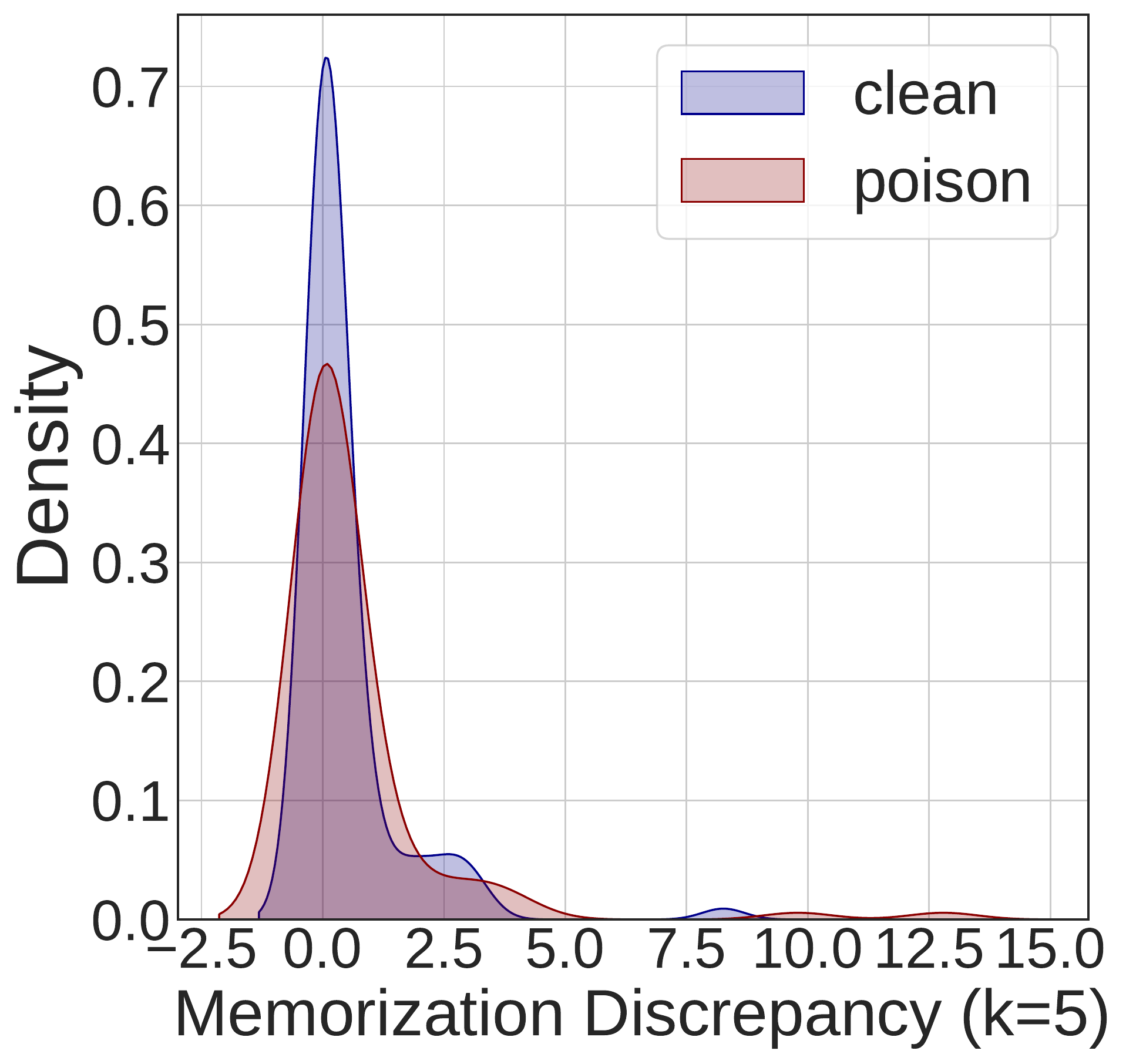}
    \includegraphics[scale=0.13]{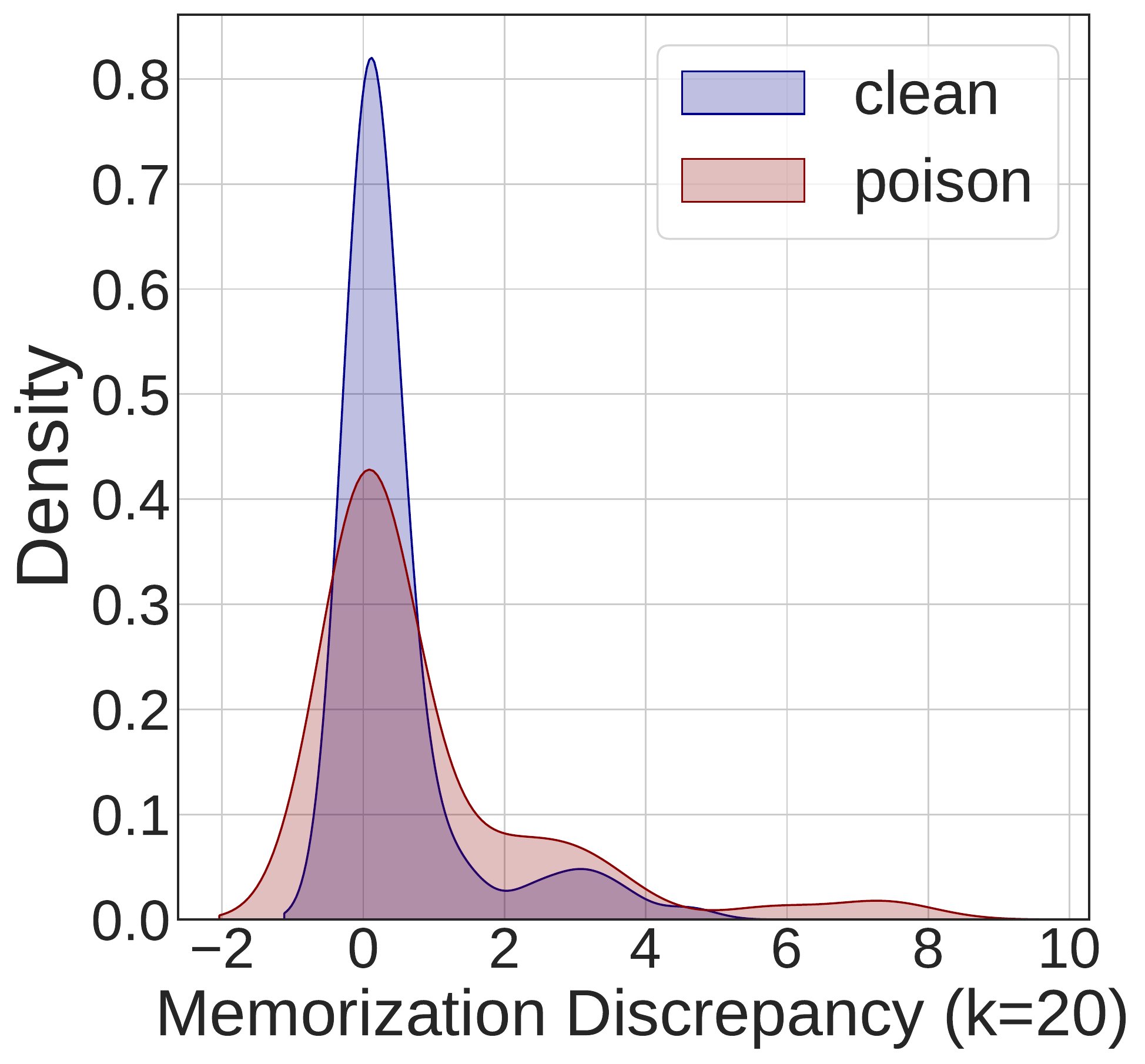}
    \label{fig4:b}
    }
    \vspace{-1mm}
    \caption{Empirical exploration about poisoning discovery using Memorization Discrepancy. (a) The distribution discrepancy of clean and poison data can be constrained by controlling the poisoning capacity, e.g., the perturbation radius $\epsilon$; (b) The poison samples can be more distinguishable from clean samples by enlarging the backtracking interval in Memorization Discrepancy to find a highly discriminative status, which involves model-level information.}
    \label{fig:investigate_1}
    \vspace{-2mm}
\end{figure}

Below we formally introduce the new information measure,
\begin{definition}[Memorization Discrepancy]
Consider $f:~\mathbb{R}^d\rightarrow \Delta^C$ that maps the input feature to the $C$-dimensional simplex, $\hat{x}(\theta^t)$ is disturbed from $x$ on the model $f(\cdot;\theta^t)$, and $\theta^{t-k}$ means the parameters of the $k$-interval historical model compared to the current $t$. Then, we define \textit{Memorization Discrepancy} on $\hat{x}(\theta^t)$ based on the current parameter $\theta^t$ and the historical parameter $\theta^{t-k}$ as,
\begin{equation}\label{eq:memorization_discrepancy}
    \mathbb{D}(f(\hat{x}(\theta^t); \theta^{t-k}), f(\hat{x}(\theta^t); \theta^{t})),
\end{equation}
which measures the discrepancy of the different model's outputs on the same $\hat{x}$ generated on $\theta^t$. 
\end{definition}

The underlying mechanism of Memorization Discrepancy is to capture the model dynamic on the same sample during the training process, which explicitly reflects the imperceptible poisoning manipulation in the training samples via the difference in model outputs. In Figure~\ref{fig:motivation}, we can find that the discrepancy value of both clean and poison samples is enlarged when we backtrack the historical models. Especially, the value of poison samples increases more than that of clean samples. According to this phenomenon, we have two following conjectures on Memorization Discrepancy.

\begin{property}[Monotonically Increasing Interval]
\label{pro:1}
There exists an interval $k$ from $t$ to $t-k$ where the value of $\mathbb{D}(f(x^*; \theta^{t-k}), f(x^*; \theta^{t}))$ is monotonically increasing from $0$ to $k$, where $x^*$ can indicate either the original clean sample $x$ or the poison sample $\hat{x}$.
\end{property}

\begin{property}[Highly Discriminative Status]
\label{pro:2}
There exists an model status $\theta^{t-k}$ where the mean value of poison samples $\mathbb{D}(f(\hat{x}(\theta^t); \theta^{t-k}), f(\hat{x}(\theta^t); \theta^{t}))$ is much larger than that of clean samples $\mathbb{D}(f(x; \theta^{t-k}), f(x; \theta^{t}))$.
\end{property}


\begin{figure}[t!]
    \centering
    \subfigure[Visualization of clean, poison, and out-of-distribution data]{
    \includegraphics[scale=0.105]{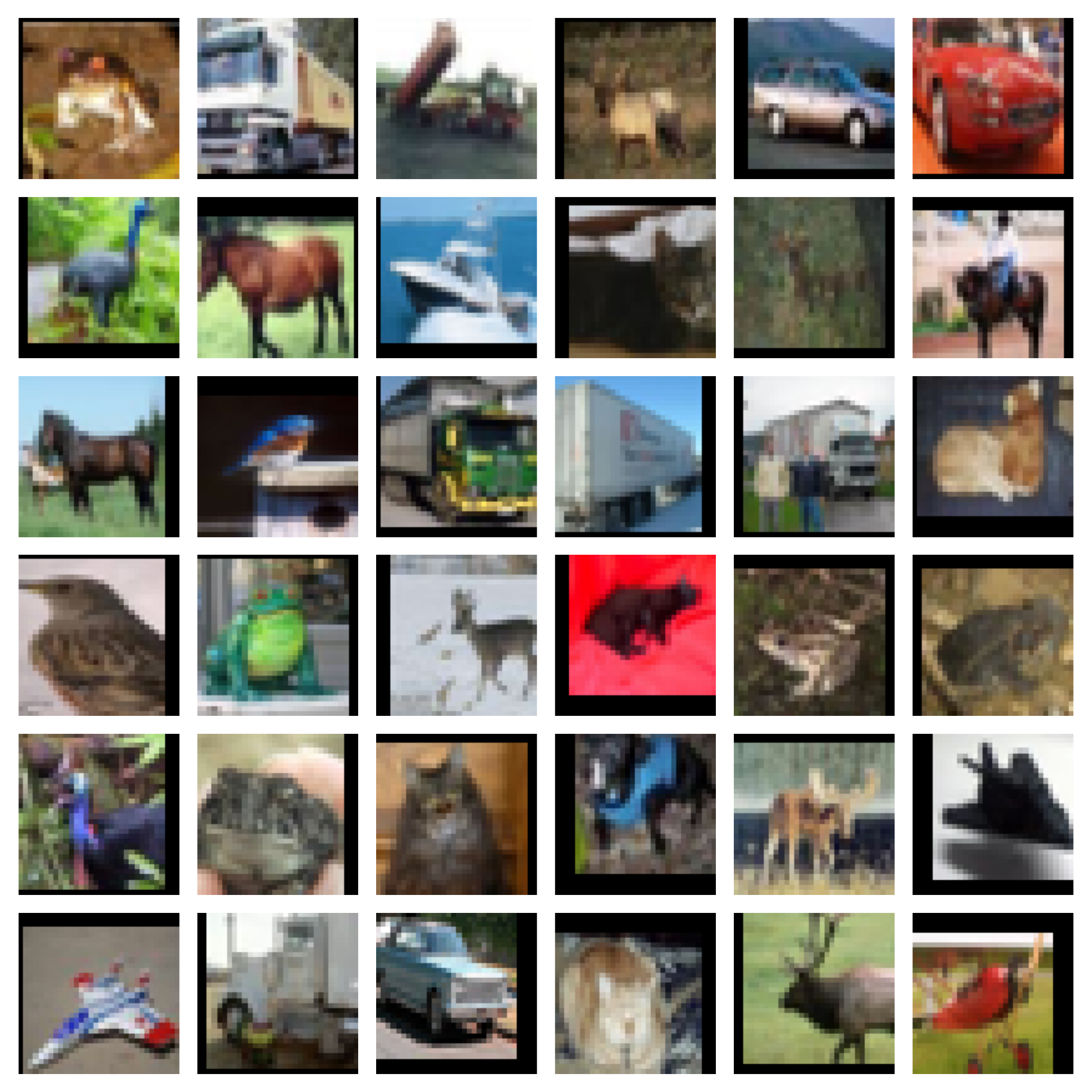}
    \hspace{0.02in}
    \includegraphics[scale=0.105]{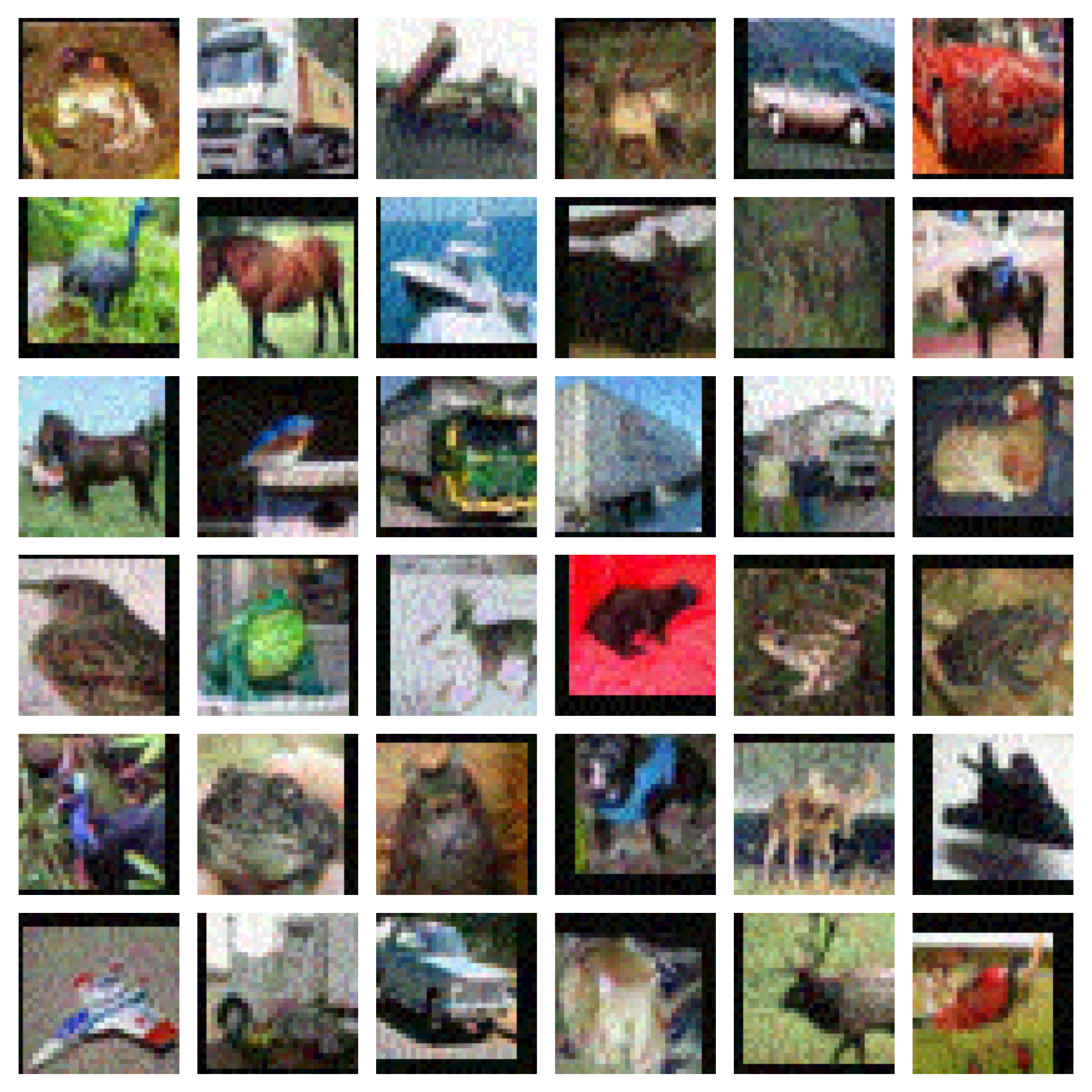}
    \hspace{0.02in}
    \includegraphics[scale=0.105]{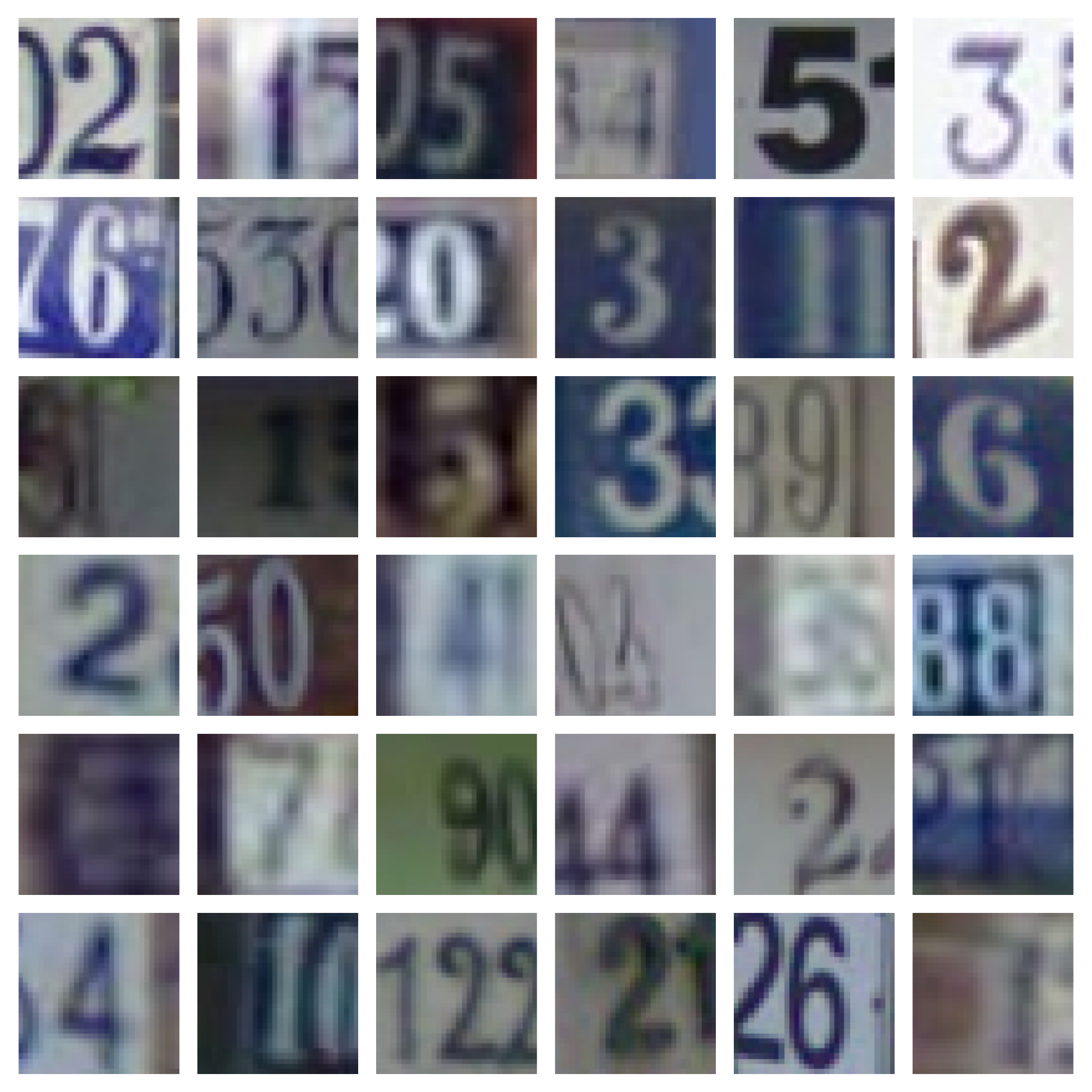}
    }\\
    \subfigure[Comparisons of the above three type of data]{
    \includegraphics[scale=0.13]{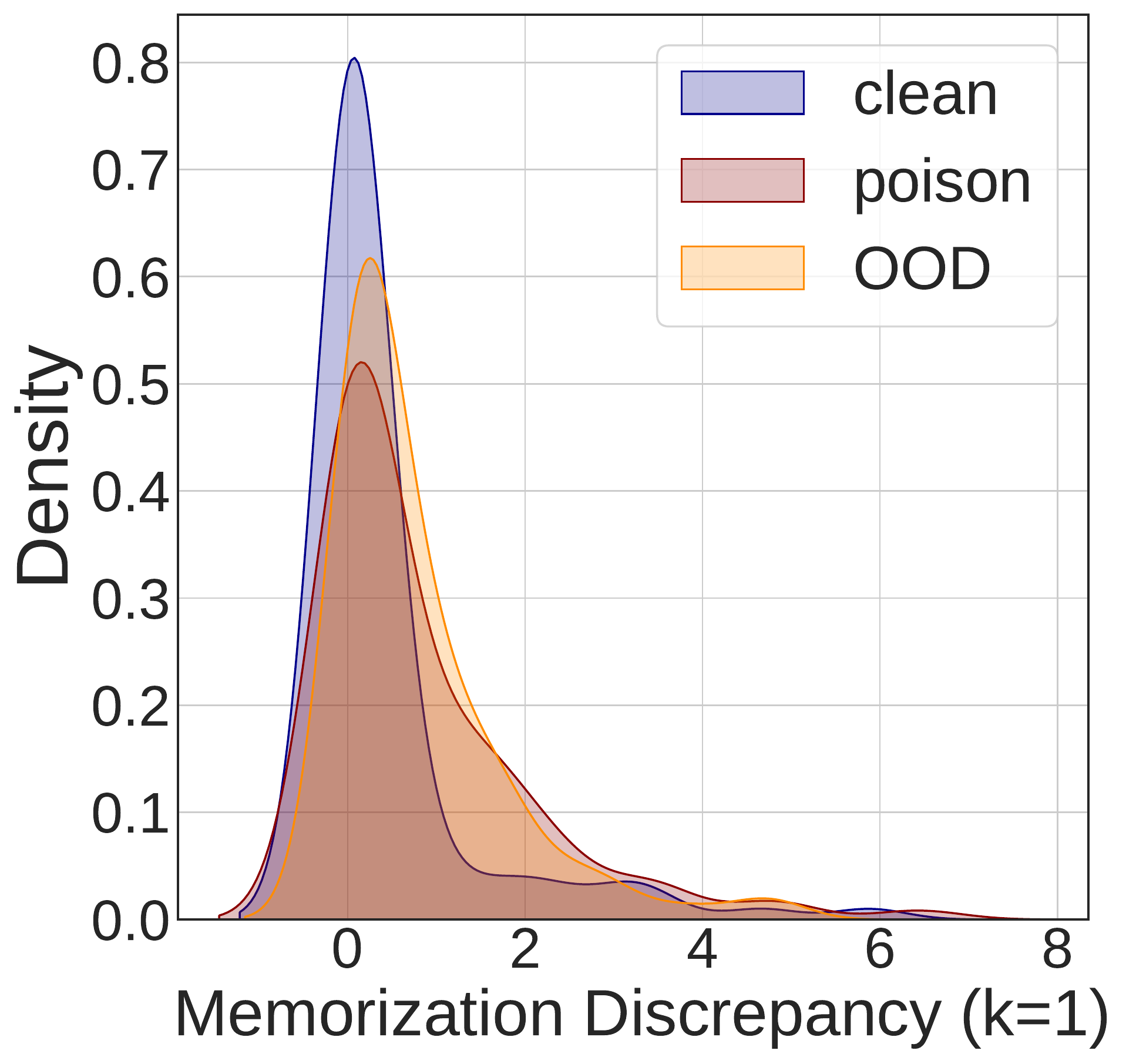}
    \includegraphics[scale=0.13]{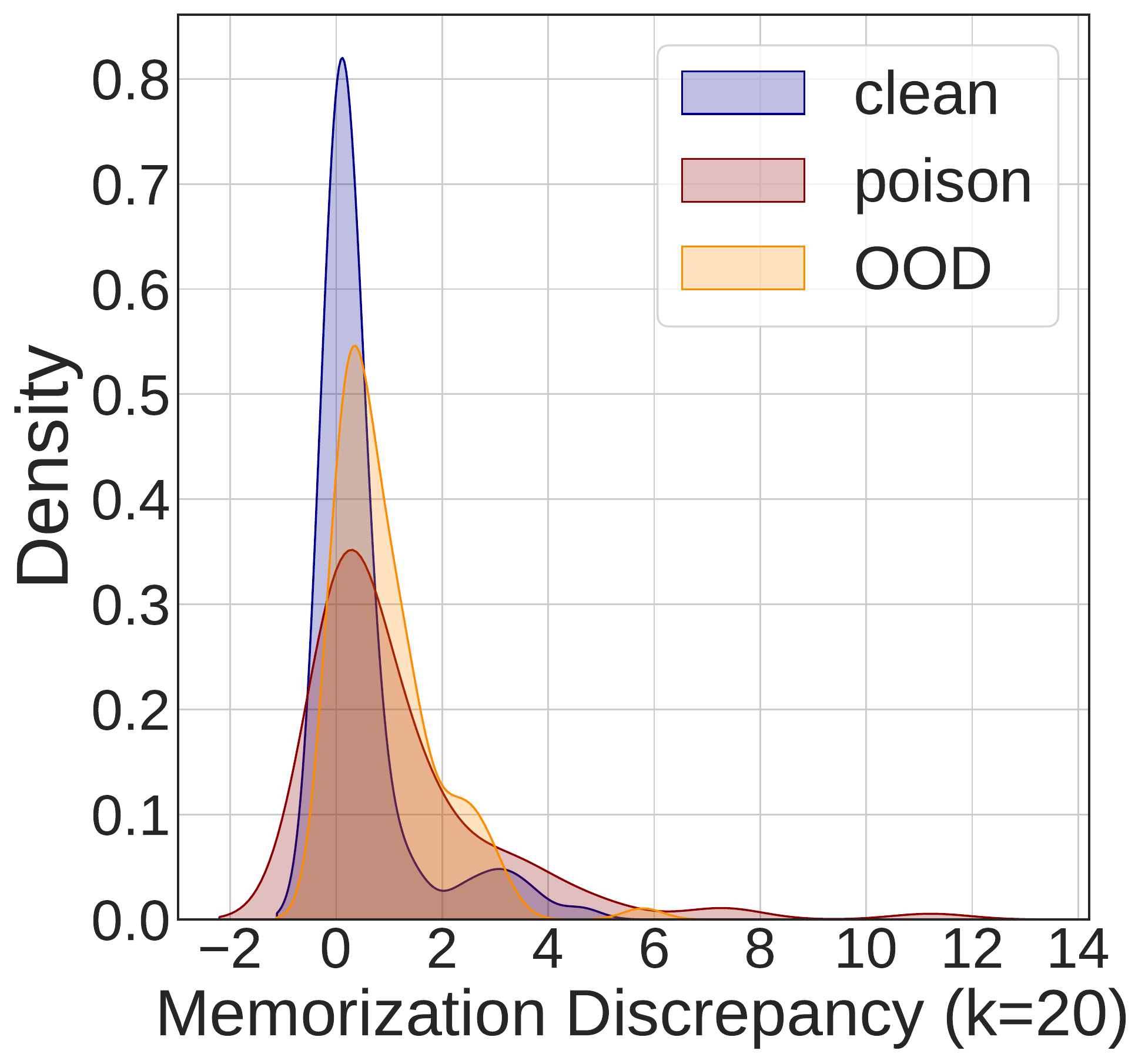}
    \includegraphics[scale=0.138]{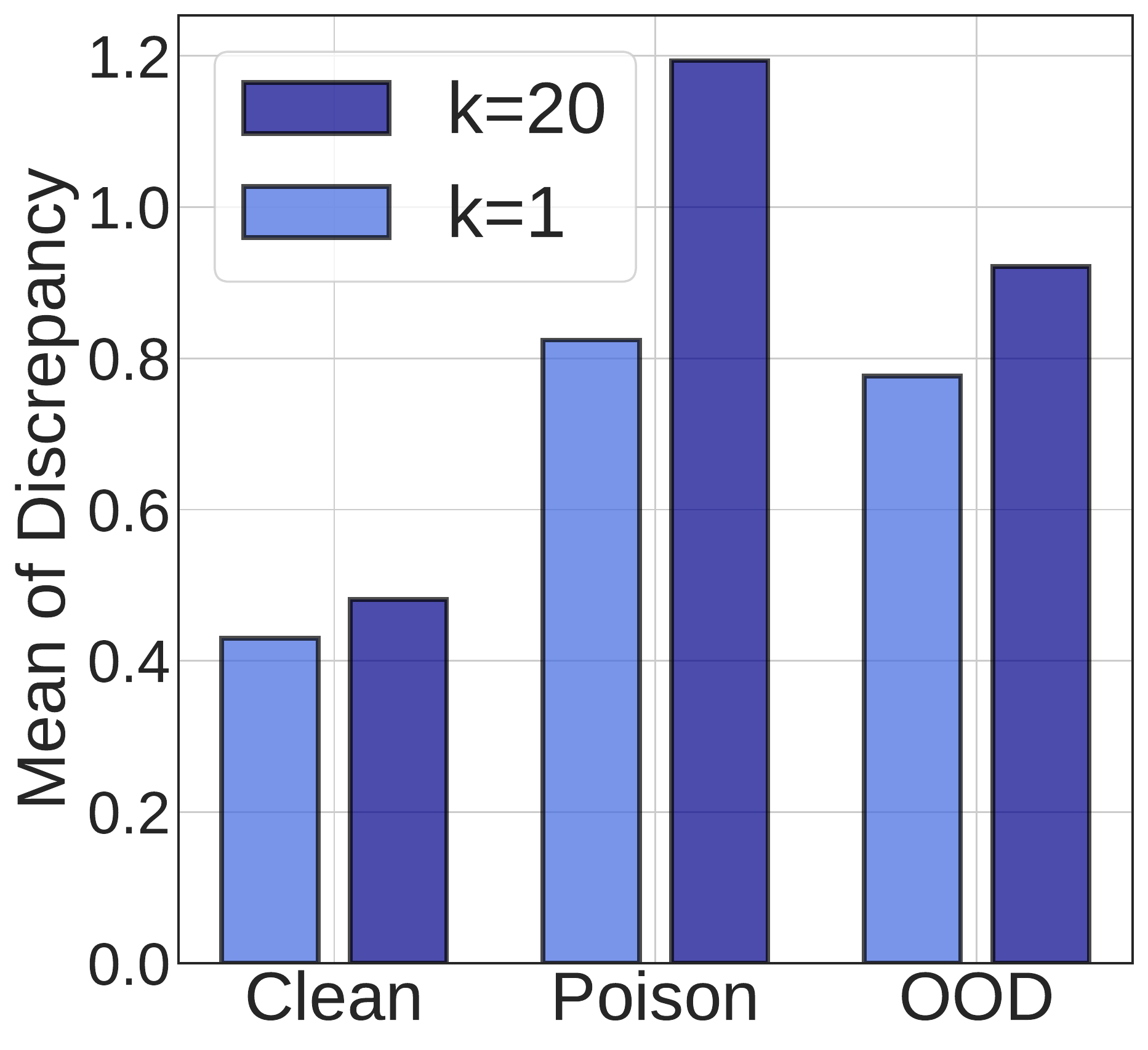}
    }
    \vspace{-1mm}
    \caption{Comparisons about optimized and statical distribution shift distinguished using Memorization Discrepancy. (a) Three types of data from left to right: clean data (CIFRA-10), poison data ($\epsilon$=0.064), and out-of-distribution (OOD) data (SVHN); (b) Compared with the statical distribution shift (OOD), the poison samples which are optimized for the targeted model are more sensitive to the backtracking interval in Memorization Discrepancy.}
    \label{fig:investigate_2}
    \vspace{-2mm}
\end{figure}

\subsection{Empirical Study on the Properties}

Here we study the Memorization Discrepancy through the simulated experiments on CIFAR-10 dataset following~\citet{pang2021accumulative}, and the detailed setups can be found in Section~\ref{sec:exp_setup}. The below empirical results 
respectively justify the previous Assumption~\ref{ass:1}, Properties~\ref{pro:1} and~\ref{pro:2}. More exploration of the discrepancy is provided in Appendix~\ref{app:more_dynamics}.

In Figure~\ref{fig:reason}, we illustrate the pipeline to obtain the Memorization Discrepancy. Specifically, by comparing the auxiliary historical model's output and the current model's output, the Memorization Discrepancy can be easily calculated. From the figure, we can find that the mean values are almost monotonically increasing from $0$ to $25$ epochs with the increasing $k$, which empirically verifies its general trend. 

In Figure~\ref{fig:reason_2}, we give the underlying explanation behind the dynamics of Memorization Discrepancy and empirically justify Assumption~\ref{ass:1} via the approximation results for $\alpha$ and $\alpha'$ in the top panel. According to the top of the figure, the Memorization Discrepancy of clean and poison samples can be denoted by $\beta$ and $\beta'$, and their relationship can be reflected by the change on models $\theta^t$ and $\theta^{t-k}$, i.e., $\alpha$ and $\alpha'$. 
In practice, as the defense party does not know whether the data is poisoning, Memorization Discrepancy is a good choice while Eq.~\eqref{eq:assumption_discrepancy} assumes the poisoning fact by default.

In Figure~\ref{fig:investigate_1}, we present that Memorization Discrepancy can distinguish the poison sample with limited poisoning capacities (e.g., the perturbation constraint $\epsilon$ is small to guarantee imperceptible data-level manipulation) by backtracking the historical models as an auxiliary inspector, which enlarges the distribution discrepancy as shown in Figure~\ref{fig4:b}. In Figure~\ref{fig:investigate_2}, we also consider another kind of natural data that is out-of-distribution (OOD) and may be confusedly reflected with higher Memorization Discrepancy. Fortunately, under the comparison, the poison sample optimized on the victim model shows being more sensitive to the dynamic changes in historical models than static OOD samples. On the other hand, those OOD samples with noticeable visual differences can be easier to clean up than the imperceptible poison ones.

\subsection{Discrepancy-aware Sample Correction}
\label{sec:proposed_method}

Inspired by the previous properties of Memorization Discrepancy, we propose the \textit{Discrepancy-aware Sample Correction} (DSC) to utilize the model dynamics which can capture the differences between the potential poison samples from the clean ones by an auxiliary historical model. 

The high-level intuition is to employ the Memorization Discrepancy to the previous principled reverse adversarial generation~\citep{tao2021better} as guidance for the sample correction. Concretely, we summarize the detailed procedure of DSC in Algorithm~\ref{alg:dsc}. In each mini-batch training, we will leverage the Memorization Discrepancy to validate whether the sample is a potential poison sample. The multi-step reverse adversarial generation will then be conducted through the following objective, 
\begin{align}\label{eq:objective}
\begin{split}
    \Tilde{{x}} = \arg\min_{\Tilde{x}\in\mathcal{B}[{x}, \epsilon]} &\ell(f(x), y), \\ 
    &\text{s.t.}, \mathbb{D}(f(\Tilde{{x}}; \theta),f(\Tilde{{x}}; \theta^*))>P,
\end{split}
\end{align}
where $\Tilde{x}$ is the calibrated sample, $\mathcal{B}[{x},\epsilon] = \{\Tilde{x} \mid d_{\infty}(\bx,\Tilde{x})\le\epsilon\}$
be the closed ball of radius $\epsilon>0$ centered at the training sample $\bx$, $P$ is the estimated discrepancy threshold, and $\theta^*$ is the historical auxiliary model.
In addition, we also record the Memorization Discrepancy using a certain measurement (e.g., KL divergence). According to Property~\ref{pro:2} and previous empirical results, the poison data has a larger discrepancy value than the clean data. Thus, we adopt an early-stopping here to relax the minimization objective for sample correction. The multi-step correction will stop if Memorization Discrepancy is smaller than an adjustable threshold during the pre-defined correction steps. This operation can avoid over-calibration for those clean samples, and we empirically justify its effectiveness in Figure~\ref{fig:method}.

\begin{figure}[t!]
    \centering
    \includegraphics[scale=0.17]{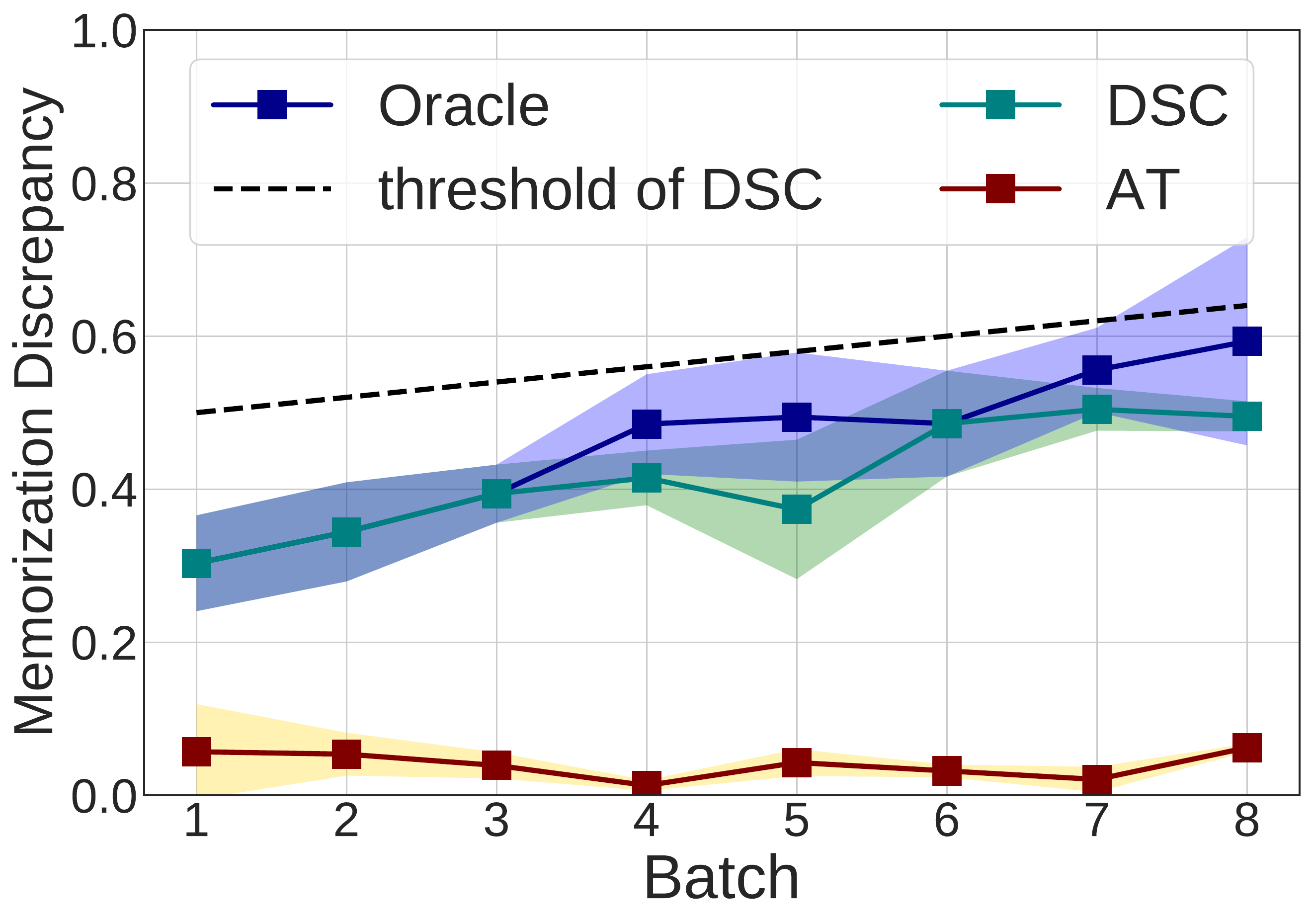}
    \hspace{0.02in}
    \includegraphics[scale=0.17]{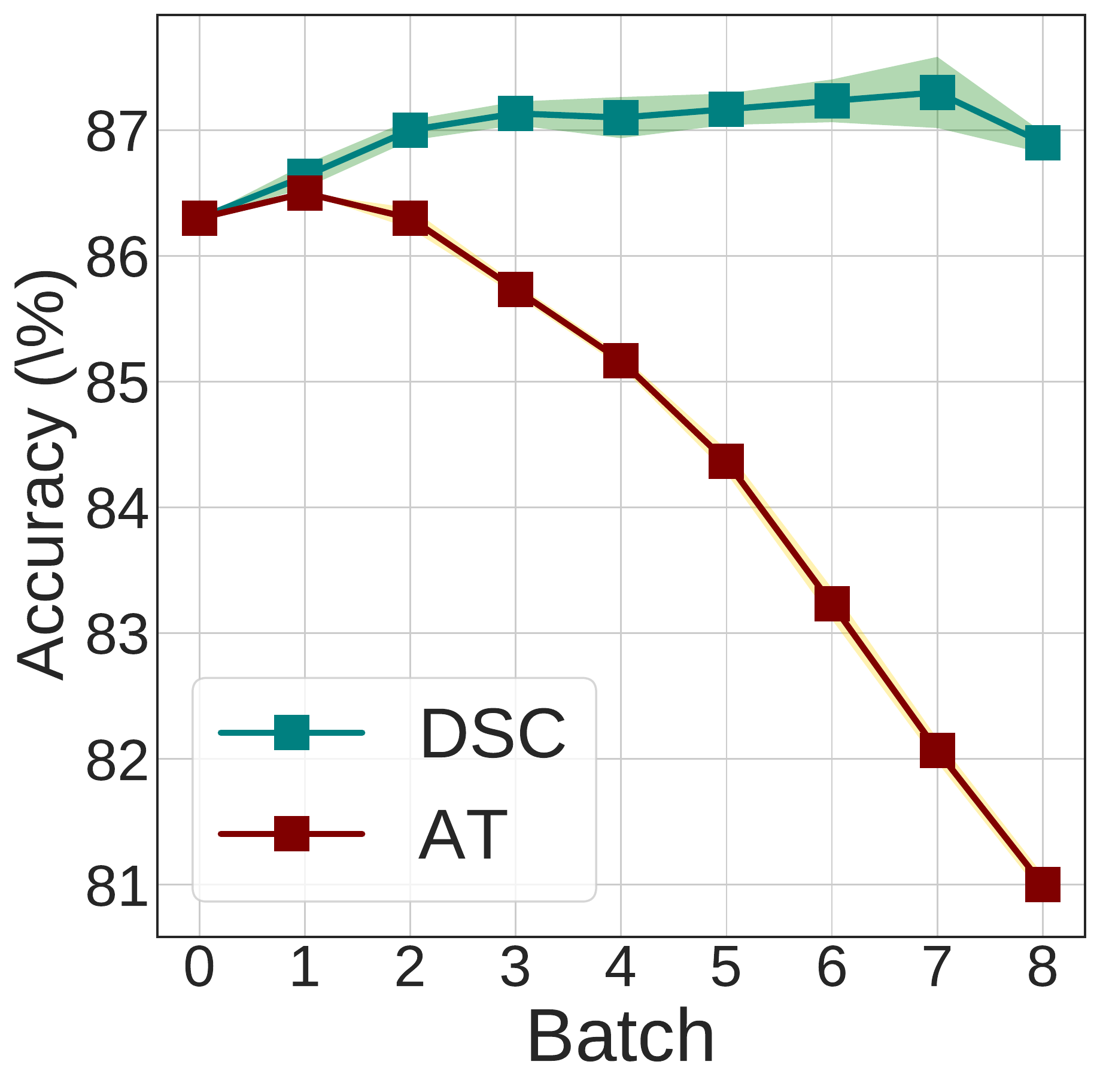}
    \vspace{-2mm}
    \caption{
    Left panel: the Memorization Discrepancy corresponding to the training samples in real-time data streaming. Right panel: the test accuracy of the AT-based method and our DSC. The reverse adversarial perturbations over-calibrate the clean samples and result in lower accuracy while our DSC can filter the clean sample with an estimated threshold by Memorization Discrepancy. 
    }
    \label{fig:method}
    \vspace{-4mm}
\end{figure}

\section{Experiments}
\label{sec:exp}

In this section, we present a comprehensive analysis of the Memorization Discrepancy and verify the effectiveness of our proposed DSC with the previous baseline methods for defending against accumulative poisoning attacks. More details and supplementary can be referred to in Appendix~\ref{sec:app_exp}.

\subsection{Experiment Setups}
\label{sec:exp_setup}


\paragraph{Training simulation.}
Following~\citet{pang2021accumulative}, we simulate the real-time data streaming using the SVHN~\citep{netzer2011reading_SVHN}, CIFAR-10 and CIFAR-100~\citep{krizhevsky2009learning_cifar10} datasets. The overall learning process consists of two specific phases with different training data (i.e., clean samples and poison samples). The first phase is named \emph{burn-in phase}, like model pre-training, the model will be trained on natural data before 
taking the training examples from other untrusted sources~\citep{biggio2018wild}. The second phase is termed as \emph{victim phase}, in which the adversaries begin to inject the poison samples to attack the current model. Same as~\citep{pang2021accumulative}, we train ResNet-18~\citep{he2016deep} using the SGD optimizer with the learning rate $0.1$, momentum $0.9$, and weight decay $0.0001$. During the whole process, we keep the batchsize of data streaming at 100.


\begin{table*}[t!]
\renewcommand\arraystretch{0.95}
\centering
\caption{Test accuracy (\%) of the simulated experiments on real-time data streaming (Mean$\pm$Std).}
\footnotesize
\label{table:exp_performance_cifar10}
\begin{tabular}{c|c|c|c|c|c|c}
\toprule[1.7pt]
\midrule[0.6pt]
CIFAR-10 &  Defense & Accuracy: Start & Batch & Accuracy: +Poison & Accuracy: + Trigger & $\Delta$ \\
\midrule[0.6pt]
\multirow{3}{*}{Clean Oracle} & ST & \multirow{9}{*}{86.3} & - & 84.4 & 84.4 & - \\
~ & GC & ~ &- & 86.2 & 86.2 & - \\
~ & AT & ~ &-& 77.2 & 77.2 & - \\
~ & \cellcolor{greyL}\textbf{DSC} & ~ &\cellcolor{greyL}-& \cellcolor{greyL}84.7 & \cellcolor{greyL}84.7 & \cellcolor{greyL}-  \\

\cmidrule(lr){0-1} \cmidrule(lr){4-7}
\multirow{3}{*}{Accu. Poison} & ST & ~ &1& 75.7$\pm$3.33 & 50.4$\pm$5.03 & -25.3$\pm$4.13 \\
~ & GC & ~& 3 & 79.7$\pm$0.25 & 75.1$\pm$0.05 & -4.6$\pm$0.26 \\
~ & AT & ~& 3 & 80.1$\pm$0.10 & 75.3$\pm$0.26 & -4.7$\pm$0.20  \\
~ & \cellcolor{greyL}\textbf{DSC} & ~& \cellcolor{greyL}3 & \cellcolor{greyL}\textbf{81.2$\pm$0.35 } & \cellcolor{greyL}\textbf{77.3$\pm$0.58} & \cellcolor{greyL}\textbf{-3.8$\pm$0.31}  \\

\midrule[0.6pt]


\midrule[0.6pt]
SVHN &  Defense & Accuracy: Start & Batch & Accuracy: +Poison & Accuracy: + Trigger & $\Delta$ \\
\midrule[0.6pt]
\multirow{3}{*}{Clean Oracle} & ST & \multirow{9}{*}{94.6} &-& 93.4 & 93.4 & - \\
~ & GC & ~ &-& 94.5 & 94.5 & - \\
~ & AT & ~ &-& 89.9 & 89.9 & - \\
~ & \cellcolor{greyL}\textbf{DSC} & ~ &\cellcolor{greyL}-& \cellcolor{greyL}94.7 & \cellcolor{greyL}94.7 & \cellcolor{greyL}-  \\

\cmidrule(lr){0-1} \cmidrule(lr){4-7}
\multirow{3}{*}{Accu. Poison} & ST & ~ &3& 85.4$\pm$3.54 & 70.4$\pm$9.16 & -15.4$\pm$6.2 \\
~ & GC & ~&7& 89.7$\pm$0.06 & 88.3$\pm$0.26 & -1.4$\pm$0.30 \\
~ & AT & ~&7& 89.6$\pm$0.21 & 88.7$\pm$0.20 & \textbf{-0.9$\pm$0.06}  \\
~ & \cellcolor{greyL}\textbf{DSC} & ~&\cellcolor{greyL}9& \cellcolor{greyL}\textbf{89.9$\pm$0.01} & \cellcolor{greyL}\textbf{88.8$\pm$0.26} & \cellcolor{greyL}-1.1$\pm$0.26  \\

\midrule[0.6pt]


\midrule[0.6pt]
CIFAR-100 & Defense & Accuracy: Start & Batch & Accuracy: +Poison & Accuracy: + Trigger & $\Delta$ \\
\midrule[0.6pt]
\multirow{3}{*}{Clean Oracle} & ST & \multirow{9}{*}{59.0} &-& 55.8 & 55.8 & - \\
~ & GC & ~ &-& 60.2 & 60.2 & - \\
~ & AT & ~ &-& 49.5 & 49.5 & - \\
~ & \cellcolor{greyL}\textbf{DSC} & ~ &\cellcolor{greyL}-& \cellcolor{greyL}55.0 & \cellcolor{greyL}55.0 & \cellcolor{greyL}-  \\

\cmidrule(lr){0-1} \cmidrule(lr){4-7}
\multirow{3}{*}{Accu. Poison} & ST & ~ & 3 & 42.9$\pm$2.74 & 32.6$\pm$2.84 & -10.3$\pm$0.29 \\
~ & GC & ~&4& 49.8$\pm$0.12 & 43.8$\pm$0.29 & -6.1$\pm$0.25 \\
~ & AT & ~&5& 47.7$\pm$0.25 & 44.4$\pm$0.21 & -3.2$\pm$0.42  \\
~ & \cellcolor{greyL}\textbf{DSC} & ~&\cellcolor{greyL}5& \cellcolor{greyL}\textbf{48.6$\pm$0.91} & \cellcolor{greyL}\textbf{45.4$\pm$1.39} & \cellcolor{greyL}\textbf{-3.2$\pm$0.65}  \\

\midrule[0.6pt]
\bottomrule[1.7pt]
\end{tabular}
\end{table*}

\paragraph{Poisoning attack.} After the burn-in phase in which the model is pre-trained for 40 epochs, we begin to inject the accumulative poison samples~\citep{pang2021accumulative}. Specifically, the crafted sample is generated by PGD under the $\ell_\infty$-norm constraint. Different from those regular poisoning generations, this poisoning attacker is allowed to intervene during training and tune the poisoning strategies dynamically with the model states. Since its poisoning target is the single-step drop of model accuracy, the poisoning effects of the secretly injected data will be accumulated and triggered in the final batch (termed as trigger batch). To simulate the monitor process in real-time data streaming, this final batch will be triggered when the model training loss is amplified by a monitored threshold of previous poison samples, and we adopted the same threshold values in~\citet{pang2021accumulative}.


\paragraph{Defense target.} To defend the attacks in real-time data streaming, there are three aspects that need to be considered. The first is the single-step drop in model accuracy. The second is the final accuracy, which reflects the overall defense effectiveness for the accumulative poisoning attacks. The third is the test accuracy of learning with clean samples, since we assume that the defender does not know when the poison sample is injected.  For the threshold schedule, we set $\mu=0.5, \tau=0.02$ for both CIFAR-10 and SVHN datasets, and $\mu=1.7, \tau=0.1$ for CIFAR-100 dataset.


\paragraph{Threshold adjustment.} The certain threshold for Memorization Discrepancy can be estimated based on the value of the controllable clean samples used in the burn-in phase. Similar to the tuning strategies in gradient clipping~\citep{pascanu2013difficulty,goodfellow2016deep}, we can set a lower threshold to conduct more correction steps for a conservative optimization for the online model. Based on the illustration in Figure~\ref{fig:reason_2} and the Property~\ref{pro:1}, the value of Memorization Discrepancy will increase as the model training. Thus, we adopt a fixed auxiliary model $\theta^*$ as the $\theta^{t-k}$ in discrepancy calculation. The threshold value will increase when training with the real-time data streaming from the untrusted sources~\citep{biggio2018wild} and it requires a dynamical threshold for filtering the clean samples with poison samples. To this end, we introduce the schedule as  $P = \mu + \tau * m$, where $P$ is our threshold, the initial value $\mu$, the dynamical growing interval $\tau$ is estimated by our controllable clean examples, and $m$ is the batch number. We provide further discussion on it in Appendix~\ref{app:condition}.

\begin{figure*}[t!]
    \centering
    \includegraphics[scale=0.175]{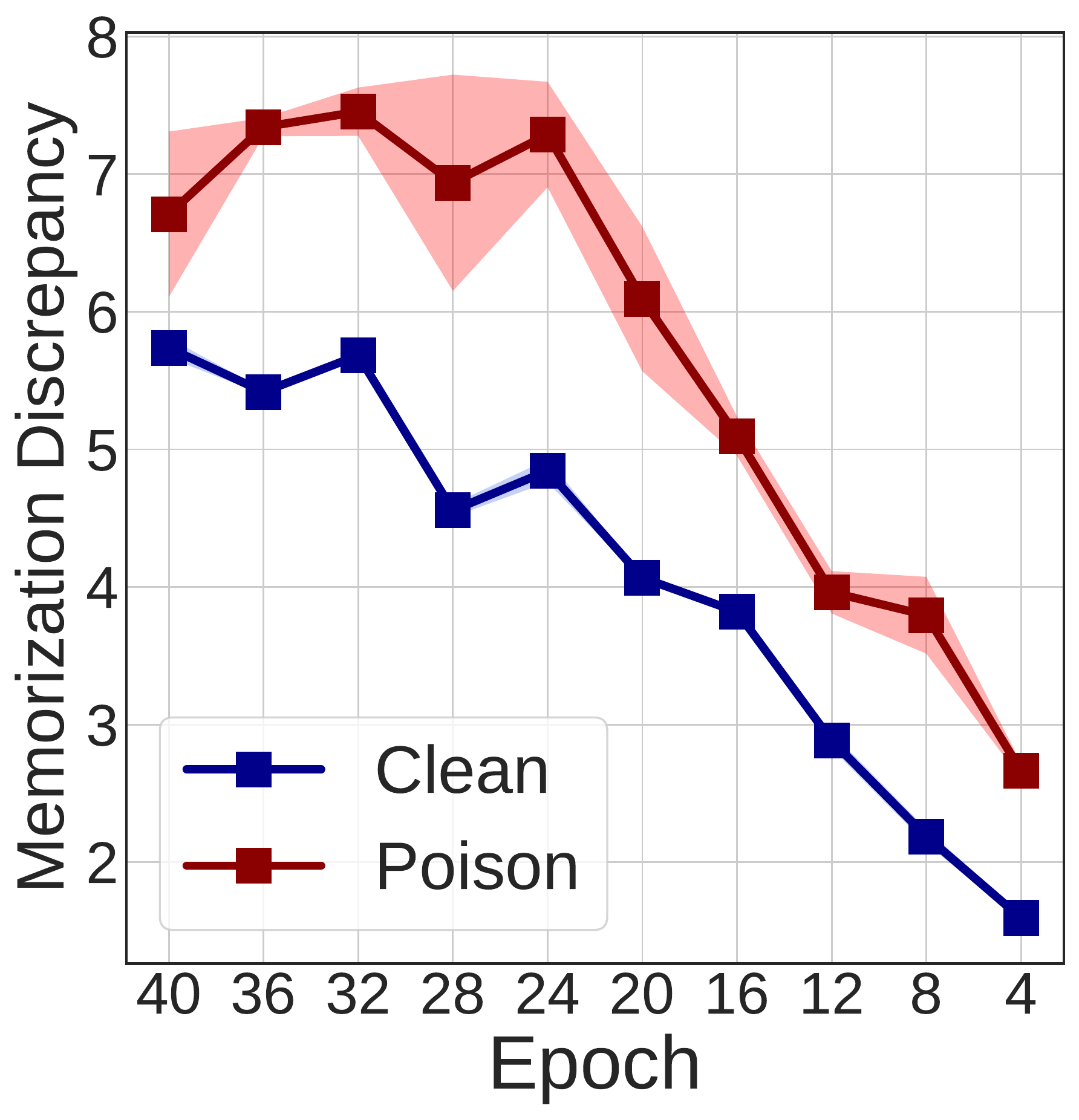}
    \includegraphics[scale=0.175]{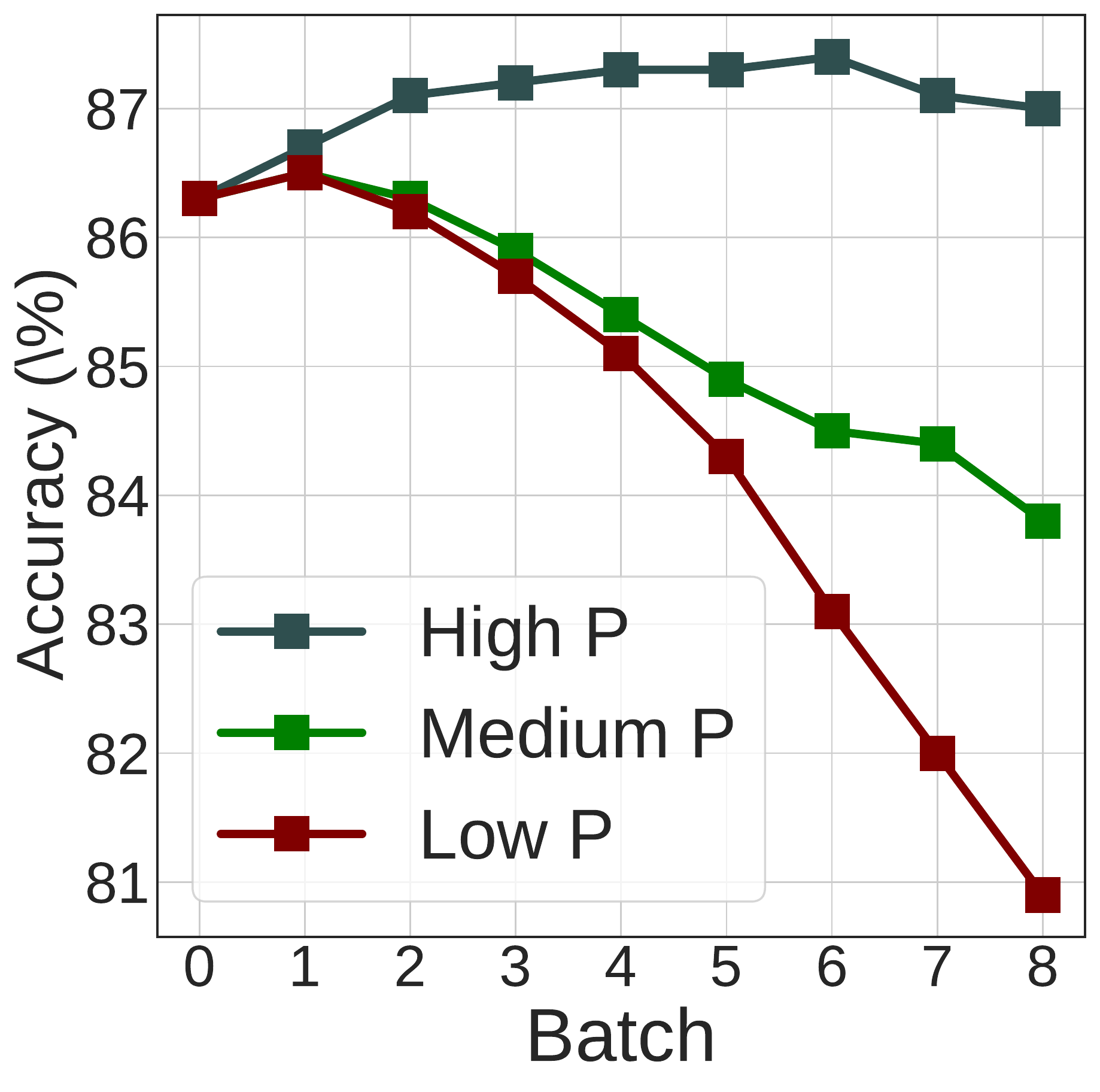}
    \includegraphics[scale=0.175]{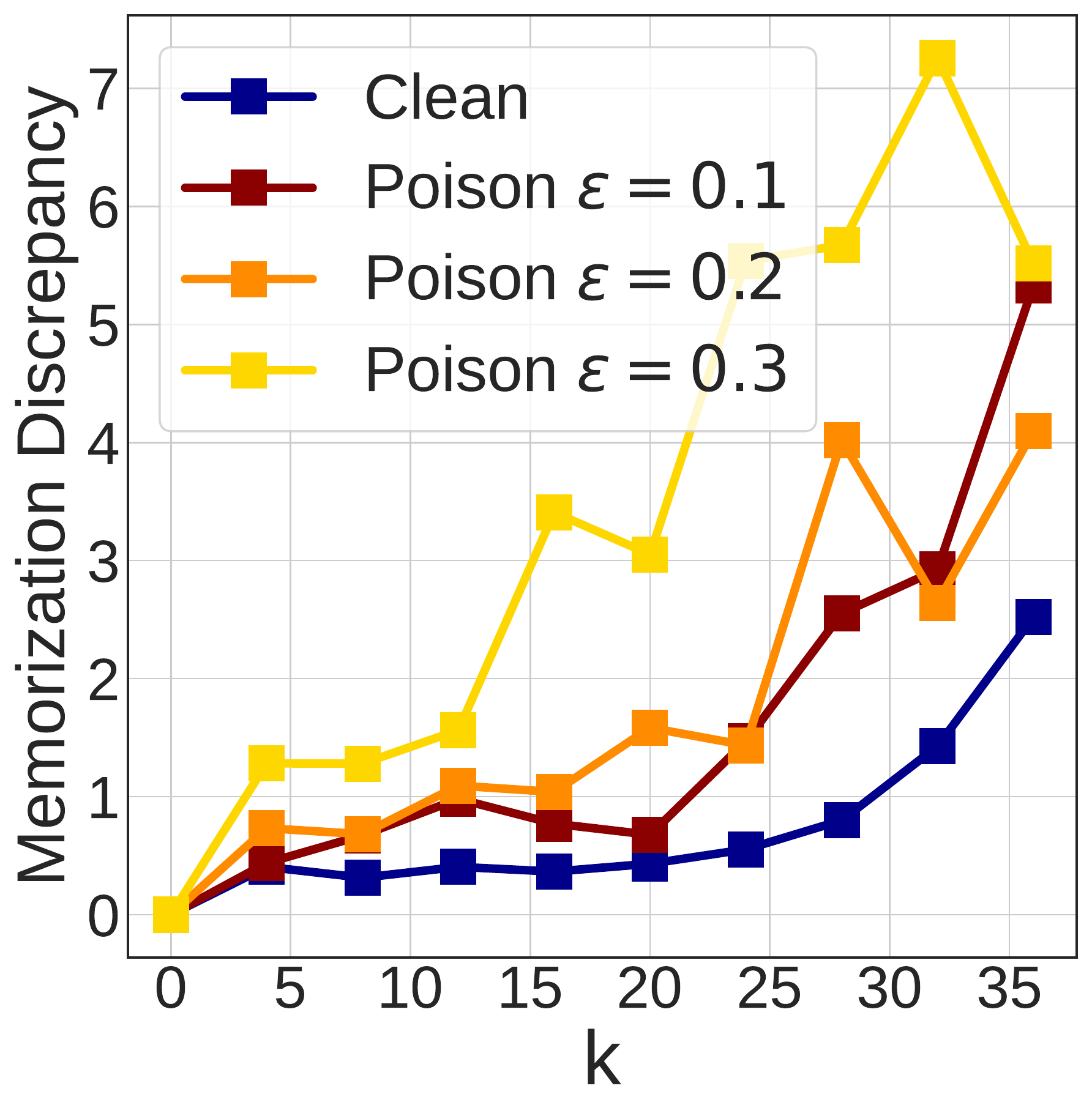}
    \includegraphics[scale=0.175]{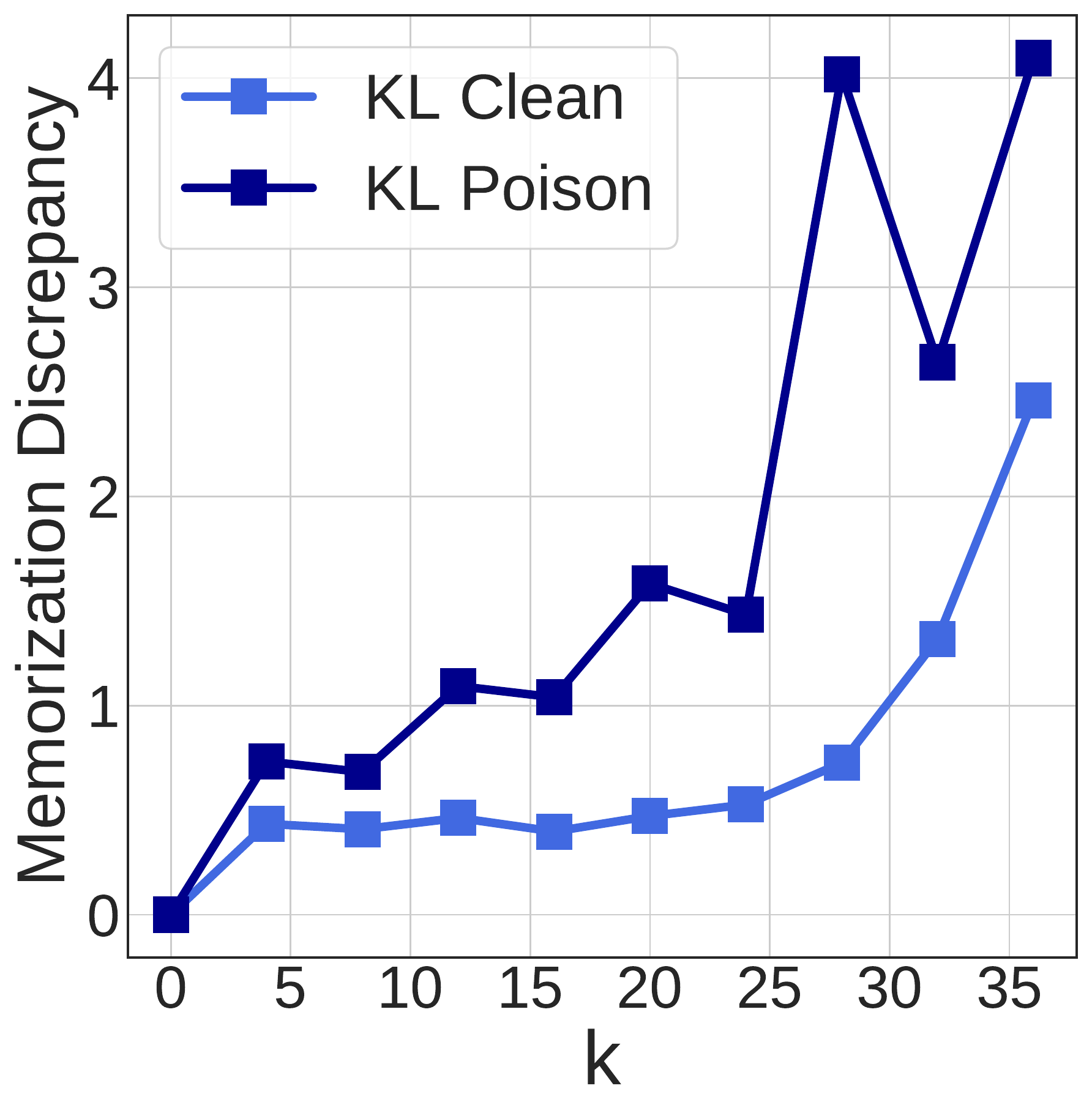}
    \includegraphics[scale=0.175]{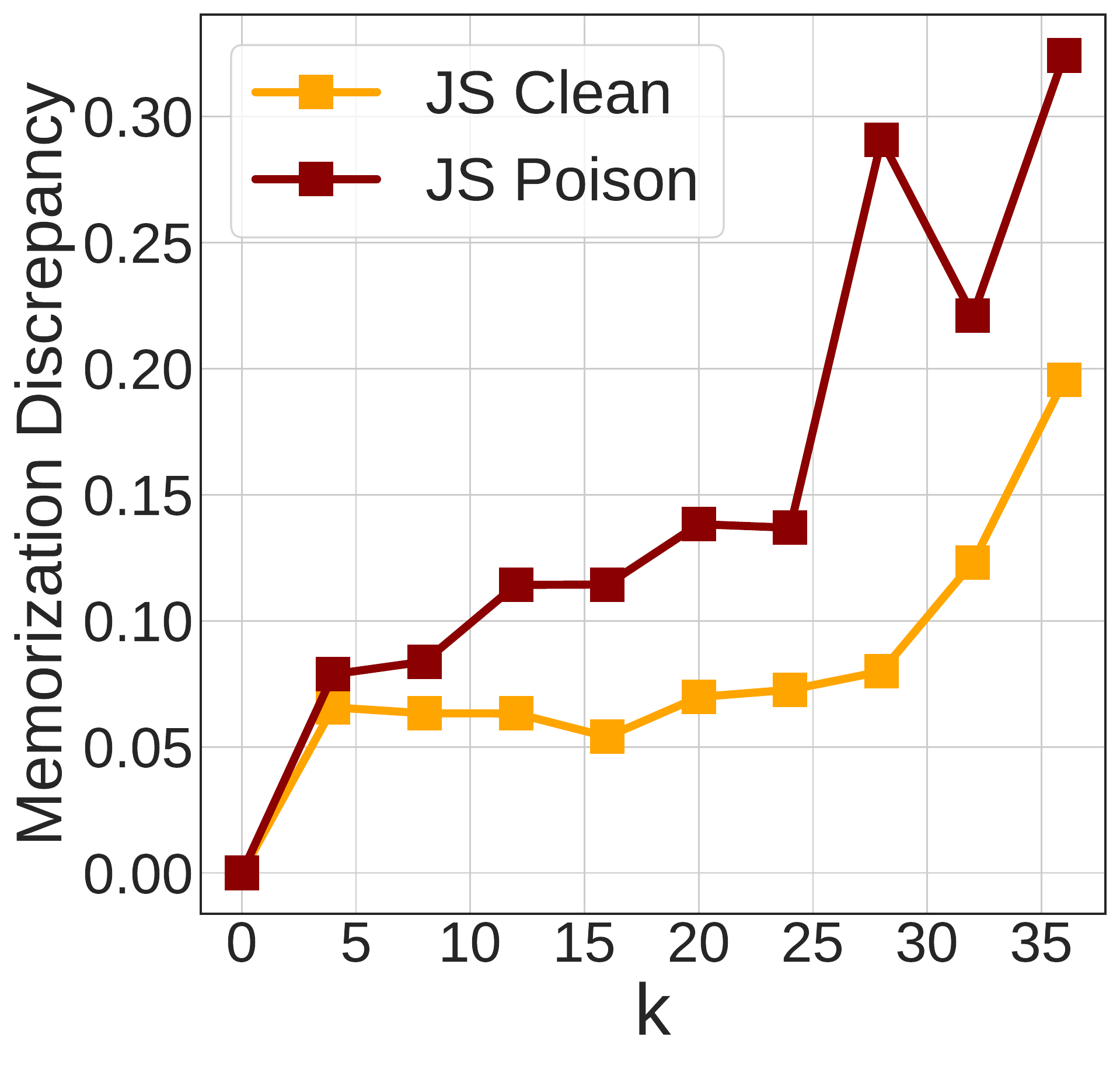}
    \vspace{-1mm}
    \caption{
    Ablation study. Left panel: Memorization Discrepancy between model $\theta^t$ in Eq.~\eqref{eq:memorization_discrepancy} and the model at Epoch 1; Left-middle panel: test accuracy with the threshold of different levels; Middle panel: Memorization Discrepancy under different poisoning capacities (imperceptibility); Right-middle and Right panel: Memorization Discrepancy corresponding to different discrepancy measurements.
    }
    \label{fig:ablation_study}
    \vspace{-2mm}
\end{figure*}

\subsection{Baseline Performance}
\label{sec:exp_baseline}

In this part, we compare our DSC with previous baseline methods (i.e., Standard Training (ST), Gradient Clipping~\citep{pascanu2013difficulty} (GC) and  Adversarial Training as Poisoning Defense~\citep{tao2021better} (AT)) on several benchmarked datasets to verify its effectiveness. In Table~\ref{table:exp_performance_cifar10}, we present the results of Clean Oracle to show the unaffected capacity of learning with clean samples and Accu. Poison to show the defense effectiveness against the secret poisoning attack. Specifically, we report four metrics according to different statuses: 1) Accuracy: +Poison, the accuracy after training with the secret poisoning batches; 2) Accuracy: +Trigger, the accuracy after training with the final trigger batch; 3) Batch, the number of batches before training loss is amplified to the monitored threshold; 4) $\Delta$, the accuracy drop of after the trigger batch. Since there are all clean samples in Clean Oracle, other accuracy values are equal to the final accuracy after training with 100 batches.

\begin{table}[t!]
\centering  
\caption{Test accuracy (\%) of adopting different adversarial optimization losses in our defense testing on the CIFAR-10 dataset. }
\renewcommand\arraystretch{0.9}
\label{table:exp_other_attack}
\small
\resizebox{0.48\textwidth}{!}{
\begin{tabular}{c|c|c|c|c}
\toprule[1.5pt]
\multicolumn{2}{c|}{Defense/Attack}  & PGD & KL (TRADES) & CW$_{\infty}$   \\
\midrule[0.6pt]
\midrule[0.6pt]
\multirow{5}*{DSC} & Start & \multicolumn{3}{c}{86.3} \\
\cmidrule{2-5}
~ & Acc. +Poison  & 81.4 & 80.3 & 81.8  \\

~ & Acc. +Trigger & 78.7 & 77.3 & 79.3  \\
~ & $\Delta$ & -2.69 & -3.04 & -2.54  \\
\midrule[0.6pt]
\bottomrule[1.5pt]
\end{tabular}}
\vspace{-2mm}
\end{table}

According to Table~\ref{table:exp_performance_cifar10}, we can find all the defensive methods can resist more batches than ST before triggering the pre-defined threshold. As for Accu. Poison, our DSC can achieve better accuracy consistently after going through the poisoning batches and the final trigger batches. Compared with GC, DSC and AT result in a smaller accuracy drop for the final single batch, it is much more important to those real-world applications since the model recovery with worse performance is a large cost~\citep{kairouz2019advances}. As for Clean Oracle, GC can achieve comparable or even higher accuracy than the pre-trained model since the clipped gradient also slow down the training process with a small gradient~\citep{pang2021accumulative}. Due to the indiscriminate correction, AT over-optimizes the clean samples and leads to much lower accuracy than the pre-trained model. In contrast, our DSC can still achieve comparable performance with ST through the selective correction of Memorization Discrepancy. Overall, the experiments running multiple times verified the general effectiveness of our DSC.

\subsection{Ablation Study and Further Discussion}
\label{sec:exp_ablation}

In this part, we conduct various experiments on CIFAR-10 to provide a thorough understanding of our presented Memorization Discrepancy and DSC. 
More ablations from different perspectives can be referred to in Appendix~\ref{sec:app_exp}.

\paragraph{Training status $\theta^t$.} In the left panel of Figure~\ref{fig:ablation_study}, we investigate the training status $\theta^t$ in Memorization Discrepancy. Specifically, we generate the accumulative poisoning attack based on the $\theta^t$ and calculate the discrepancy with the model checkpoint in Epoch 1. As can be seen, the mean values of Memorization Discrepancy on poison samples are consistently distinguishable from that on clean ones. This phenomenon provides us a chance to set just one auxiliary model for checking the dynamics instead of several historical models used in Figure~\ref{fig:reason} to fix the interval $k$.

\paragraph{Interval $k$.} In our previous illustration of Figure~\ref{fig:reason}, we visualize the discrepancy with the fixed interval $k$ (e.g., $k\in [4,36]$). The Memorization Discrepancy of both poison samples and clean samples increases and becomes more distinguishable with the increasing of the interval $k$. However, it is hard to use general criteria to choose the best interval or the previously analyzed training status. In the left panel of Figure~\ref{fig:ablation_study}, we adopt a dynamical interval $k$ which increases with the training status with a fixed auxiliary model (i.e., $\theta^{t-k}$) at Epoch 1. A similar trend with the distinguishable values can also be captured during the training process.

\begin{table}[t!]
\centering  
\caption{Test accuracy (\%) of considering adaptive attacks being aware of Memorization Discrepancy on the CIFAR-10 dataset. }
\renewcommand\arraystretch{0.9}
\label{table:exp_adaptive_attack}
\small
\resizebox{0.48\textwidth}{!}{
\begin{tabular}{c|c|c|c|c|c}
\toprule[1.5pt]
\multicolumn{2}{c|}{Constraint $\beta$}  & 0 & 0.05 & 0.1  & 0.2  \\
\midrule[0.6pt]
\midrule[0.6pt]
\multirow{5}*{ST} & Start & \multicolumn{4}{c}{86.3} \\
\cmidrule{2-6}
~ & Acc. +Poison  & 81.4 & 80.9 & 80.9  & 80.3\\

~ & Acc. +Trigger & 51.3 & 59.9 & 69.8 & 74.9 \\
~ & $\Delta$ & -30.04 & -20.97 & -11.12 & -5.43 \\
\midrule[0.6pt]
\bottomrule[1.5pt]
\end{tabular}
}
\vspace{-2mm}
\end{table}

\begin{table*}[!t]
    \centering
    \caption{Comparison of Memorization Discrepancy along the backtracking interval across different backbones.}
    \vspace{2mm}
    \resizebox{\textwidth}{!}{
    \begin{tabular}{c|c|c|c|c|c|c|c|c|c|c|c}
    \toprule[1.7pt]
    \midrule[0.6pt]
        Dataset & Model/Interval k & Discrepancy on & 4 & 8 & 12 & 16 & 20 & 24 & 28 & 32 & 36 \\
        \midrule[0.6pt]
        \multirow{6}{*}{CIFAR-10} & \multirow{2}{*}{ResNet} & Clean & 0.43444 & 0.40864 & 0.46287 & 0.39770 & 0.47189 & 0.52683 & 0.72566 & 1.31921 & 2.45922 \\
        ~ & ~ & Poison & 0.73276 & 0.68203 & 1.09318 & 1.03971 & 1.58336 & 1.43465 & 4.01915 & 2.64194 & 4.09608 \\
        \cmidrule{2-12}
        ~ & \multirow{2}{*}{VGG-11} & Clean & 0.52409 & 0.41497 & 0.51741 & 0.82906 & 0.94531 & 1.32099 & 1.97665 & 2.82950 & 5.69943 \\ 
        ~ & ~ & Poison & 0.33073 & 0.44917 & 0.41988 & 0.80057 & 0.89317 & 1.25389 & 3.65167 & 6.73917 & 14.53571 \\ 
        \cmidrule{2-12}
        ~ & \multirow{2}{*}{SmallCNN} & Clean & 0.42651 & 0.67719 & 0.48234 & 0.53239 & 0.46322 & 0.64273 & 1.05432 & 0.65214 & 1.89398 \\ 
        ~ & ~ & Poison & 1.04400 & 1.49517 & 1.08457 & 0.94815 & 1.91275 & 1.88436 & 3.73850 & 3.34128 & 5.79640 \\ 
        \bottomrule[1.7pt]
    \end{tabular}}
    \label{table:exp_model}
    \vspace{-1mm}
\end{table*}

\paragraph{Threshold $P$.} In the left-middle panel of Figure~\ref{fig:ablation_study}, we validate our proposed DSC with different levels of the threshold $P$. The intuition behind the threshold is to better utilize the distinguishable Memorization Discrepancies of poison samples and clean samples to filter out the specific samples. With a high threshold $P$, the test accuracy would not drop significantly when training with clean samples, since the sample correction can early-stop to avoid over-calibration. In contrast, using low threshold results in a severe accuracy drop since we conduct the indiscriminate correction. To further investigate the characteristics of the threshold $P$, we conduct additional experiments about Eq.~\eqref{eq:objective} with the threshold $P$ in Appendix~\ref{sec:app_exp}. To sum up, on the one hand, the results on natural data confirm that its discrepancy value shares a similar trend (e.g., increasing along the training process as indicated in Table~\ref{table:exp_model}) across different datasets, while the specific threshold $P$ needs different setups according to different training data. On the other hand, we can find that the performance of DSC can be stable during a specific range of the threshold value for identifying the poison data. 



\paragraph{Poisoning capacity.} In the middle panel of Figure~\ref{fig:ablation_study}, we check the effect of the poisoning capacity, i.e., the imperceptibility, on the value of Memorization Discrepancy. The imperceptibility is controlled by a parameter $\epsilon$ which corresponds to the manipulations. As the same in adversarial attacks~\citep{Goodfellow14_Adversarial_examples}, the larger $\epsilon$ indicates more perturbations and lower imperceptibility. The results show that the discrepancies between the two values of poison and clean samples also increase along with the enlargement of $\epsilon$. 

\paragraph{Different backbones.} To verify the generality of Memorization Discrepancy, we conduct experiments using different model structures (e.g., ResNet~\citep{he2016deep}, VGG-11~\citep{simonyan2014very}, SmallCNN~\citep{Zhang_trades}) on both clean and poison samples to check the discrepancy value in Table~\ref{table:exp_model}. The results confirm the phenomenon generally exists across different backbones in our experiments on the CIFAR-10 dataset, e.g., the difference between the Memorization Discrepancy on poison samples from that on clean samples is generally more distinguishable along with the enlargement of backtracking interval $k$.

\paragraph{Discrepancy measurement.} In the rest two panels of Figure~\ref{fig:ablation_study}, we also investigate another discrepancy measurement to check the relationship between poison and clean samples. Here we adopt the Jensen–Shannon divergence~\citep{dagan1997similarity} (JS) to calculate it and compare the results with that calculated on KL divergence. Both discrepancy measurements can capture a similar trend for their Memorization Discrepancy. Due to the different definitions for the measurement, there exists a difference on the scale of specific discrepancy values. The overall results show that the distinguishable relationship between two Memorization Discrepancies is not a consequence of 
a certain measurement but all of them, and the general intuition behind the discrepancy can also be captured by other measurements.

\paragraph{Discussion on the adaptive attacker.} In Tables~\ref{table:exp_other_attack} and~\ref{table:exp_adaptive_attack}, we consider different adversarial methods for generating imperceptible poison samples. The results demonstrate the robust effectiveness of DSC on different attacks. Furthermore, we also discuss a potential adaptive attacker~\citep{tramer2020adaptive} which is aware of Memorization Discrepancy, and try to incorporate it into its generation constraint to escape from identifying. However, the constraint can directly mitigate the poisoning effect that is reflected by the $\Delta$ in Table~\ref{table:exp_adaptive_attack}, where the poison sample is generated under a constraint controlled by $\beta$ with the historical model.


In addition, we also provide more explorations of Memorization Discrepancy and the DSC from different perspectives in Appendix~\ref{sec:app_exp}, including extra experiments of DSC in different learning and identification settings, the effects of different components, and corresponding discussions.




\section{Conclusion}
\label{sec:conclusion}

In this work, we investigated the accumulative poisoning attacks in real-time data streaming through the views of model dynamics. Through the exploration of the dynamic changes, we present a novel measure, i.e., {Memorization Discrepancy}, which is aware of the imperceptible manipulation added to the clean samples. Based on the novel measure, we propose the Discrepancy-aware Sample Correction method, which can selectively calibrate the poison samples. We present a comprehensive understanding of the discrepancy, and also various experiments to show the effectiveness of the DSC. We believe the underlying spirit of our Memorization Discrepancy, i.e., the dynamic changes in different models, can also motivate other defensive methods or applications. 

\section*{Acknowledgements}

JNZ and BH were supported by NSFC Young Scientists Fund No. 62006202, Guangdong Basic and Applied Basic Research Foundation No. 2022A1515011652, Alibaba Innovative Research Program, and HKBU CSD Departmental Incentive Grant. JCY was supported by the National Key R\&D Program of China (No. 2022ZD0160703),  STCSM (No. 22511106101, No. 22511105700, No. 21DZ1100100), 111 plan (No. BP0719010).

\clearpage

\bibliography{example_paper}
\bibliographystyle{icml2023}

\newpage
\appendix
\onecolumn

\section*{Appendix}

\section*{Reproducibility Statement}

We provide the repository of our source codes to ensure the reproducibility of main experimental results: \url{https://github.com/tmlr-group/Memorization-Discrepancy}. All experiments are conducted with multiple runs on NVIDIA GeForce RTX 3090 GPUs.

\section{Property Insights of Memorization Discrepancy}
\label{app:proof}

In this part, we provide the formal analysis of the property insights (e.g. Theorem~\ref{theorem:correlation} introduced in the main text) of our Memorization Discrepancy. To reveal the underlying mechanism of the proposed information measure, we start by revisiting the different targets of poisoning adversaries from the original training objective. Without specifying any detailed strategy for generating poison samples, the malicious objective generally targets to deteriorate the model performance on clean inputs, e.g., $\max\mathcal{L}(S; \theta^{*})$, where $\theta^*$
is assumed to be well-trained in the given samples. 

However, considering a model that is updated with clean training data, it gradually approaches to different side (e.g., $\min\mathcal{L}(S; \theta^{*})$) of the previous target. Based on that, we can naturally make the following assumption about the sample-wise discrepancy with the difference between the current and target loss value,

\begin{assumption}\label{ass:correlation}
Let $f(x;\theta^t)$ denote the model dynamics about the sample $x$ and at round $t$, k denotes the interval rounds for backtracking. Considering the ordinary objective $\min\mathcal{L}(S; \theta^{*})$ and the poisoning objective $\max\mathcal{L}(S; \theta^{*})$ with the clean inputs set $S$ and a poisoned set $P$, we have,
\begin{equation}
    \mathbb{D}(f(\hat{x}(\theta^{t}); \theta^{t}), f(x; \theta^{t})) \propto \max\mathcal{L}(S; \theta^{*})-\mathcal{L}(S; \theta^{t}), \quad s.t.\; \theta^* \in \arg\min_{\theta}\mathcal{L}(P;\theta)
\end{equation}
\end{assumption}

Intuitively, it indicates that the model output of the poison sample will be much more different from that of the clean sample when the model is well-trained on the clean training data (i.e., has a small loss value on clean set $S$). In other words, the poisoning adversary needs a larger effort to achieve the malicious target, since the model has already performed well on the clean set. 

Here we present the Theorem~\ref{theorem:correlation} again (i.e., the same as the following Theorem~\ref{theorem:correlation_part2}) to start the analysis and the further discussion on the critical property of the defined Memorization Discrepancy.

\begin{theorem}\label{theorem:correlation_part2}
Let $f(x;\theta^t)$ denote the output about the sample $x$ at epoch $t$, k denotes the interval rounds, and $S$ denotes a clean dataset. Considering the opposite between objective $\min\mathcal{L}(S,\theta^*)$ and the poisoning objective $\max\mathcal{L}(S,\theta^*)$ where $\theta^*$ is the well-trained model respectively, there exists a learning period where we have,
\begin{equation}\label{eq:corr}
    \mathbb{D}(f(\hat{x}(\theta^{t}); \theta^{t}), f(\hat{x}(\theta^{t-k}); \theta^{t-k}))-\mathbb{D}(f(x; \theta^{t}), f(x; \theta^{t-k})) \propto \mathcal{L}(S; \theta^{t-k})-\mathcal{L}(S; \theta^{t}),
\end{equation}
\end{theorem}

\begin{proof}[proof of Theorem~\ref{theorem:correlation_part2}.]
The correlation of the two parts in Eq.~\eqref{eq:corr} can be formulated in the following. 

Given the two approximate optimization targets as,
\begin{align}
\begin{split}
    &\theta^t - \beta\nabla_{\theta^t}\mathcal{L}(f(x;\theta^t), y) \rightarrow \min\mathcal{L}(S;\theta^{t+1})\\
    &\theta^t - \beta\nabla_{\theta^t}\mathcal{L}(f(\hat{x}(\theta^t);\theta^t), y) \rightarrow \max\mathcal{L}(S;\theta^{t+1}),
\end{split}
\end{align}
we can obtain the correlation about these two opposite target parts as,
\begin{align}
    \mathbb{D}(\nabla_{\theta}\mathcal{L}(f(x), y, \theta^t),\nabla_{\theta}\mathcal{L}(f(\hat{x}(\theta^t)), y, \theta^t))\propto \max\mathcal{L}(S;\theta^{*})-\mathcal{L}(S;\theta^{t}),
\end{align}
where $\theta^* \in \arg\min_{\theta}\mathcal{L}(P;\theta)$ is the model parameter well-trained on the poison samples. Similarly, we can also get the following equation via backtracking,
\begin{align}
    \mathbb{D}(\nabla_{\theta}\mathcal{L}(f(x), y, \theta^{t-k}),\nabla_{\theta}\mathcal{L}(f(\hat{x}(\theta^{t-k})), y, \theta^{t-k}))\propto \max\mathcal{L}(S;\theta^{*})-\mathcal{L}(S;\theta^{t-k}),
\end{align}
Since the two gradient parts share the same anchor of model parameter and the labels, we can get the consistent relationship that similar to Assumption~\ref{ass:correlation} as, 
\begin{align}
\begin{split}\label{eq:dis_dis}
    \mathbb{D}(\mathcal{L}(f(x), y, \theta^t),\mathcal{L}(f(\hat{x}(\theta^t)), y, \theta^t))&\propto \max\mathcal{L}(S;\theta^{*})-\mathcal{L}(S;\theta^{t}),\\
    \mathbb{D}(\mathcal{L}(f(x), y, \theta^{t-k}),\mathcal{L}(f(\hat{x}(\theta^{t-k})), y, \theta^{t-k}))&\propto \max\mathcal{L}(S;\theta^{*})-\mathcal{L}(S;\theta^{t-k}),
\end{split}
\end{align}
By accumulate the approximate discrepancy correlation with historical models, we can introduce the discrepancy considering the samples of same type,
\begin{align}
\begin{split}
    \mathbb{D}(f(x; \theta^{t-k})&, f(x; \theta^{t})),\\
    \mathbb{D}(f(\hat{x}(\theta^{t-k}); \theta^{t-k})&, f(\hat{x}(\theta^{t}); \theta^{t})),
\end{split}
\end{align}
Using the above discrepancy on model outputs, we can explicitly obtain the formulation by constructing discrepancy for each side of Eq.~\eqref{eq:dis_dis},
\begin{align}
\mathbb{D}(f(\hat{x}(\theta^{t}); \theta^{t}), f(\hat{x}(\theta^{t-k}); \theta^{t-k}))-\mathbb{D}(f(x; \theta^{t}), f(x; \theta^{t-k})) \propto \mathcal{L}(S; \theta^{t-k})-\mathcal{L}(S; \theta^{t}).
\end{align}
This gives the property insights on the dynamics of the Memorization Discrepancy. 

\end{proof}

In summary, the above correlation of the Memorization Discrepancy and the loss discrepancy between two different model stages is built on the high-level target discrepancy. The Eq.~\eqref{eq:corr} indicates that we can enlarge the discrepancy of the two information values on clean and poison samples via construct the proper loss discrepancy. Backtracking the historical model can serve this goal since it naturally reflects the dynamical behavior of learning with the ordinary objective. 

As enlarging the backtracking interval $k$, the loss discrepancy is further enlarged. The corresponding poison and clean samples become more distinguishable on the basis of our proposed information value. It is consistent with previous empirical results in Figures~\ref{fig:motivation} and~\ref{fig:reason}. This property exactly meets our requirement described in Section~\ref{sec:memo_motivation}, i.e., to gain useful information about imperceptible poison samples via model dynamics. To be specific, as presented in Figure~\ref{fig:motivation}, the two distributional statics become more distinguishable when we construct the discrepancy by involving the historical models. Similar in Figure~\ref{fig:method}, the Memorization Discrepancy of poison samples is larger than that of clean samples. It is general and has no specific assumption about the poisoning generation.

From the new perspective, the proposed Memorization Discrepancy can accumulate the target-level discrepancy in model dynamics for better distinguishing poison samples from clean samples, which is appropriate to figure out the accumulative poisoning attacks since the adversary try to spread the perceived risk over a single round of optimization.




\section{Further discussion about the Demonstration in Figure~\ref{fig:reason_2}}
\label{app:further_discuss}

\begin{figure}[h!]
    \centering
    \includegraphics[scale=0.175]{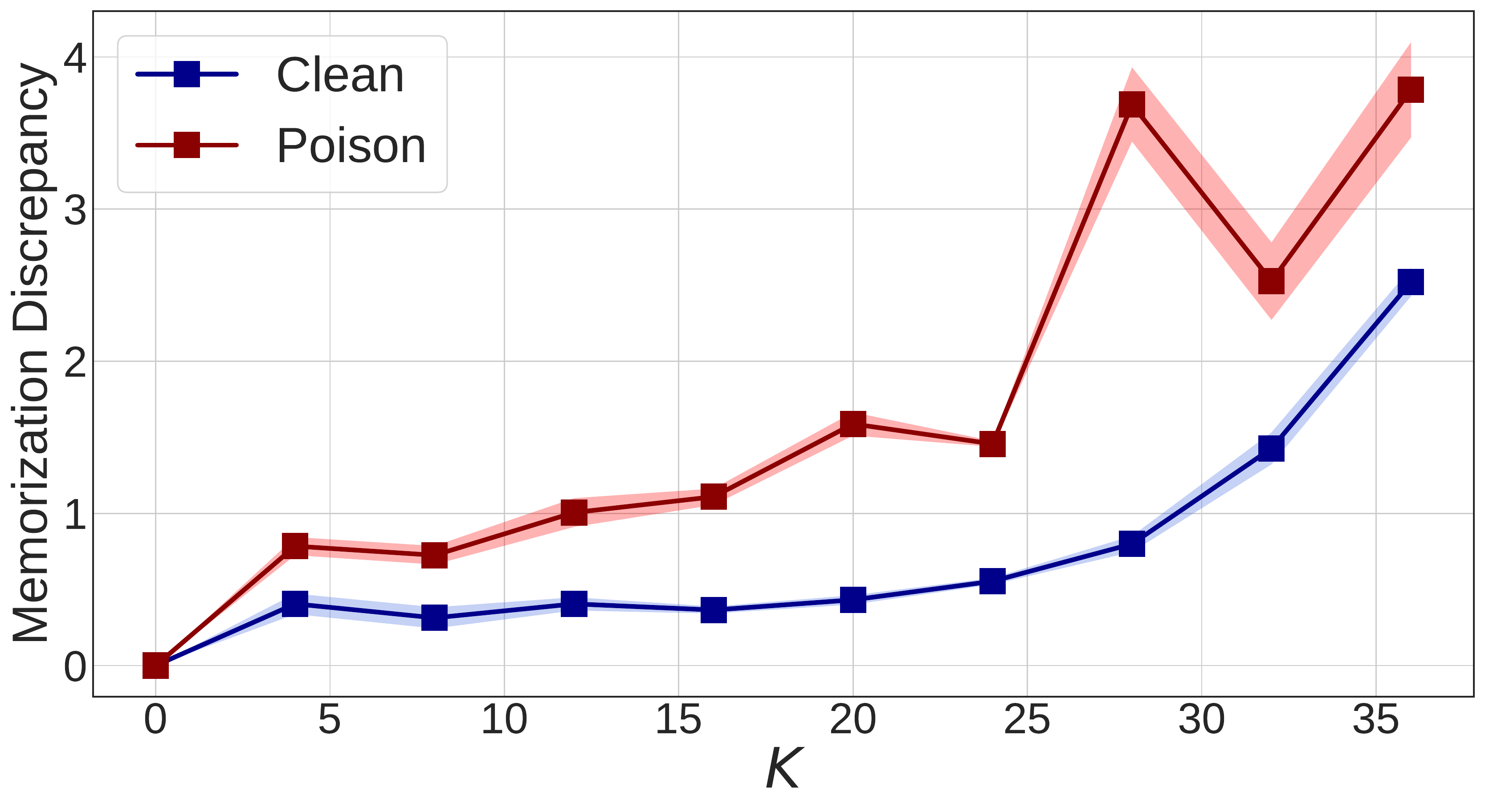}
    \includegraphics[scale=0.175]{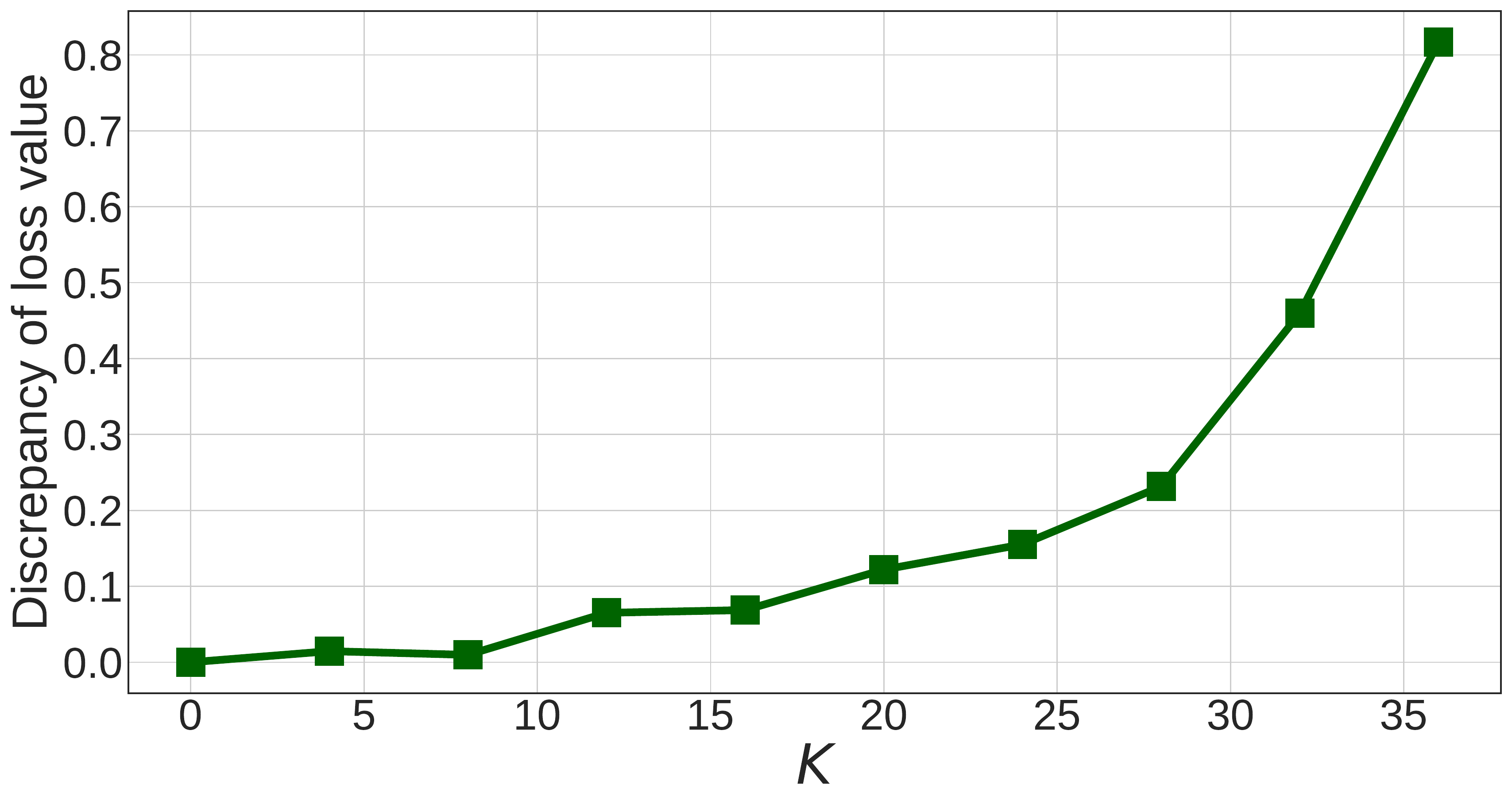}
    \includegraphics[scale=0.175]{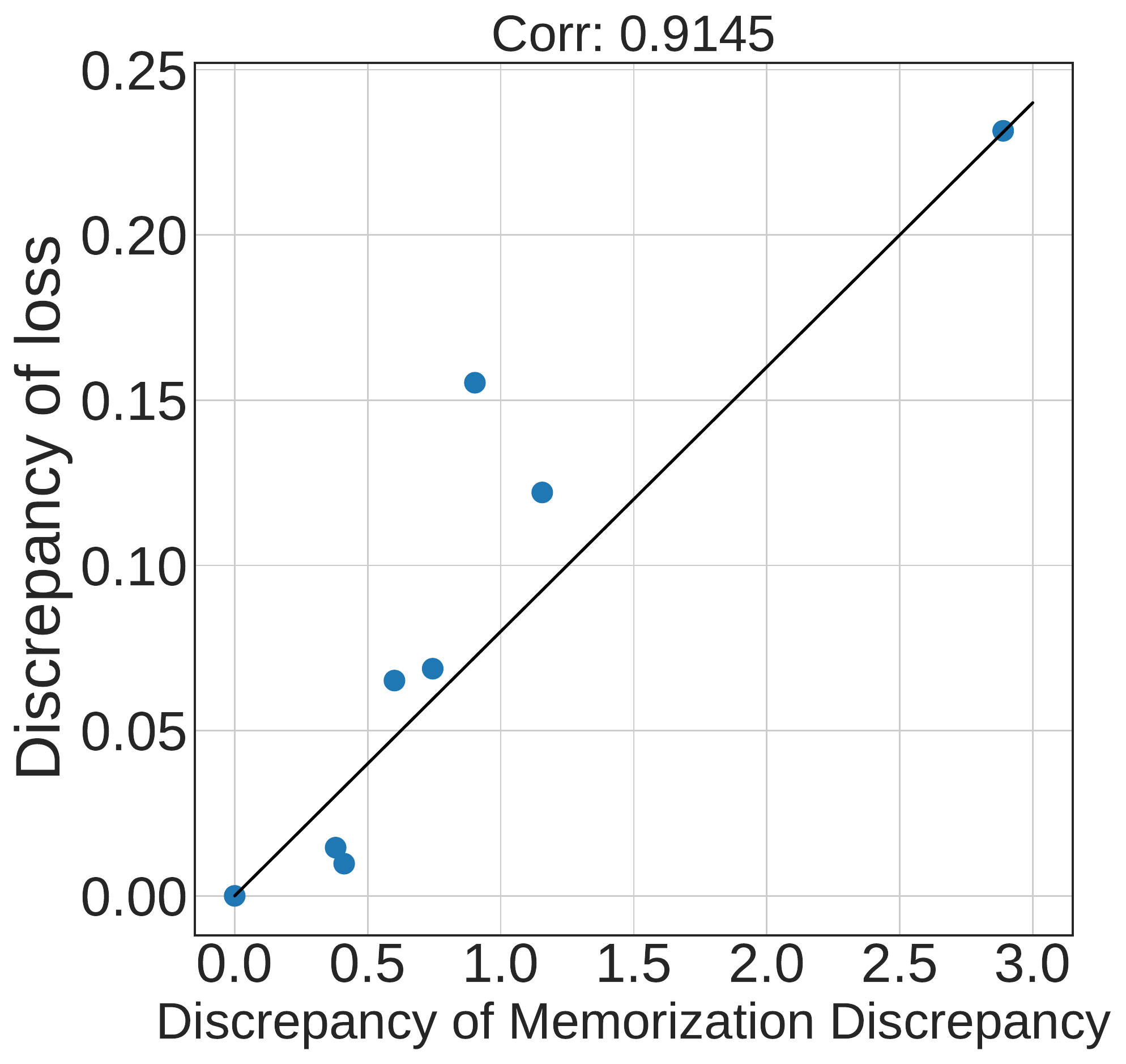}
    \caption{
    Empirical verification about the property insights on the Memorization Discrepancy.
    }
    \label{fig:empirical_jus_for_theo}
\end{figure}

In this part, we provide the empirical verification of the previous property insights draw from the discrepancy of model dynamics. On the same simulation experiments on CIFAR-10, we check the Memorization Discrepancy and the corresponding loss discrepancy between the current and historical model in Figure~\ref{fig:empirical_jus_for_theo}. It can be found that during the stable training phase (e.g. back from Epoch 40 to Epoch 5) the correlation between the discrepancy in model output and loss values are proportional. In the early stage, we can find some inconsistent relationships exist, we attribute the possible reason to the unstable optimization which can not accurately reflect the relative distance between the malicious target (training with poison samples) and the ordinary target (training with clean samples).

\begin{figure}[h!]
    \centering
    \subfigure[Illustration of Clean and Poison Samples ($\epsilon=0.064$)]{
    \includegraphics[scale=0.17]{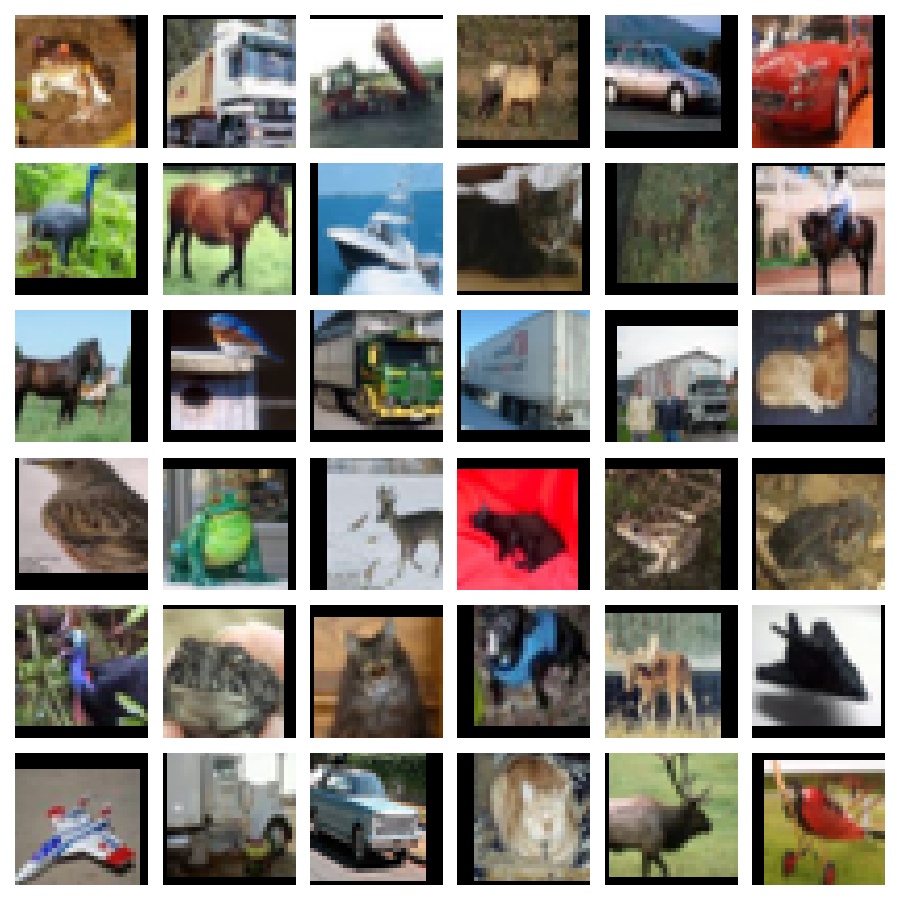}
    \hspace{0.05in}
    \includegraphics[scale=0.17]{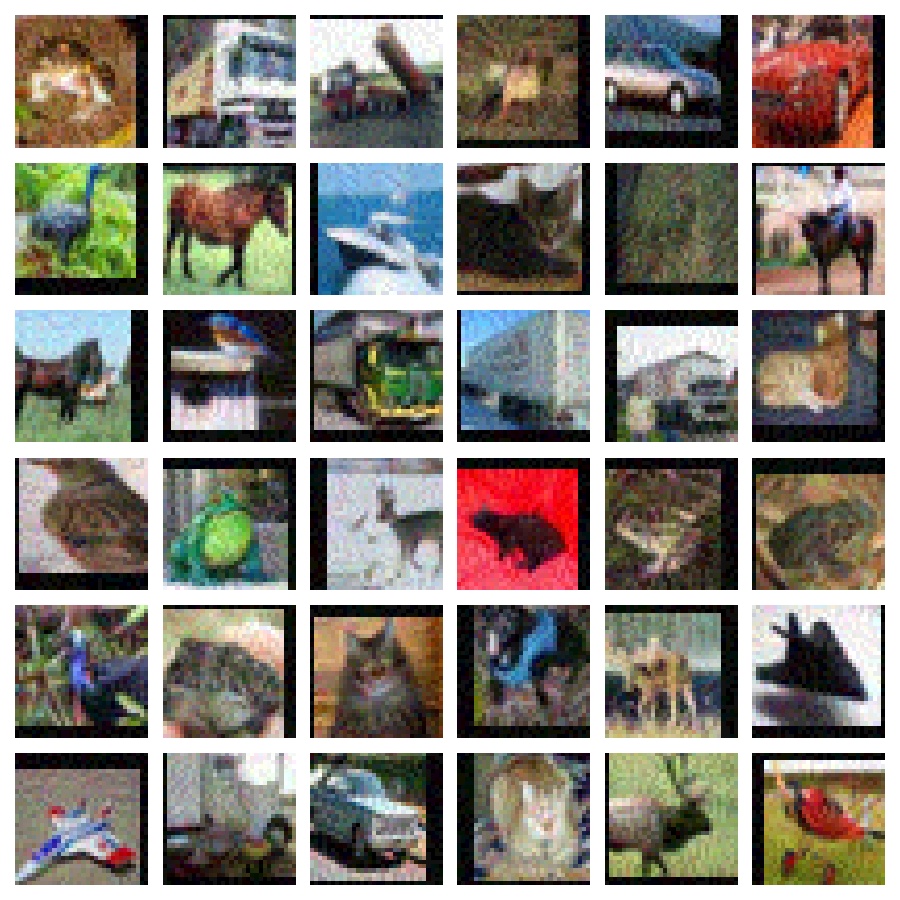}
    \label{fig1:a}
    }
    \subfigure[Empirical Unawareness of Accumulation]{
    \includegraphics[scale=0.18]{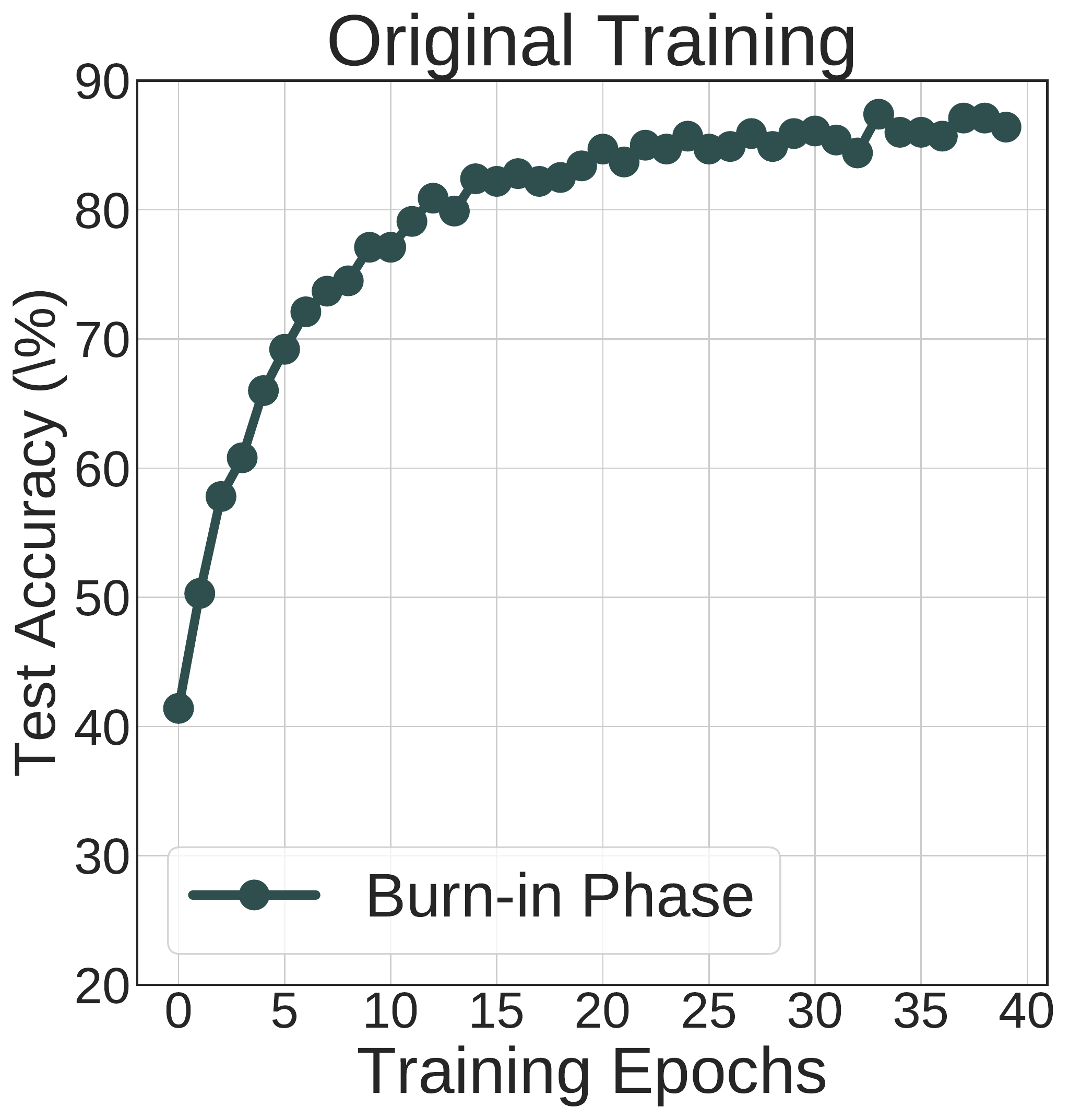}
    \hspace{0.05in}
    \includegraphics[scale=0.18]{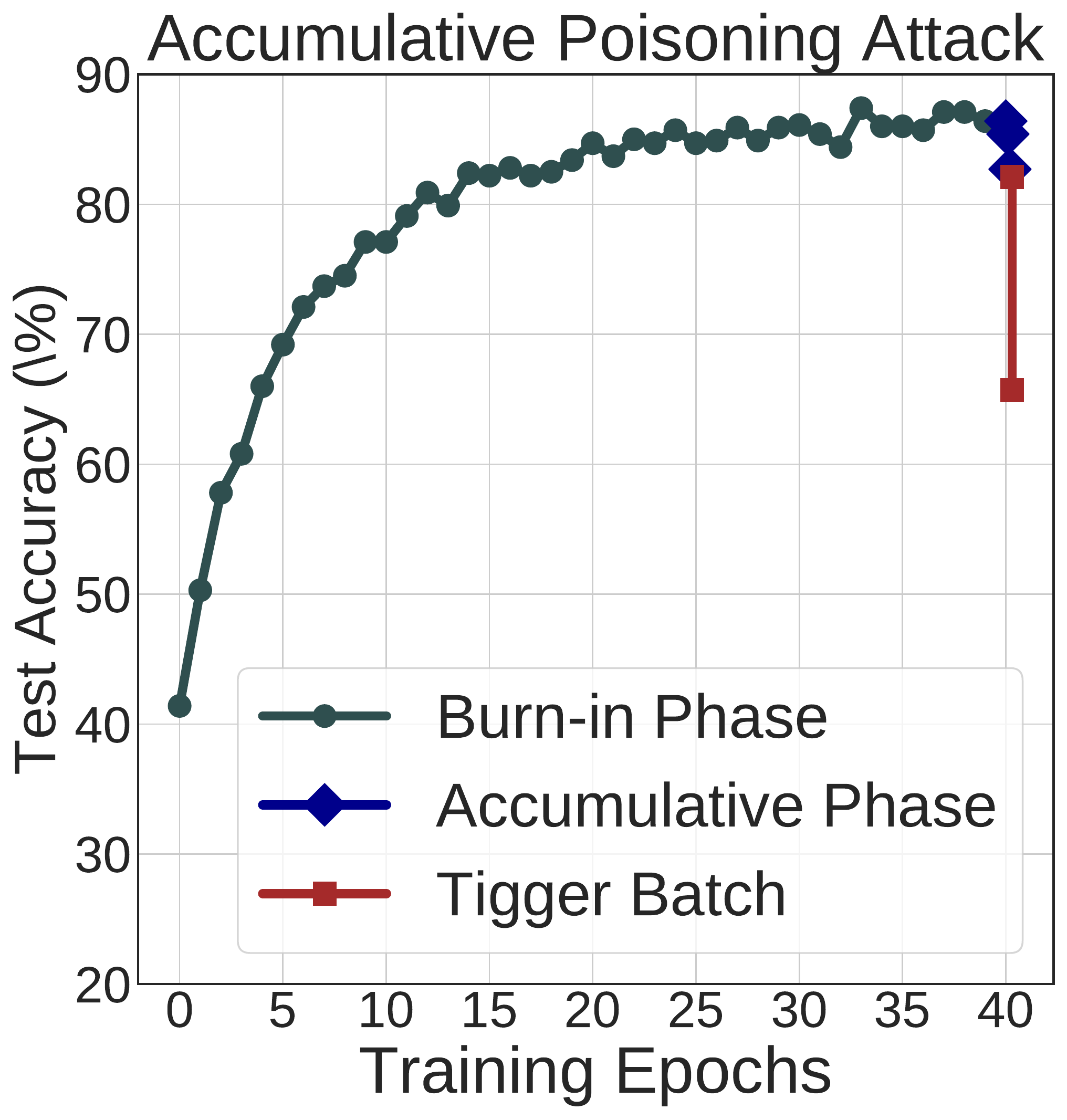}
    \label{fig1:b}
    }
    \caption{
    Visualization of the empirical imperceptibility on accumulative poisoning attack using CIFAR-10 dataset. Except for the visual-level imperceptibility, the accumulative poison samples will not induce a significant accuracy drop which may be caught by a simple monitor.
    }
    \label{fig:empirical_unawareness}
\end{figure}

\section{Additional details and explorations}
\label{sec:app_exp}

In this section, we provide completed information about the accumulative poisoning attacks with extra details of algorithm implementation, as well as extra experimental results. 

\subsection{Details about Accumulative Poisoning Attack}
\label{app:poison_details}

In this part, we describe the details of the Accumulative Poisoning Attack. Let $S_{train}$ be the clean training set and $S_{val}$ be the separate validation set, an attacker will poison the $S_{train}$ into a poisoned $\mathcal{P}(S_{train})$. Except for the original malicious objective as,
\begin{equation}
\label{eq:original_malicious_ob}
    \max_{\mathcal{P}}\mathcal{L}(S_{val};\theta^*),\quad s.t.\; \theta^*\in \arg\min_{\theta}\mathcal{L}(\mathcal{P}(S_{train});\theta),    
\end{equation}
the accumulative poisoning attack utilizes the characteristics of online learning to inject the poison samples. Hence, the real-time malicious objective is formulated as follows at the training round $T$,
\begin{equation}
\label{eq:real_time_malicious_ob}
    \max_{\mathcal{P}}\mathcal{L}(S_{val};\theta^{T+1}),\quad s.t.\; \theta^{T+1}=\theta^{T}-\beta\nabla_{\theta}\mathcal{L}(\mathcal{P}(S_{train});\theta^{T}),
\end{equation}
where $\beta$ is the learning rate of gradient descent. 

By expanding the previous malicious objective, it can be rewritten as,
\begin{equation}
\label{eq:real_time_malicious_ob_re}
    \min_{\mathcal{P}}\nabla_{\theta}\mathcal{L}(S_{val};\theta^{T})^\top\nabla_{\theta}\mathcal{L}(\mathcal{P}(S_{T});\theta^{T}),
\end{equation}

Based on Eq.~\eqref{eq:real_time_malicious_ob_re}, \citet{pang2021accumulative} introduce the accumulative phase $\mathcal{A}$ to make the model parameter at round $T$ obtained after the accumulative phase be more sensitive and 
fragile to the poisoning. So the overall objective can be formulated as,
\begin{equation}
\label{eq:real_time_malicious_ob_re_accu}
    \min_{\mathcal{P},\mathcal{A}}\nabla_{\theta}\mathcal{L}(S_{val};\mathcal{A}(\theta^{T}))^\top\nabla_{\theta}\mathcal{L}(\mathcal{P}(S_{T});\mathcal{A}(\theta^{T})),
\end{equation}
and the perturbed data batch $\mathcal{A}(S_t)$ can be crafted by solving a first-order expansion of the real-time learning update,
\begin{equation}
\label{eq:real_time_malicious_ob_re_accu_expand}
    \max_{\mathcal{P},\mathcal{A}_t}\nabla_{\theta}\mathcal{L}(\mathcal{A}_t(S_{t});\theta^t)^\top\left[ \nabla_{\theta}\mathcal{L}(S_{t};\theta^{t})+\lambda\cdot\nabla_{\theta}(\nabla_{\theta}\mathcal{L}(S_{val};\mathcal{A}(\theta^{T}))^\top\nabla_{\theta}\mathcal{L}(\mathcal{P}(S_{T});\mathcal{A}(\theta^{T})))\right],
\end{equation}
which equals to,
\begin{equation}
\label{eq:real_time_malicious_ob_re_accu_expand_v2}
    \max_{\mathcal{P},\mathcal{A}_t}\nabla_{\theta}\mathcal{L}(\mathcal{A}_t(S_{t});\theta^t)^\top\left[ \underbrace{\nabla_{\theta}\mathcal{L}(S_{t};\theta^{t})}_\text{keep accuracy}+\lambda\cdot\underbrace{\nabla_{\theta}(\nabla_{\theta}\mathcal{L}(S_{val};\theta^{T})^\top\nabla_{\theta}\mathcal{L}(\mathcal{P}(S_{T});\theta^{T}))}_\text{accumulating poisoning effects for the trigger batch}\right],
\end{equation}

Following~\citet{pang2021accumulative}, we adopt the burn-in phase that pretrains the model for 40 epochs. Then we begin to inject the accumulative poison samples~\citep{pang2021accumulative}. Specifically, the crafted sample is generated by PGD~\citep{Madry_adversarial_training} under the $\ell_\infty$-norm constraint. Since its poisoning target is the single-step drop of model accuracy, the poisoning effects of the secretly injected data will be accumulated and triggered in the final batch (termed as trigger batch). To simulate the monitor process in real-time data streaming, this final batch will be triggered when the training loss is amplified by a threshold of previous poison samples (the same as threshold in~\citet{pang2021accumulative}).

\subsection{Algorithm Realization of DSC}

Here we provide the detailed realization of our proposed Discrepancy-aware Sample Correction in Algorithm~\ref{alg:dsc}.

\begin{algorithm}[t!]
   \caption{Discrepancy-aware Sample Correction (DSC)}
   \label{alg:dsc}
   {\bf Input:} data streaming $S = \{(\bx_i, y_i) \}^{n}_{i=1}$, learning rate $\eta$, number of epochs $T$, batch size $m$, number of batches $M$, data ${x}\in \cX$, label $y \in \cY$, victim model $\theta$, loss function $\ell$, PGD step $K$, perturbation bound $\epsilon$, step size $\delta$, projection opt. $\Pi$, Memorization Discrepancy threshold $P$, auxiliary historical model $\theta^*$. \\
   {\bf Output:} model $\theta^{T}$;
\begin{algorithmic}[1]
  \FOR{epoch $= 1$, $\dots$, $T$}
    \FOR {mini-batch $=1$, $\dots$, $M$ }
    \STATE Sample a mini-batch $\{(\bx_i, y_i) \}^{m}_{i=1}$ from $S$
        \FOR{$i = 1$, $\dots$, $m$ (in parallel) }
         \STATE \colorbox[RGB]{239,240,241}{Obtain the corrected sample $\bxtidle_i$ of $x_i$:}
         \STATE $\Tilde{{x_i}} \gets {x_i}, n=1$
    \WHILE{$\mathbb{D}(f(\Tilde{{x_i}}; \theta),f(\Tilde{{x_i}}; \theta^*))>P$ \text{and} $n<K$}
   \STATE $\Tilde{{x_i}} \gets \Pi_{\mathcal{B}[{x_i},\epsilon]}\big( \Tilde{{x_i}} - \delta \cdot \text{sign}(\nabla_{\Tilde{{x_i}}} \ell(f(\Tilde{{x_i}}), y))  \big)$
   \STATE $n=n+1$
    \ENDWHILE
         \ENDFOR
    \STATE $\mathbf{\theta} \gets \mathbf{\theta} - \eta \nabla_{\mathbf{\theta}} \ell(f_{\mathbf{\theta}}(\bxtidle_i), y_i)$
  \ENDFOR
 \ENDFOR
\end{algorithmic}
\end{algorithm}

\subsection{Comparison of Training Time}

In this part, we check the training time of different defenses with the accumulative poisoning attacks. For our proposed DSC which incorporates Memorization Discrepancy in identifying the incoming samples on the data streaming, the cost of the backtracked historical model mainly lies in the storage to save the historical model checkpoints. However, considering that in practice it is common to save the checkpoints regularly during training, such cost of our method is acceptable. Here we report the training time of each method in Table~\ref{table:training_time} to give a more intuitive comparison. The experiment setups keep the same as Table~\ref{table:exp_performance_cifar10}. According to the results, we can see that DSC requires slightly more time than other methods.

\begin{table}[!ht]
    \centering
    \caption{Comparison of per mini-batch training and the accuracy (\%) after the poisoning attack across different datasets.}
    \vspace{2mm}
    \label{table:training_time}
    \begin{tabular}{c|c|c|c}
    \toprule[1.7pt]
    \midrule[0.6pt]
    Dataset  &  Method & Training Time (seconds) & Acc. +Tigger \\
    \midrule[0.6pt]
    \multirow{4}{*}{CIFAR-10} & ST & 0.0206 & 50.4 \\
    ~ & GC & 0.0253 & 75.1\\
    ~ & AT &0.3735 & 75.3\\
    ~ & DSC &0.3241  & 77.3\\
    \midrule[0.6pt]
    \multirow{4}{*}{CIFAR-100} & ST & 0.0189 & 32.6 \\
    ~ & GC & 0.0314 & 43.8\\
    ~ & AT & 0.3732 & 44.4 \\
    ~ & DSC & 0.2948 & 45.4\\
    \midrule[0.6pt]
    \multirow{4}{*}{SVHN} & ST & 0.0198 & 70.4  \\
    ~ & GC & 0.0250 & 88.3\\
    ~ & AT & 0.3630  & 88.7 \\
    ~ & DSC &0.0689 & 88.8\\
    \bottomrule[1.7pt]
\end{tabular}
\end{table}

\subsection{Extra Validation of the Threshold $P$}

We conduct more experiments to present Eq.~\eqref{eq:objective} of our DSC. First, we check the trend of discrepancy across different datasets in Table~\ref{table:exp_trend_data}, and the results on natural data confirm that it shares a similar trend (increasing along the training process) across different datasets, while the specific threshold $P$ needs different setups due to different datasets. Second, we conduct experiments about hyperparameter tuning in the proper thresholds $P$ and compare the performance on clean oracle and poisoned training data, respectively. Note that, in Table~\ref{table:exp_range_p}, there are no further results when the test accuracy drops to a certain level (indicated by "-"). The results show that the performance is stable during the specific range of the threshold $P$ value. To be more specific. Similar to the hyperparameters of gradient clipping and adversarial training, the threshold $P$ can not be too high to lose control on correcting the poisoned data and also needs to be not too low to induce over-calibration on the clean sample. It is consistent with the results presented in the left-middle panel of Figure~\ref{fig:ablation_study}.

\begin{table}[!ht]
    \centering
    \caption{The discrepancy trend of batch data across different datasets.}
    \label{table:exp_trend_data}
    \vspace{2mm}
    \resizebox{\textwidth}{!}{
    \begin{tabular}{c|c|c|c|c|c|c|c|c|c|c|c|c}
    \toprule[1.7pt]
    \midrule[0.6pt]
        Dataset & Discrepancy on/Batch & 1 & 2 & 3 & 4 & 5 & 6 & 7 & 8 & 9 & 10 & 11 \\ 
        \midrule[0.6pt]
        CIFAR-10 & Clean & 0.2941 & 0.2996 & 0.4271 & 0.4487 & 0.4489 & 0.6101 & 0.6083 & 0.5755 & 0.6505 & 0.8110 & 0.7807 \\
        CIFAR-100 & Clean & 1.7981 & 1.9814 & 2.0760 & 2.2362 & 2.4590 & 2.9175 & 3.1986 & 3.3778 & 2.9070 & 3.3067 & 3.6144 \\
        SVHN & Clean & 0.0784 & 0.0793 & 0.1007 & 0.1221 & 0.0922 & 0.1433 & 0.1249 & 0.1588 & 0.1468 & 0.1531 & 0.2175 \\
    \bottomrule[1.7pt]
    \end{tabular}}
\end{table}

\begin{table}[!ht]
    \centering
    \caption{Test accuracy during the training process w.r.t. hyperparameter tuning on the proper thresholds.}
    \label{table:exp_range_p}
    \vspace{2mm}
    \begin{tabular}{c|c|c|c|c|c|c|c|c|c|c}
    \toprule[1.7pt]
    \midrule[0.6pt]
        Dataset & Data & Threshold $P$/Batch $m$ & 1 & 2 & 3 & 4 & 5 & 6 & 7 & 8 \\ 
        \midrule[0.6pt]
        \multirow{5}{*}{CIFAR-10} & Clean & 0.8+0.02$m$ & 0.863 & 0.867 & 0.871 & 0.872 & 0.873 & 0.873 & 0.874 & 0.873 \\
        ~ & Clean & 0.5+0.02$m$ & 0.863 & 0.867 & 0.871 & 0.872 & 0.873 & 0.873 & 0.874 & 0.871 \\ 
        ~ & Clean & 0.4+0.02$m$ & 0.863 & 0.867 & 0.873 & 0.873 & 0.874 & 0.871 & 0.869 & 0.868 \\ 
        ~ & Clean & 0.3+0.02$m$ & 0.863 & 0.867 & 0.871 & 0.872 & 0.87 & 0.87 & 0.867 & 0.867 \\ 
        ~ & Clean & 0.1+0.02$m$ & 0.863 & 0.865 & 0.865 & 0.862 & 0.861 & 0.859 & 0.858 & 0.857 \\
        \midrule[0.6pt]
        \multirow{5}{*}{CIFAR-10} & Poison & 0.8+0.02m & 0.863 & 0.841 & 0.76 & 0.665 & - & - & - & - \\ 
        ~ & Poison & 0.5+0.02$m$ & 0.863 & 0.859 & 0.842 & 0.812 & 0.771 & - & - & - \\ 
        ~ & Poison & 0.4+0.02$m$ & 0.863 & 0.859 & 0.842 & 0.812 & 0.771 & - & - & - \\ 
        ~ & Poison & 0.3+0.02$m$ & 0.863 & 0.859 & 0.842 & 0.81 & 0.77 & - & - & - \\ 
        ~ & Poison & 0.1+0.02$m$ & 0.863 & 0.859 & 0.842 & 0.81 & 0.77 & - & - & - \\
    \bottomrule[1.7pt]
    \end{tabular}
\end{table}

\subsection{Ablations about Attack Success}

We have conducted extra ablation study on evaluating the attack success (keep the same setups with the left middle panel of Figure~\ref{fig:ablation_study}), and summarize the results in Table~\ref{table:attack_success_rate}. The attack success rate here is defined as the percentage of received examples that can circumvent the defense method with specific threshold. The results show there is a trade-off between the model accuracy and the attack success rate. To be specific, it is due to the critical characteristic of our Memorization Discrepancy on clean samples and poison samples. The lower threshold tend to cover the correction ability of AT that indiscriminately treat all examples as poison sample, while the higher threshold tend to behave as the ST. As for the defense for controlling the attack success rate, it can be further designed referring to other specific techniques to utilize the critical nature of our Memorization Discrepancy.

\begin{table*}[h!]
\renewcommand\arraystretch{1.0}
\centering
\caption{Evaluations (\%) about attack success w.r.t. the threshold $P$ in CIFAR-10.}
\vspace{2mm}
\label{table:attack_success_rate}
\begin{tabular}{c|c|c}
\toprule[1.7pt]
\midrule[0.6pt]
 Threshold of Different Level &  Accuracy & Attack Success Rate \\
\midrule[0.6pt]
High P & \textbf{87.1} & 87.5 \\
Medium P & 83.7 & 37.5 \\
Low P & 80.8 & \textbf{12.5} \\

\bottomrule[1.7pt]
\end{tabular}
\vspace{-0mm}
\end{table*}

\subsection{Ablations about Different Auxiliary Models}
\label{app:aux_model}

As for the phenomenon of Memorization Discrepancy (as shown in the left panel of Fig 5), it can be found in other settings that using the model checkpoint in epoch $E\;\;(E\in[1, \text{present}])$ as the auxiliary model. The overall results show a similar trend as the left panel of Figure~\ref{fig:ablation_study}. In our main experiments in Table~\ref{table:exp_performance_cifar10}, we use the checkpoint at Epoch 20 for CIFAR-10/100 as our auxiliary model. We also conduct the experiments on CIFAR-10 using the different auxiliary model with the same threshold to see how it affect our DSC and summarize the results in Table~\ref{table:performance_using_diff_auxiliary}. The results show that the threshold may need adjustment when we choose the different auxiliary models to compute the Memorization Discrepancy. It can be found if we backtrack the earlier checkpoint (i.e., Epoch 10), the threshold estimated using checkpoint at Epoch 20 maybe still compatible. However, it is not appropriate when we use the later checkpoint (i.e., Epoch 30). Using the different auxiliary models needs further estimate the threshold by a small batch of clean data used in the previous training stage.

\begin{table*}[h!]
\renewcommand\arraystretch{1.0}
\centering
\caption{Performance of DSC using the different auxiliary models with the same/different threshold setup.}
\vspace{2mm}
\label{table:performance_using_diff_auxiliary}
\begin{tabular}{l|c|c|c|c|c}
\toprule[1.7pt]
\midrule[0.6pt]
Auxiliary Epoch &  Acc. Start & Batch & Acc. +Poison & Acc. + Trigger & $\Delta$ \\
\midrule[0.6pt]
10 & 86.3\% & 3 & 80.9$\pm$0.09\% &  77.1$\pm$0.16\% &  -3.9$\pm$0.12\%  \\
10 [adjust P] & 86.3\% & 3 & 81.4$\pm$0.05\% &  77.5$\pm$0.23\% &  -3.8$\pm$0.21\%  \\
\midrule[0.6pt]
20 & 86.3\% & 3 &  81.2$\pm$0.35\%  & 77.3$\pm$0.58\%  & -3.8$\pm$0.31\%   \\
20 [adjust P] & 86.3\% & 3 & 81.0$\pm$0.09\% &  77.8$\pm$0.22\% &  -3.6$\pm$0.05\%  \\
\midrule[0.6pt]
30 & 86.3\% & 3 &  77.3$\pm$1.25\%  &  63.6$\pm$3.39\% &  -13.6$\pm$4.64\%  \\
30 [adjust P] & 86.3\% & 3 & 80.2$\pm$0.12\% &  76.8$\pm$0.27\% &  -4.0$\pm$0.32\%  \\

\bottomrule[1.7pt]
\end{tabular}
\vspace{-0mm}
\end{table*}

\subsection{Comparisons about Other AT Variants}

As for our proposed DSC, the critical part is to selectively correct the potential poison samples using the Memorization Discrepancy. We can extend those AT variants~\citep{Zhang_trades,wang2020improving_MART,ding2020mma,zhang2020fat} to be sample corrections in our problem setting. We conduct the comparison on CIFAR-10 dataset and summarize the results in Table~\ref{table:performance_using_at_variant}. Since all those variants are designed for further improving adversarial robustness or other issues in adversarial training, its objective all introduce other optimization parts which sacrifice the natural performance, the results also demonstrate that the accuracy drop using these AT-variants-based methods for accumulative poisoning defense is more severe than the original AT.

\begin{table*}[h!]
\renewcommand\arraystretch{1.0}
\centering
\caption{Comparison with variants of AT methods for the sample correction.}
\vspace{2mm}
\label{table:performance_using_at_variant}
\begin{tabular}{c|c|c|c|c|c}
\toprule[1.7pt]
\midrule[0.6pt]
Method &  Acc. Start & Batch & Acc. +Poison & Acc. + Trigger & $\Delta$ \\
\midrule[0.6pt]
ST & 86.3\% & 1 & 75.7$\pm$3.33\% &  50.4$\pm$5.03\% &  -25.3$\pm$4.13\%  \\
AT & 86.3\% & 3 &  80.1$\pm$0.10\%  & 75.3$\pm$0.26\%  & -4.7$\pm$0.20\%   \\
TRADES & 86.3\% & 3 &  78.2$\pm$0.28\%  &  72.5$\pm$0.45\% &  -5.8$\pm$0.32\%  \\
MART & 86.3\% & 3 &  77.5$\pm$0.32\%  &  68.4$\pm$0.66\% &  -9.1$\pm$1.20\%  \\
MMD & 86.3\% & 3 &  77.2$\pm$0.81\%  &  71.4$\pm$0.77\% &  -5.8$\pm$0.89\%  \\
FAT & 86.3\% & 3 &  80.4$\pm$0.27\%  &  76.2$\pm$0.23\% &  -4.2$\pm$0.45\%  \\
DSC & 86.3\% & 3 &  \textbf{81.2$\pm$0.35\%}  &  \textbf{77.3$\pm$0.58\%} &  \textbf{-3.8$\pm$0.31\%}  \\

\bottomrule[1.7pt]
\end{tabular}
\vspace{-0mm}
\end{table*}



\subsection{Ablations about Black-box Setting}

Empirically, we also verify the effect of our proposed method on the extended black-box setting for accumulative poisoning attack, and summarize the results compared with White-box setting in Table~\ref{table:performance_using_diff_setting}. In this setting, we use other surrogate models (e.g., the historical model earlier than the current model stage) to generate the adversarial examples and feed them into the vaccine model. The results show that our DSC has a comparable defense effect to that of the white-box setting.

\begin{table*}[h!]
\renewcommand\arraystretch{1.0}
\centering
\caption{Comparison with variants of AT methods for the sample correction.}
\vspace{2mm}
\label{table:performance_using_diff_setting}
\resizebox{\textwidth}{!}{
\begin{tabular}{c|c|c|c|c|c|c}
\toprule[1.7pt]
\midrule[0.6pt]
Setting & White/Black &  Acc. Start & Batch & Acc. +Poison & Acc. + Trigger & $\Delta$ \\
\midrule[0.6pt]
CIFAR-10 (Clean Oracle) & - & 86.3\% & - & 84.7\% &  84.7\% &  -  \\
Accu. Poison (DSC) & White-box & 86.3\% & 3 &  81.2$\pm$0.35\%  & 77.3$\pm$0.58\%  & -3.8$\pm$0.31\%   \\
Accu. Poison (DSC) & Black-box [30] & 86.3\% & 3 &  81.7$\pm$0.23\%  & 78.2$\pm$0.14\%  & -3.5$\pm$0.12\%   \\
Accu. Poison (DSC) & Black-box [20] & 86.3\% & 3 &  82.0$\pm$0.11\%  & 78.9$\pm$0.23\%  & -3.1$\pm$0.08\%   \\
Accu. Poison (DSC) & Black-box [10] & 86.3\% & 3 &  82.5$\pm$0.02\%  & 79.7$\pm$0.03\%  & -2.8$\pm$0.05\%   \\
\bottomrule[1.7pt]
\end{tabular}}
\vspace{-0mm}
\end{table*}

\subsection{Empirical Evaluation of the Correction Condition}
\label{app:condition}

As for the hyper-parameters $\mu$ and $\tau$, in the burn-in phase that follows the~\cite{pang2021accumulative}, we can estimate them by using a small batch sample of clean data. According to the previous properties of the Memorization Discrepancy we observed, we can approximate the $\mu$ and $\tau$ by the value computed on the clean data in some period of the burn-in phase.  And we did not change the defense parameters between these two kinds of experiments for fair evaluation. To provide more informative results, we check the experiments for running the clean oracle with 50 batches of samples and summarize how often the threshold condition is satisfied during training in Table~\ref{table:evaluation_condition}. The results show that part of the clean samples is also affected by our DSC and their value satisfies the condition in Algorithm~\ref{alg:dsc}. For the experiments with clean oracle, we use the same threshold as the experiments on defending against the accumulative poisoning attack. It shows the selective mechanism based on the condition.

\begin{table*}[h!]
\renewcommand\arraystretch{1.0}
\centering
\caption{How often the threshold condition is satisfied during training?}
\vspace{2mm}
\label{table:evaluation_condition}
\begin{tabular}{c|c|c|c}
\toprule[1.7pt]
\midrule[0.6pt]
 Dataset &  Acc. Start & Acc. Oracle & Frequency (Satisfy the Correction Condition) \\
\midrule[0.6pt]
CIFAR-10 & 86.3\% & 84.7\% & 28\% \\
CIFAR-100 & 59.0\% & 55.0\% & 24\% \\
\bottomrule[1.7pt]
\end{tabular}
\vspace{-0mm}
\end{table*}

\subsection{Preliminary Exploration on Federated Setting}

Different from real-time data streaming, for the accumulative poisoning attacks in a federated setting, we need to adapt our method to the federated learning framework where we can not directly conduct the sample-wise correction. Specifically, we incorporate the proposed Memorization Discrepancy into the selective defense (e.g., Discrepancy-aware Gradient Clipping (DGC)) against accumulative poisoning attacks and conduct the experiments in the following table. We can see that the extended method can also perform comparable or better based on selectively adjusting the training.

\begin{table}[!ht]
    \centering
    \caption{Classification accuracy (\%) on CIFAR-10 during the accumulative phase for 500 steps. Our new information measure on learning dynamics with the historical model can serve as an auxiliary for gradient clipping operations.}
    \vspace{2mm}
    \begin{tabular}{c|c|c|c|c|c|c|c}
    \toprule[1.7pt]
    \midrule[0.6pt]
        CIFAR-10 & Method & 10 & 100 & 200 & 300 & 400 & 500 \\ 
    \midrule[0.6pt]    
        \multirow{3}{*}{Clean Oracle} & ST & 85.35 & 83.87 & 83.9 & 83.88 & 83.81 & 83.86 \\
        ~ & GC & 84.34 & 85.27 & 85.5 & 85.48 & 85.46 & \textbf{85.43} \\ 
        ~ & DGC & 84.34 & 85.31 & 85.49 & 85.5 & 85.45 & \textbf{85.43} \\ 
    \midrule[0.6pt]  
        \multirow{3}{*}{Accu. Poisoned} & ST & 84.65 & 69.96 & 68.74 & 69.36 & 69.31 & 69.23 \\ 
        ~ & GC & 84.88 & 84.27 & 83.14 & 81.78 & 80.15 & 78.95 \\ 
        ~ & DGC & 84.87 & 84.31 & 83.3 & 82.13 & 80.66 & \textbf{79.84} \\ 
        \bottomrule[1.7pt]
    \end{tabular}
\end{table}


\subsection{More Dynamics of the Memorization Discrepancy}
\label{app:more_dynamics}

In this part, we present more exploration about the dynamics of the proposed Memorization Discrepancy. For the poisoning generation, we follow the same malicious objective in Eq.~\eqref{eq:mal_1} and adopt~\citet{fowl2021adversarial} to generate the poison samples for presenting the discrepancy trend, i.e., generating the adversarial poison samples via the adversarial generation procedure~\citep{Madry_adversarial_training}. In Figure~\ref{fig:reason_0_app}, we change the backtracking interval $k$ during the historical 40 epochs. The differences between the Memorization Discrepancy of clean and poison samples approximately become more separable when we increase $K$. In Figure~\ref{fig:reason_1_app}, we fix the auxiliary model at Epoch 1 and investigate the value of Memorization Discrepancy using different intervals. The overall results show a similar trend with the previous analysis, that we can better utilize the model dynamics via enlarging the backtracking interval in computing the Memorization Discrepancy. In Figure~\ref{fig:reason_2_app}, we change the different auxiliary models from Epoch 1 to Epoch 28. Although there exists the same trend as the previous two explorations, the value of Memorization Discrepancy varies among the different auxiliary models. It can draw the same conclusion as the experiment in Appendix~\ref{app:aux_model} that we may need further estimate the appropriate threshold for distinguishing the clean and poison samples. The overall results demonstrate that model dynamics are aware of the imperceptible poison samples.

\begin{figure}[t!]
    \centering
    \vspace{4mm}
    \includegraphics[scale=0.19]{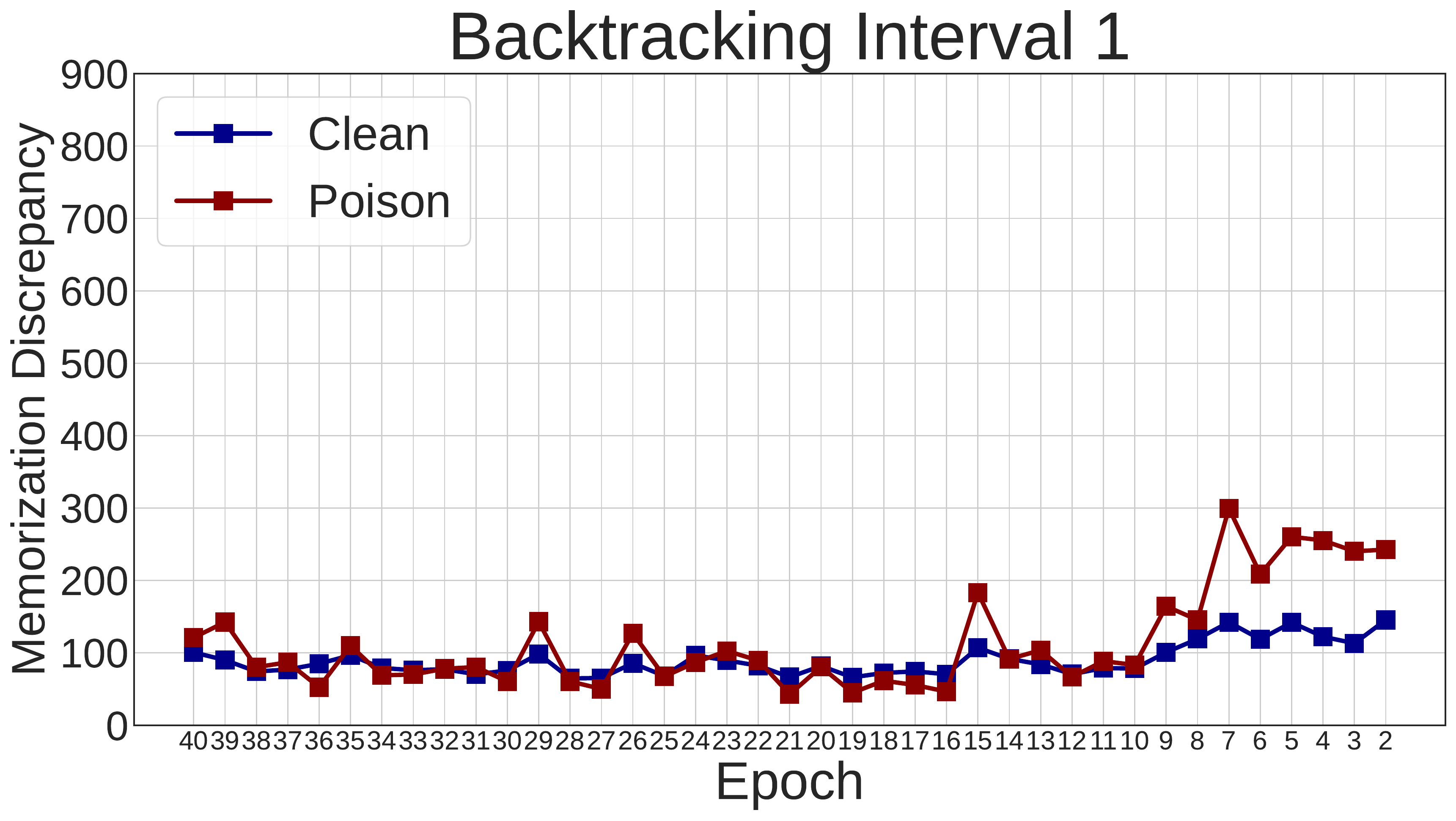}
    \includegraphics[scale=0.19]{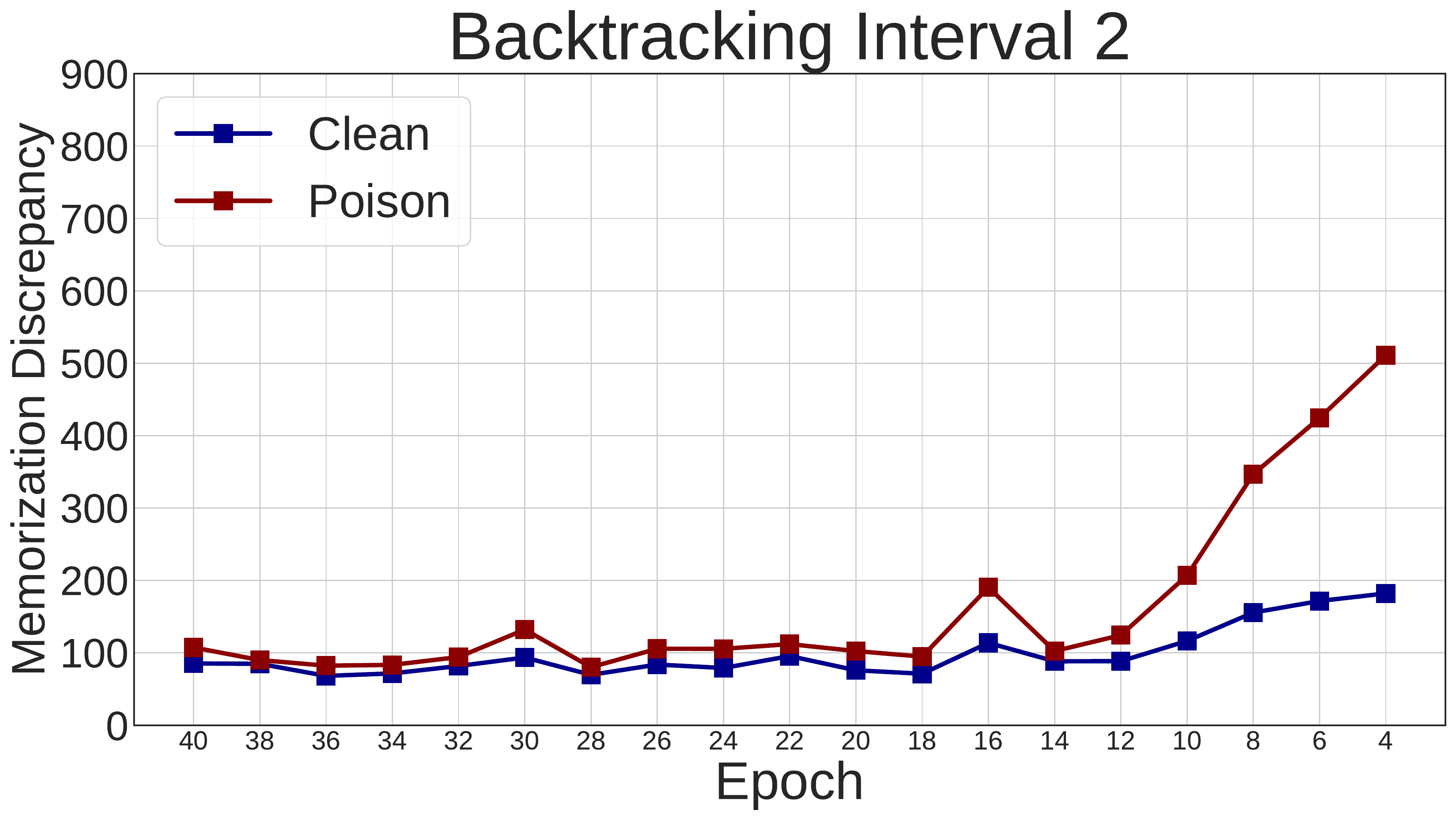}\\
    \vspace{2mm}
    \includegraphics[scale=0.19]{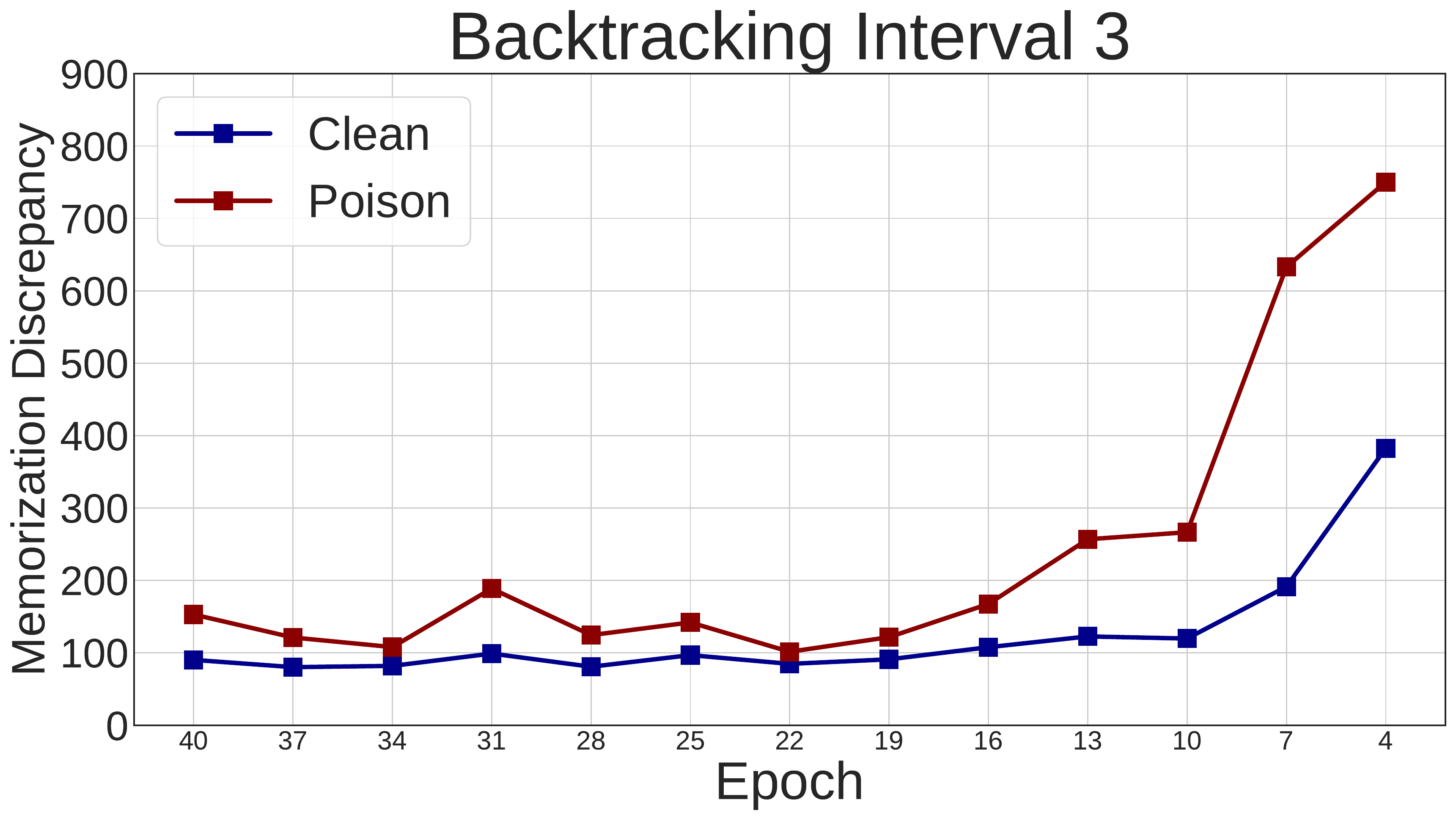}
    \includegraphics[scale=0.19]{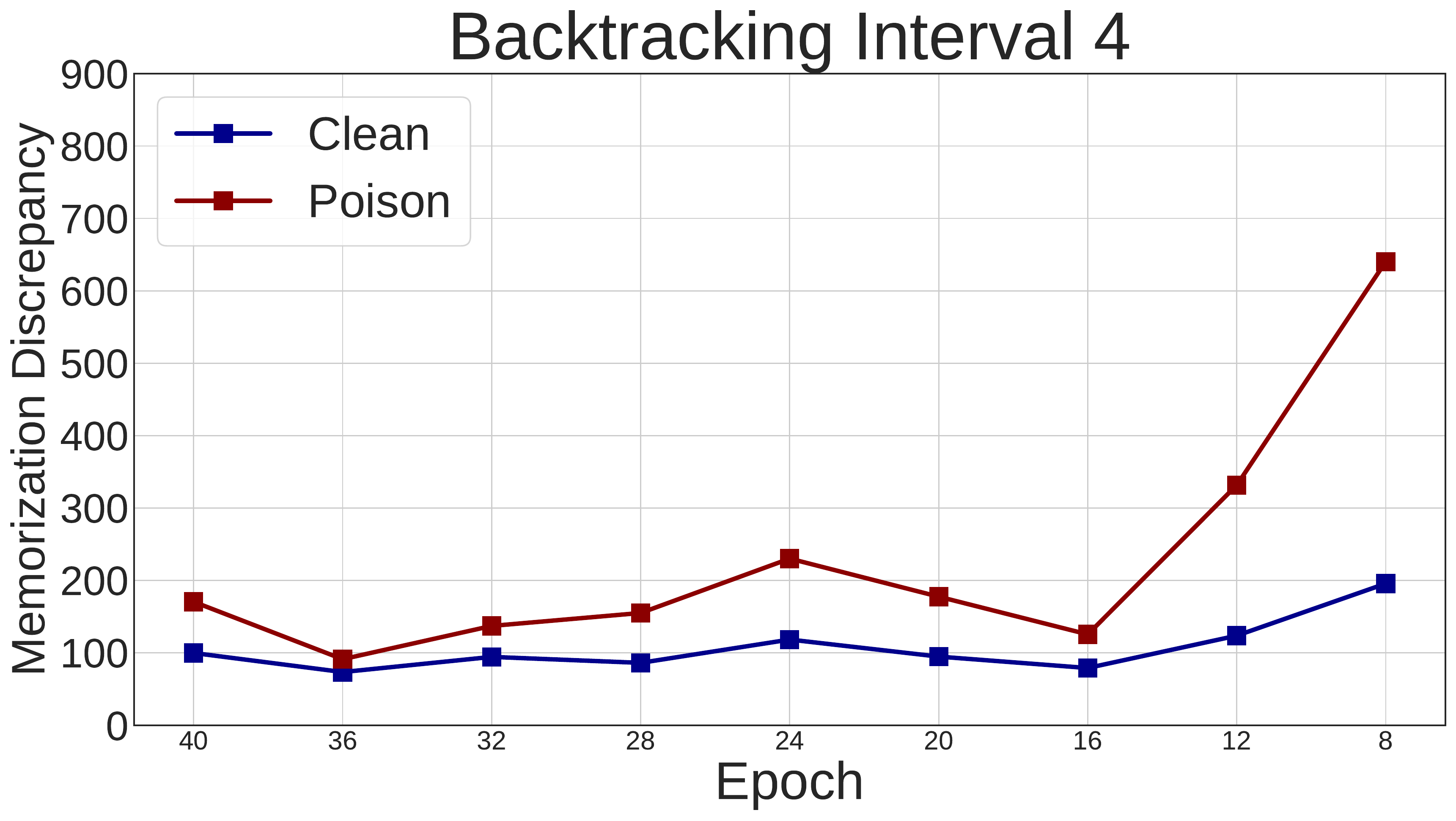}\\
    \vspace{2mm}
    \includegraphics[scale=0.19]{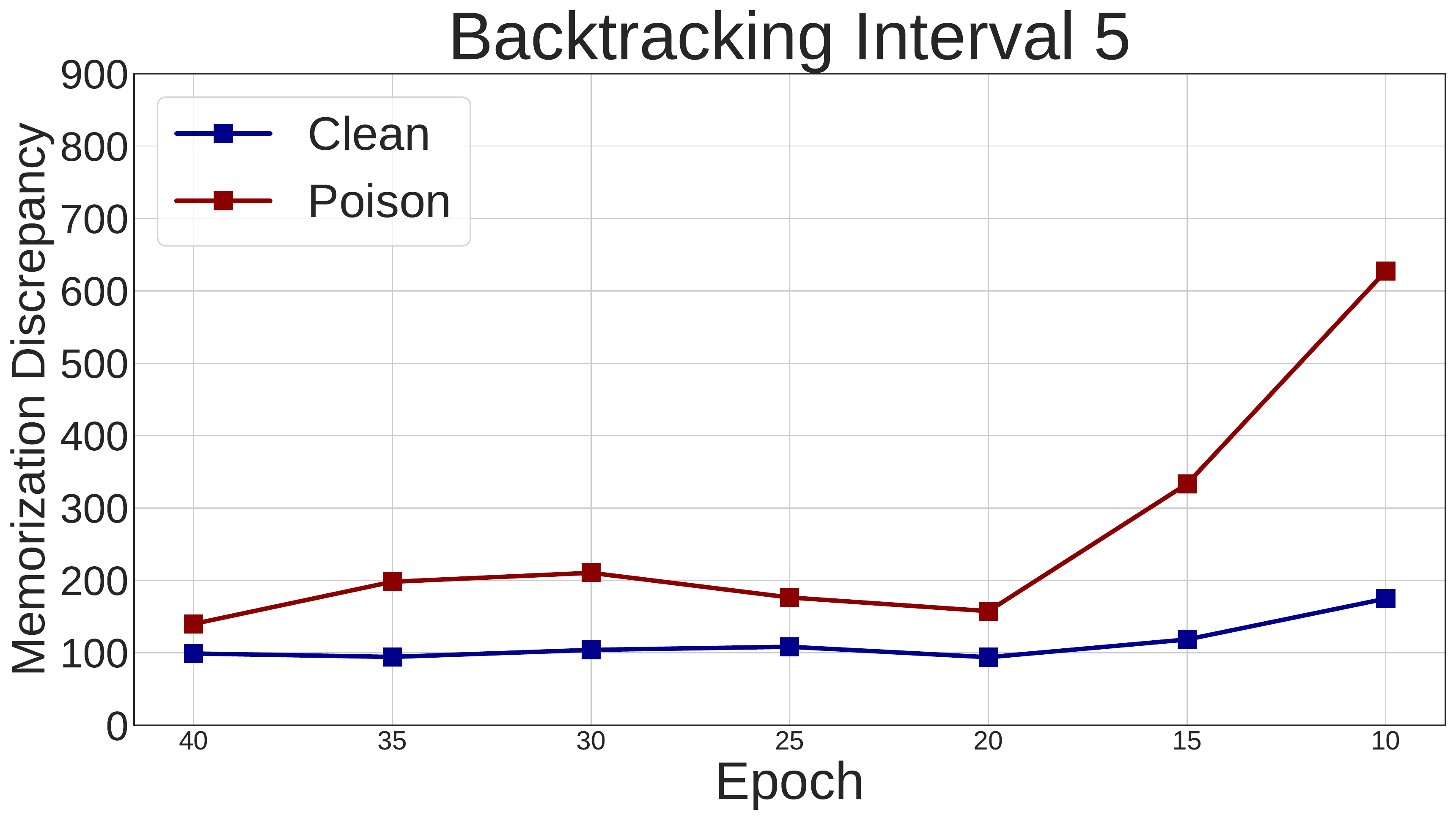}
    \includegraphics[scale=0.19]{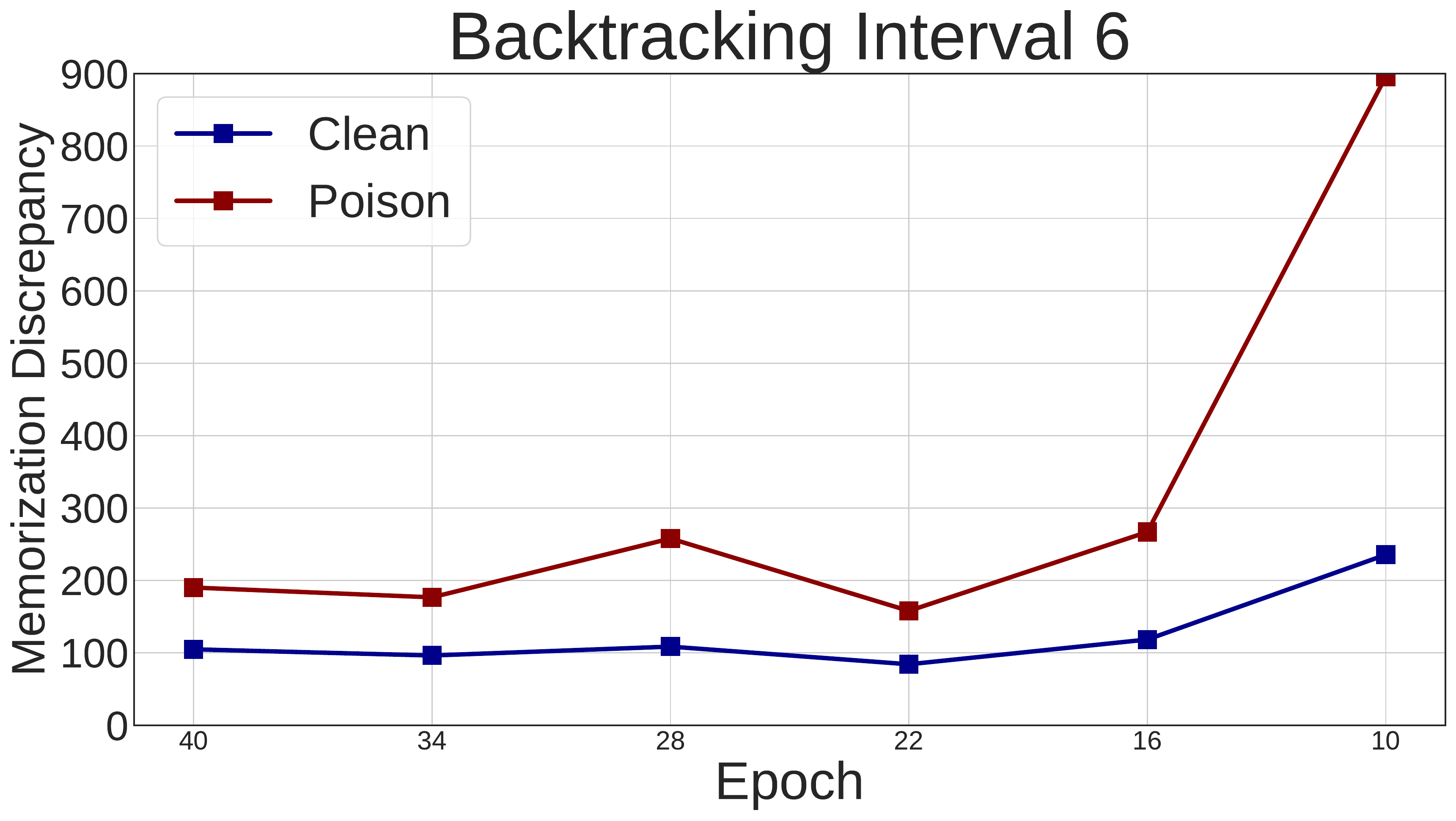}\\
    \vspace{2mm}
    \includegraphics[scale=0.19]{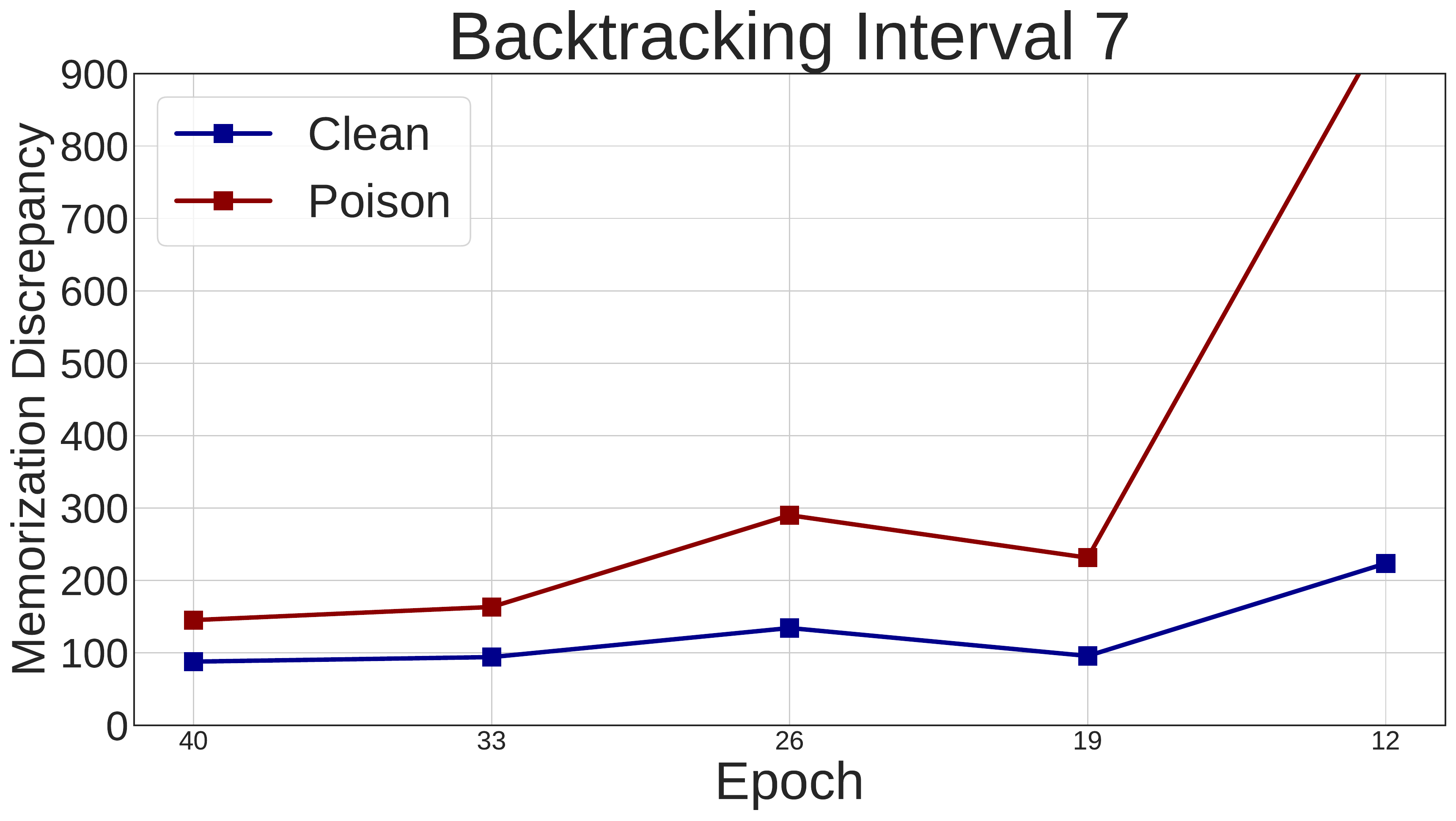}
    \includegraphics[scale=0.19]{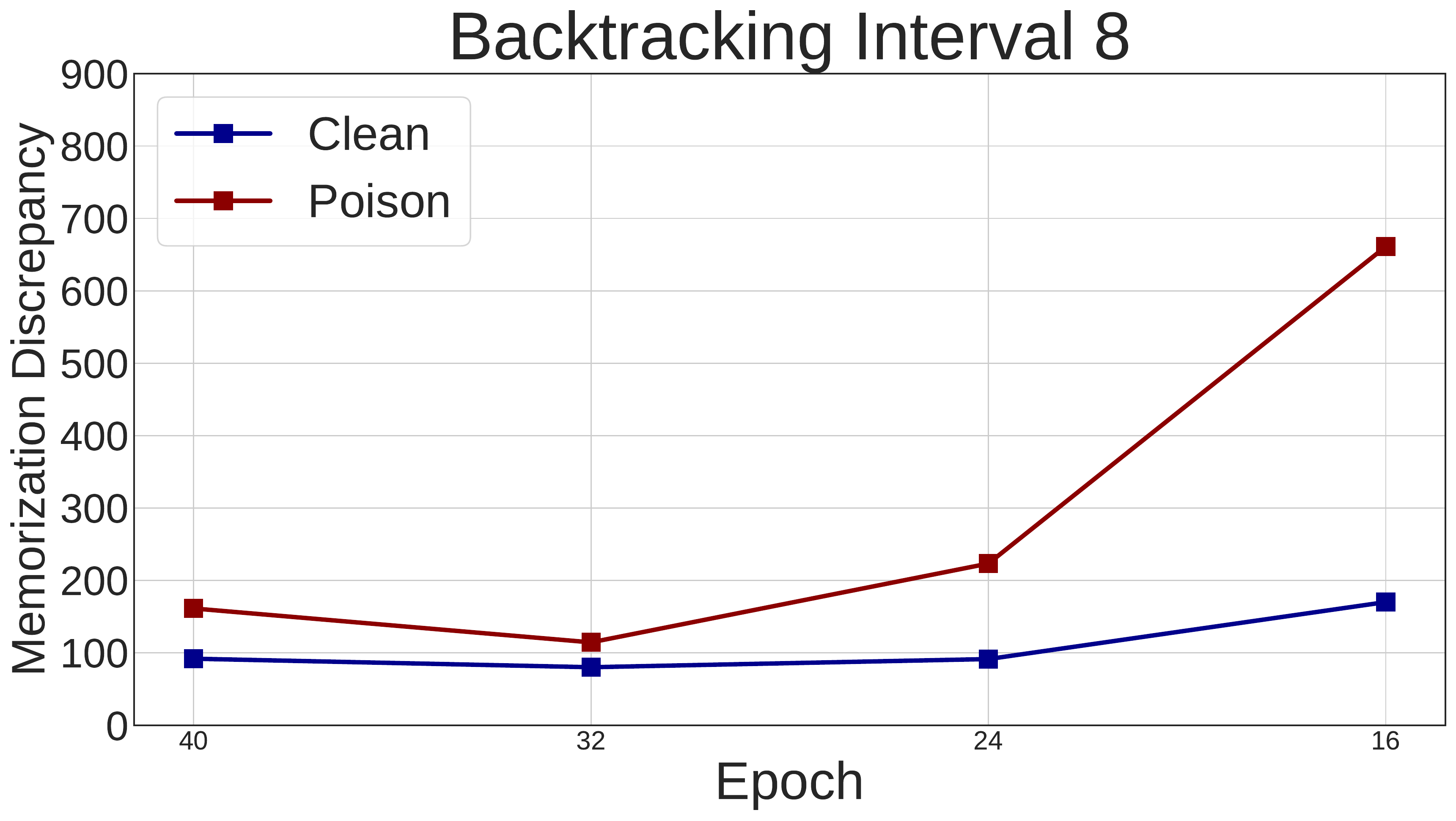}\\
    \caption{Dynamics of backtracking interval on Memorization Difference in CIFAR-10.}
    \label{fig:reason_0_app}
    \vspace{2mm}
\end{figure}

\begin{figure}[t!]
    \centering
    \vspace{4mm}
    \includegraphics[scale=0.19]{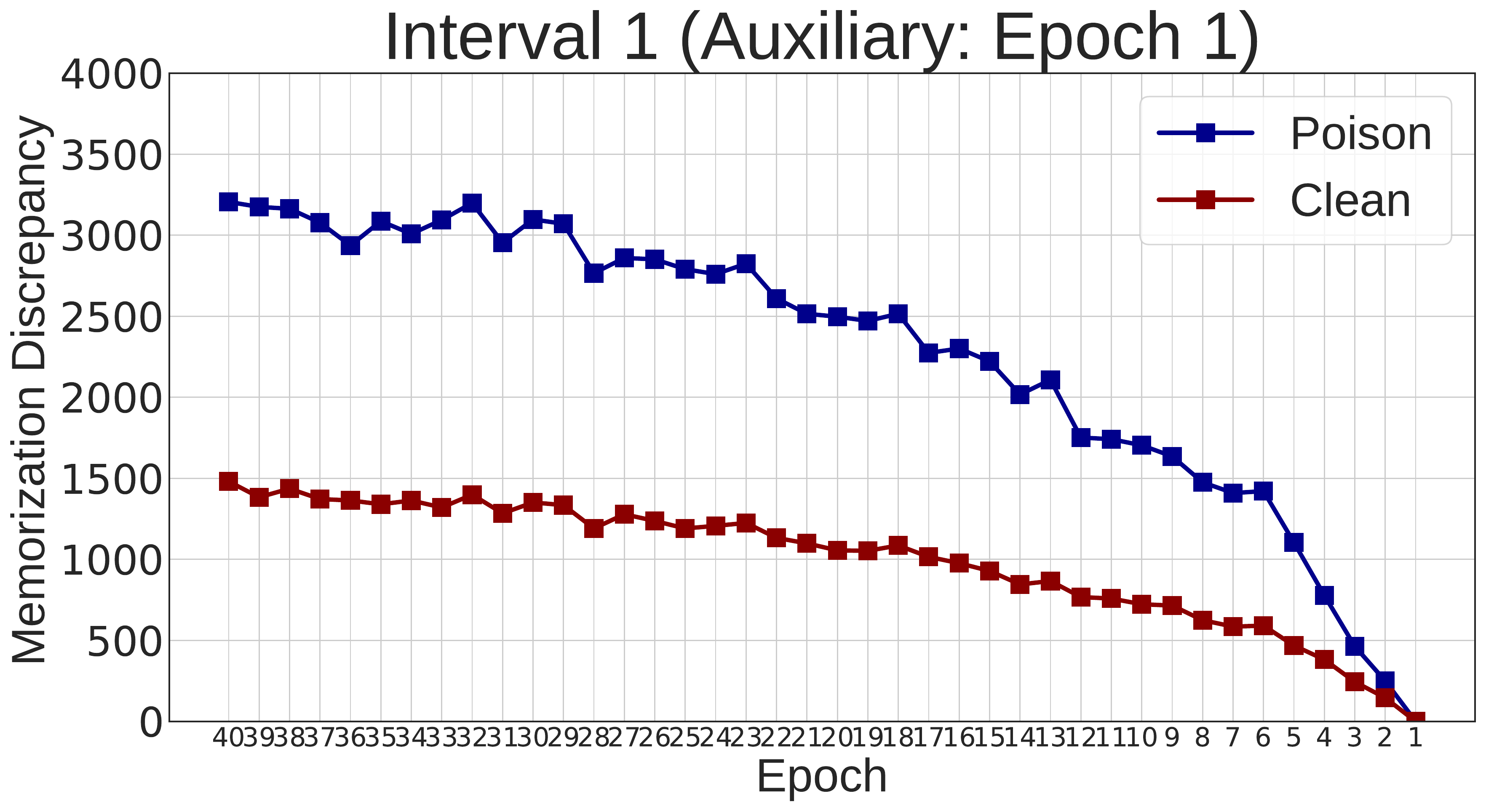}
    \includegraphics[scale=0.19]{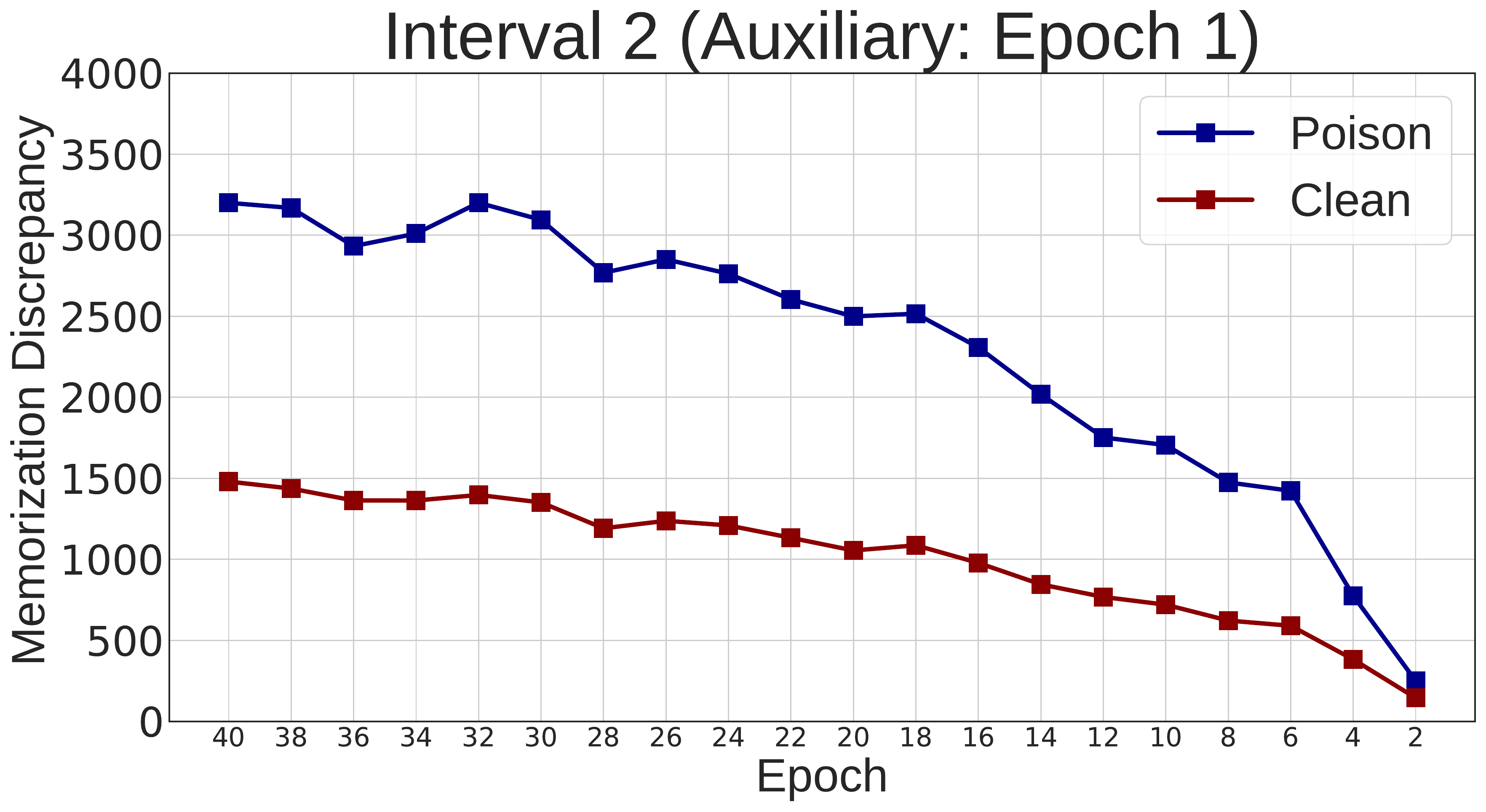}\\
    \vspace{2mm}
    \includegraphics[scale=0.19]{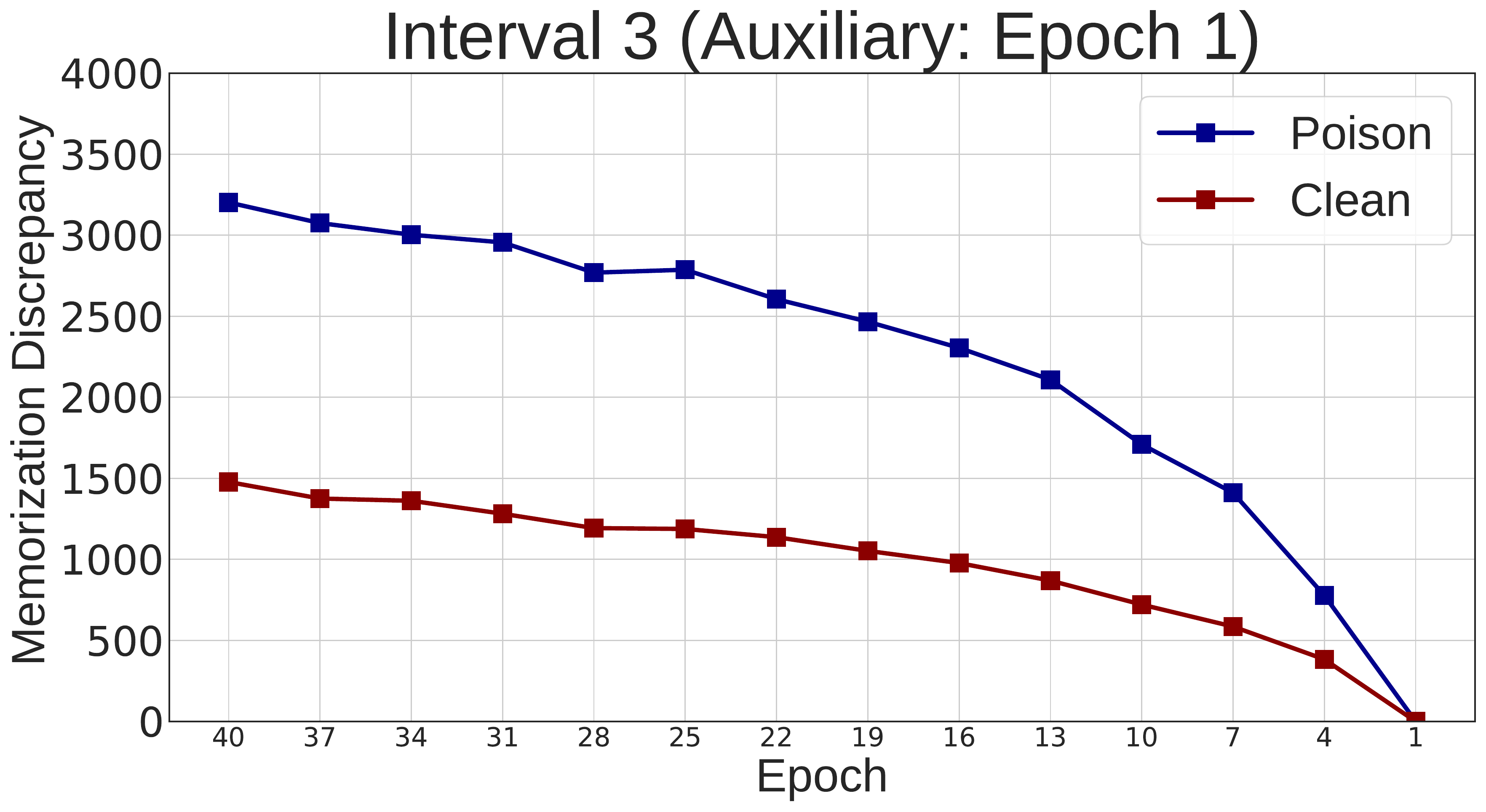}
    \includegraphics[scale=0.19]{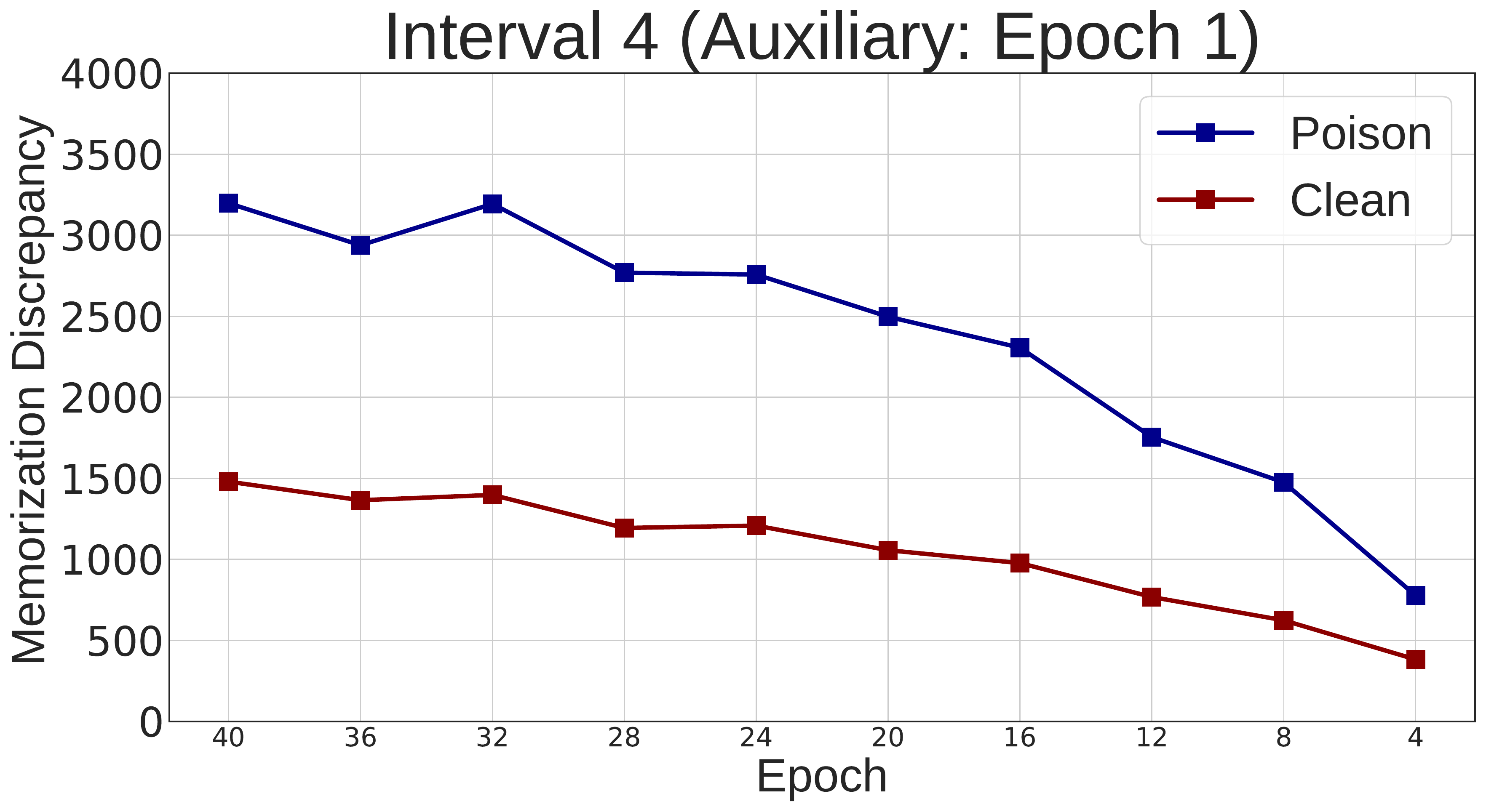}\\
    \vspace{2mm}
    \includegraphics[scale=0.19]{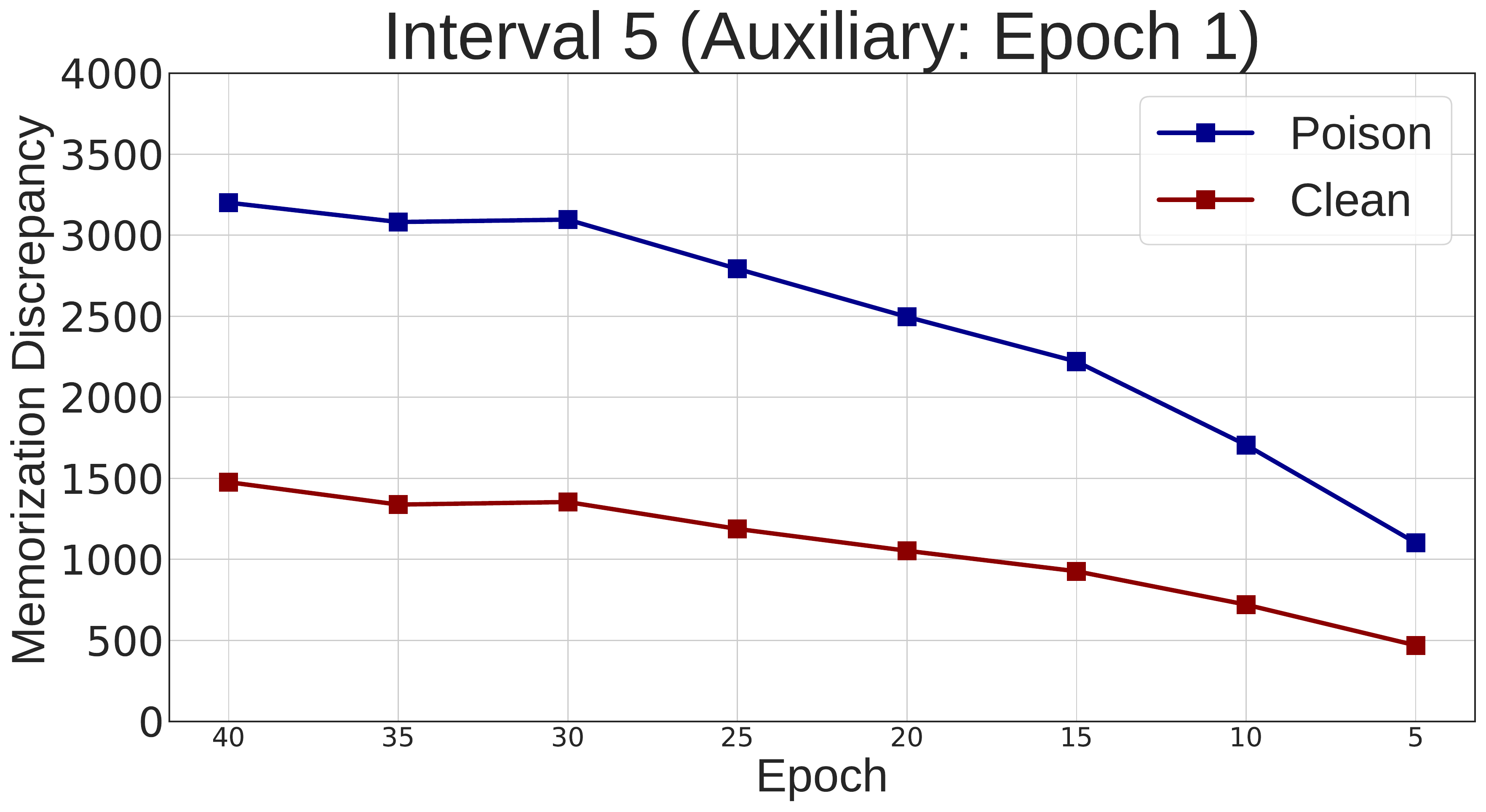}
    \includegraphics[scale=0.19]{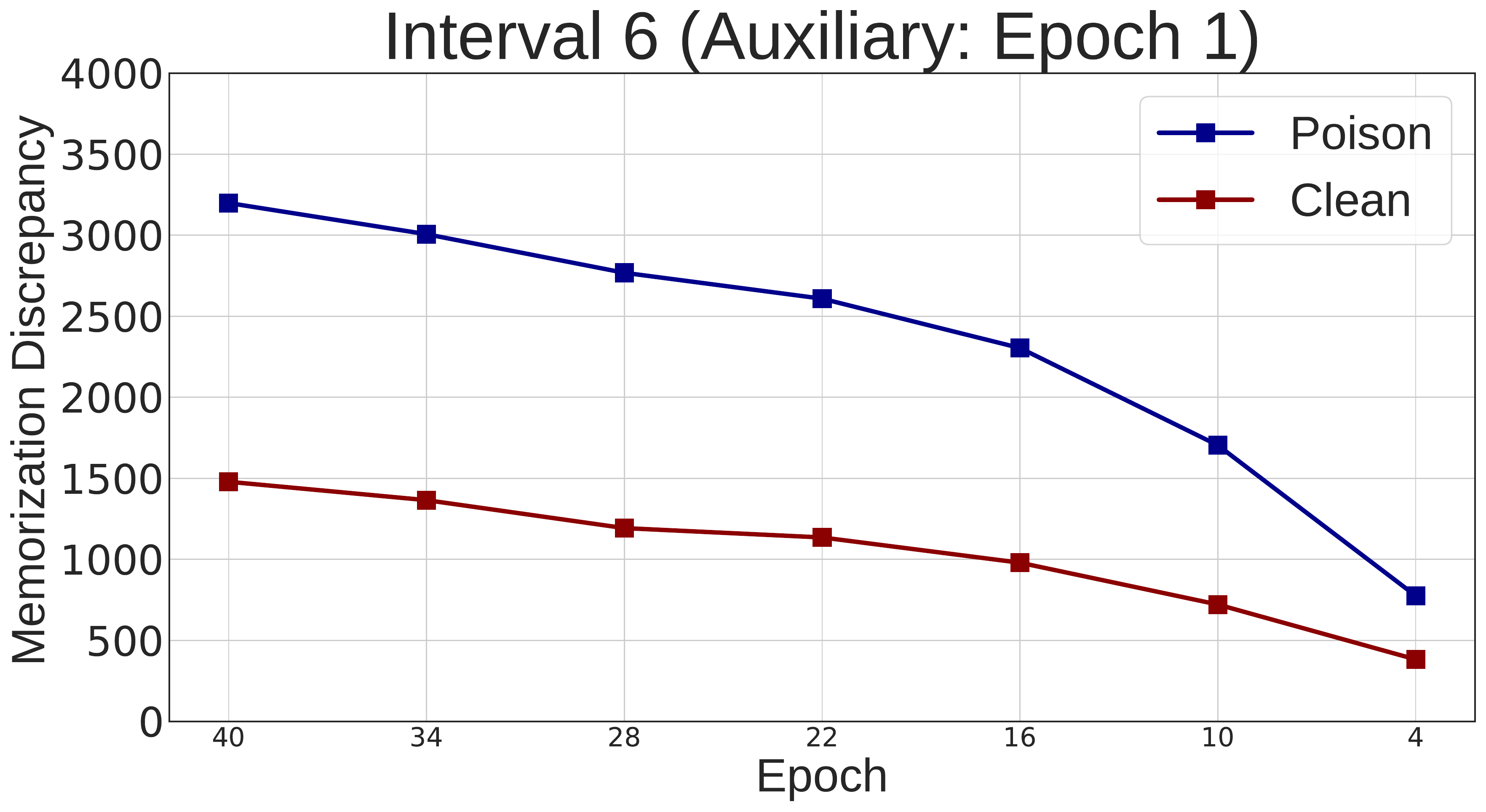}\\
    \vspace{2mm}
    \includegraphics[scale=0.19]{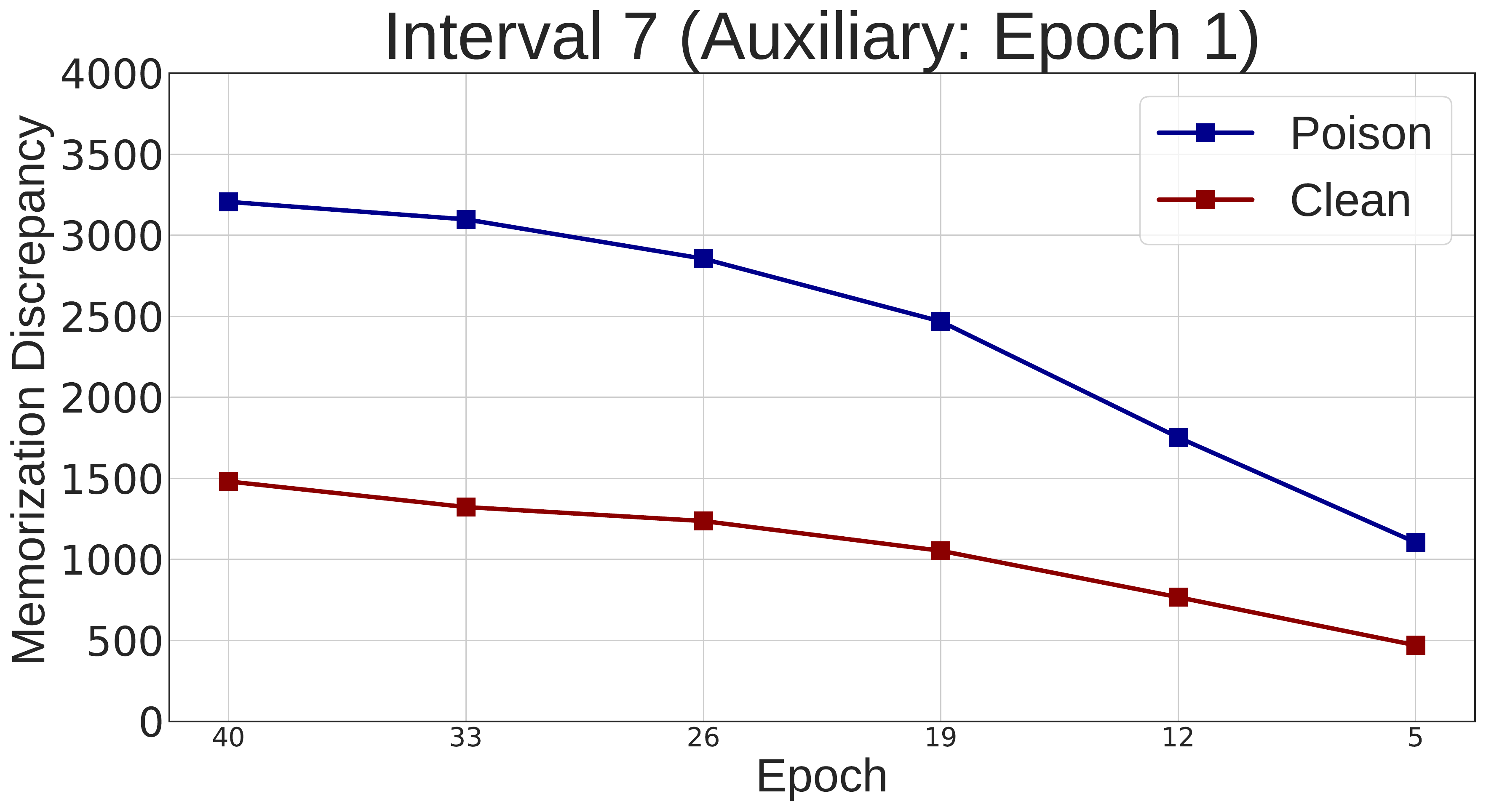}
    \includegraphics[scale=0.19]{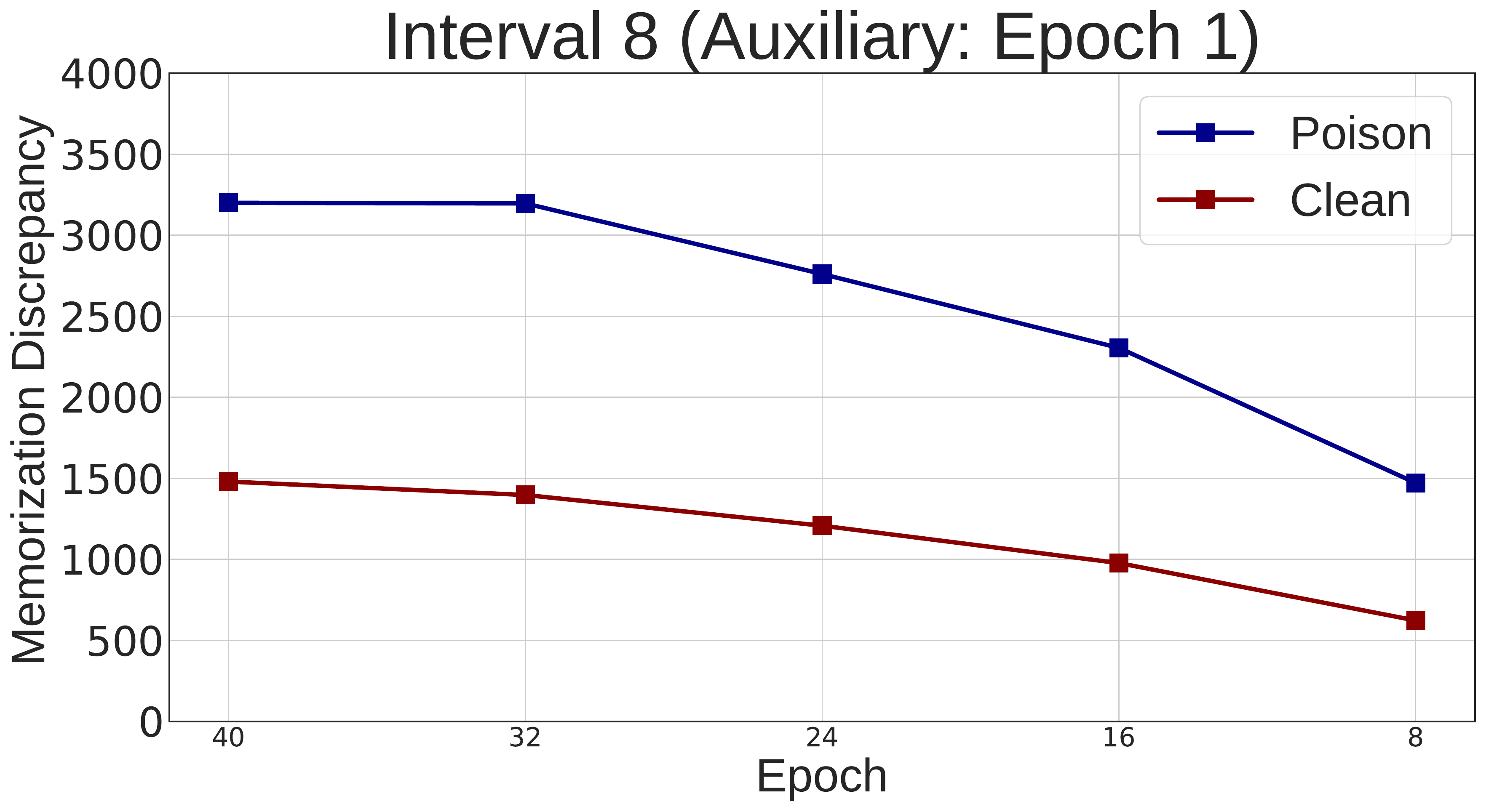}\\
    \caption{Dynamics of same auxiliary epoch on Memorization Difference in CIFAR-10.}
    \label{fig:reason_1_app}
\end{figure}

\begin{figure}[t!]
    \centering
    \vspace{4mm}
    \includegraphics[scale=0.19]{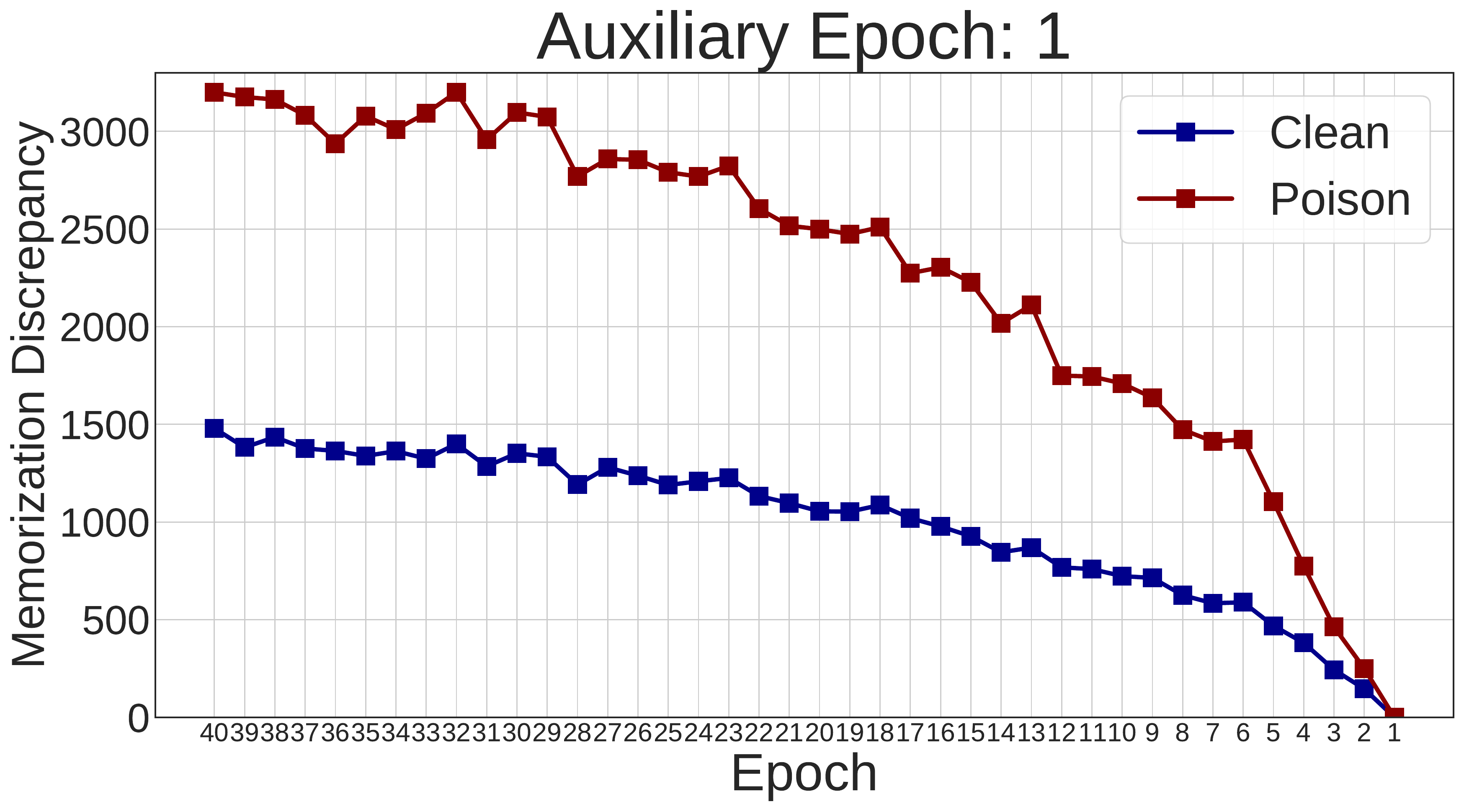}
    \includegraphics[scale=0.19]{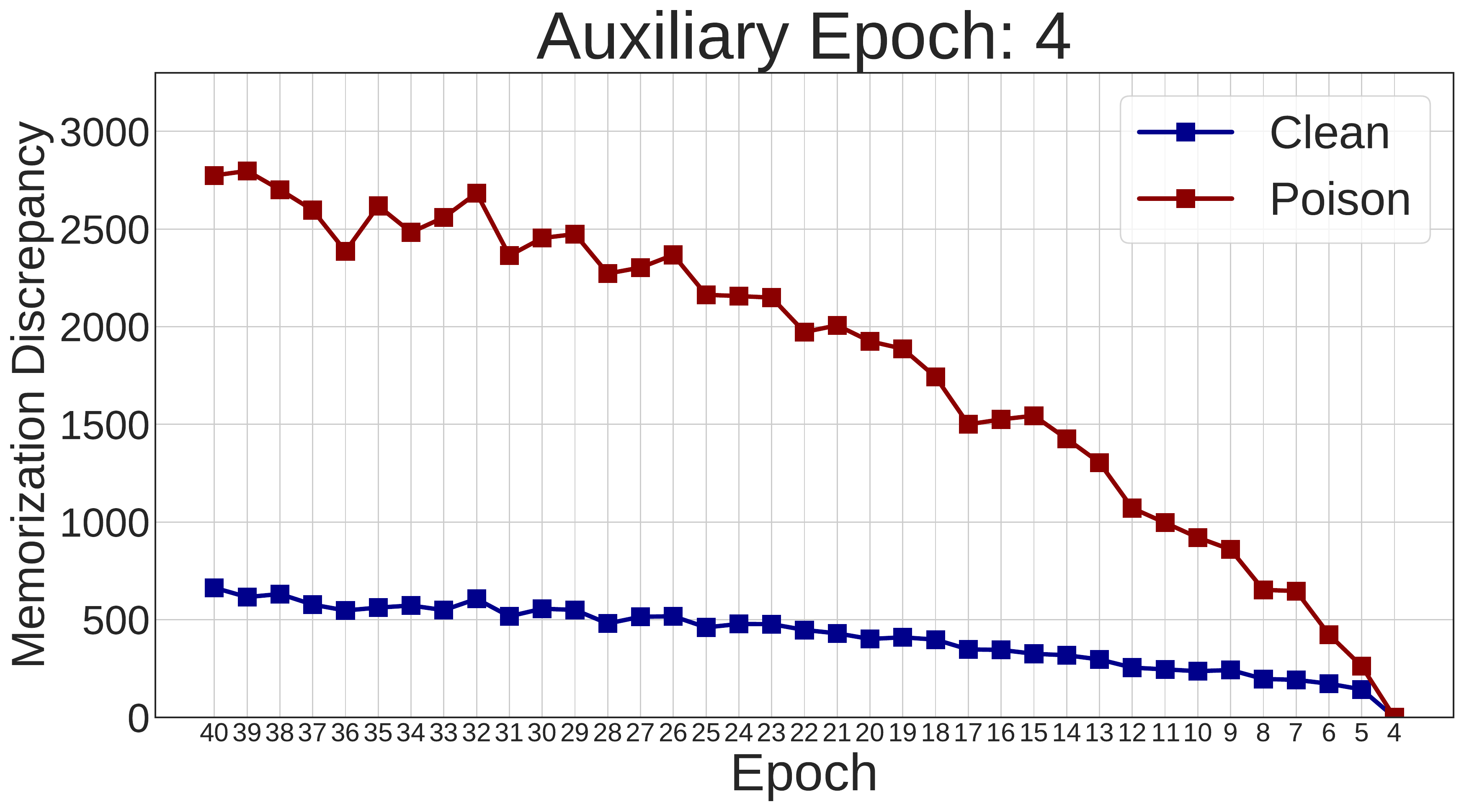}\\
    \vspace{2mm}
    \includegraphics[scale=0.19]{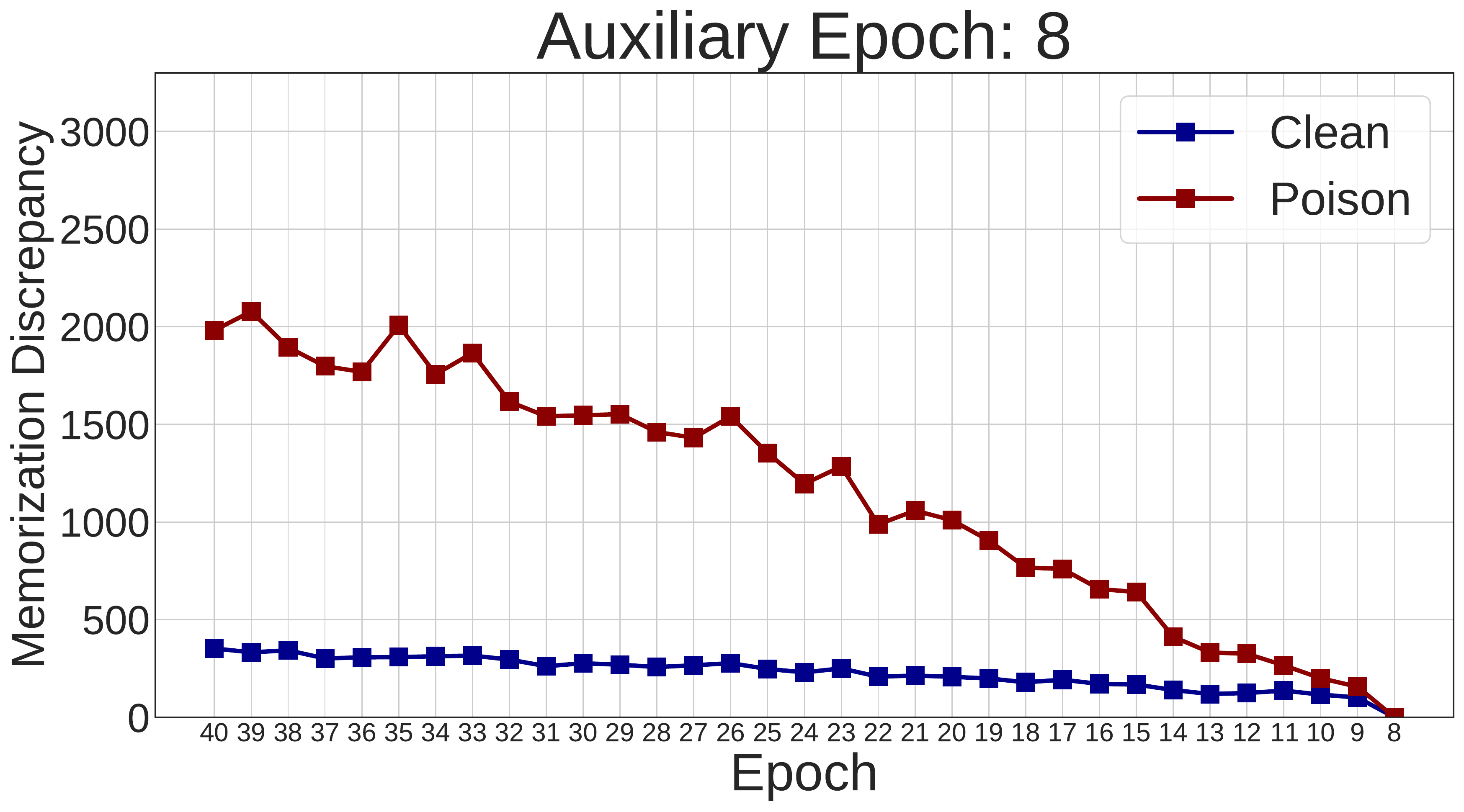}
    \includegraphics[scale=0.19]{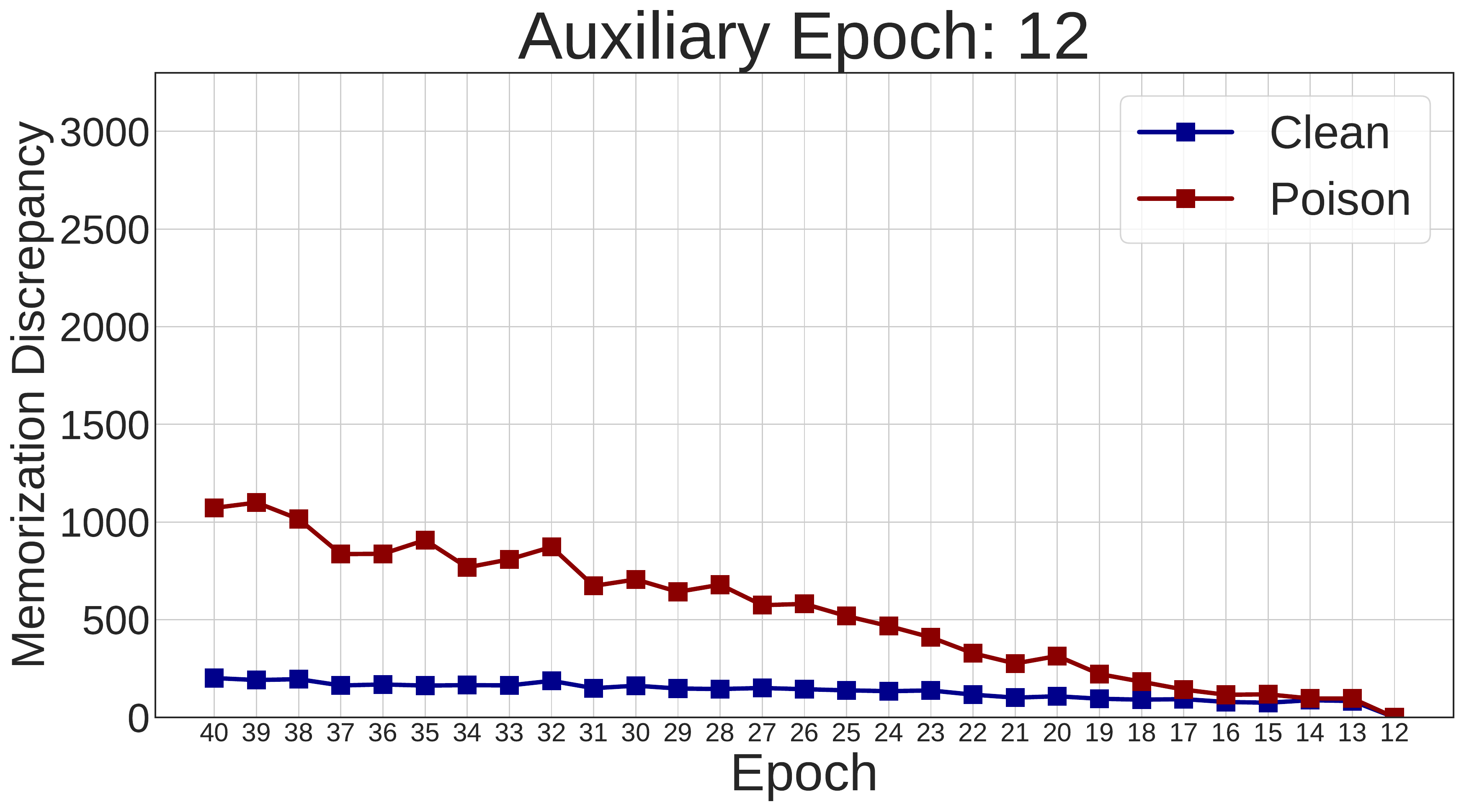}\\
    \vspace{2mm}
    \includegraphics[scale=0.19]{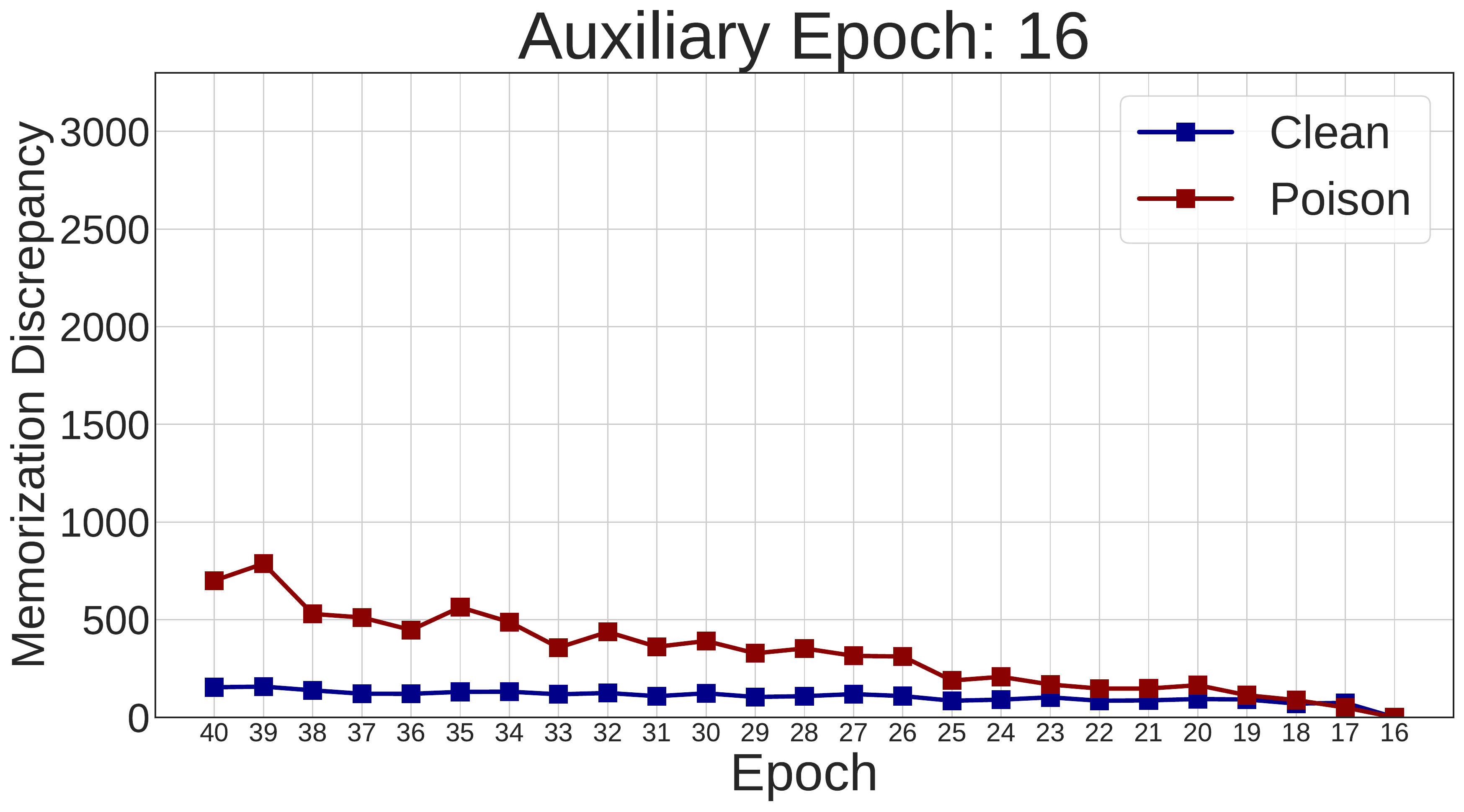}
    \includegraphics[scale=0.19]{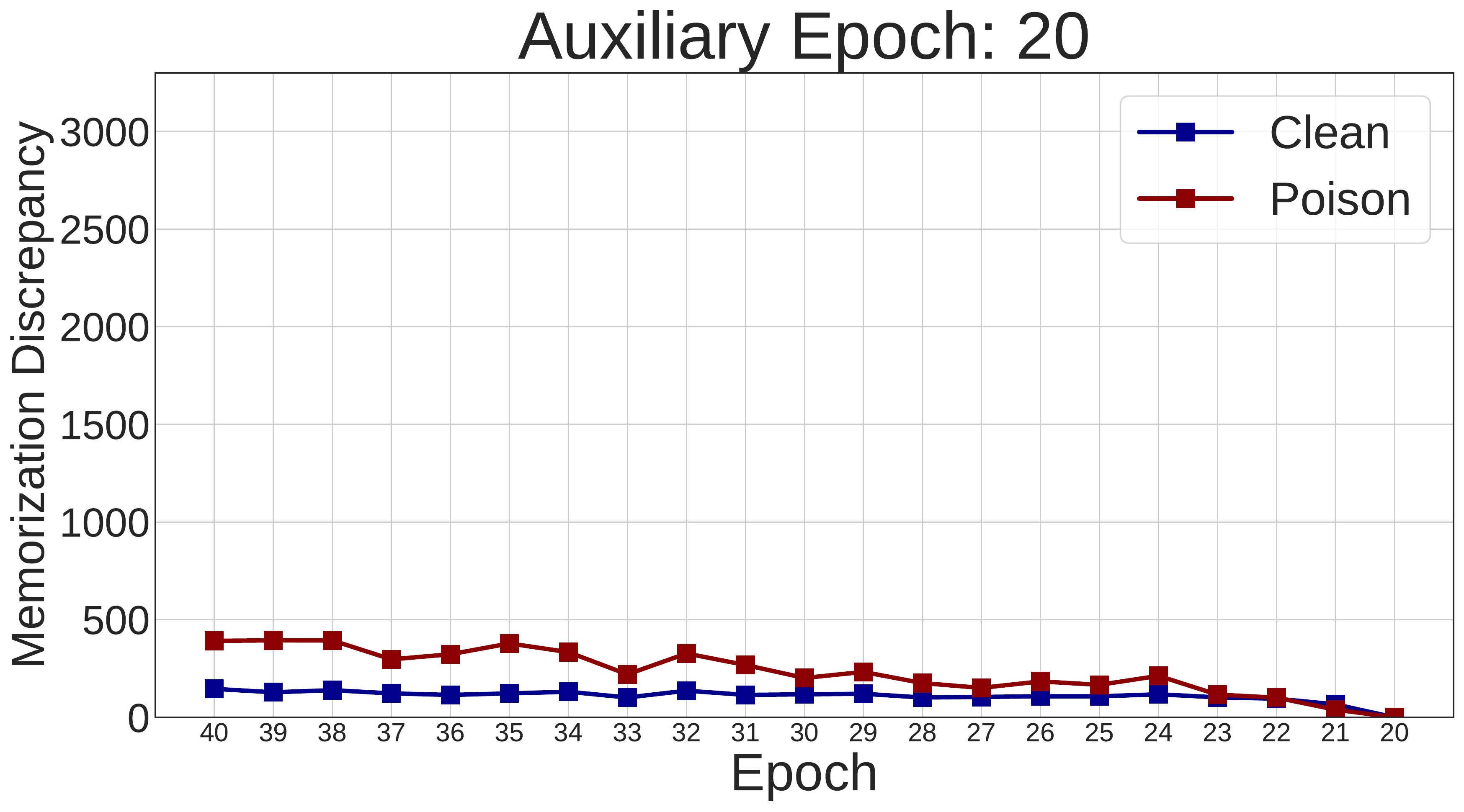}\\
    \vspace{2mm}
    \includegraphics[scale=0.19]{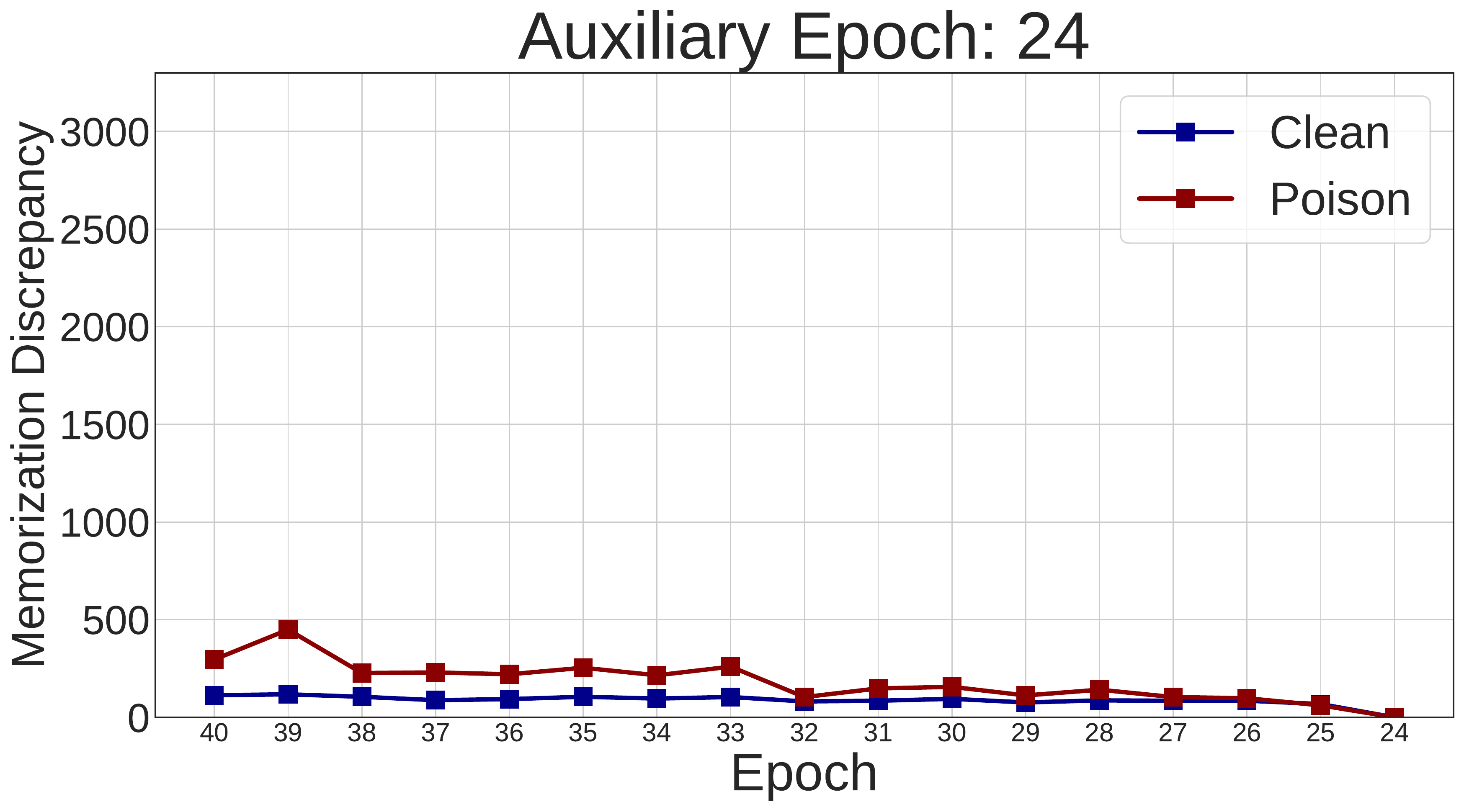}
    \includegraphics[scale=0.19]{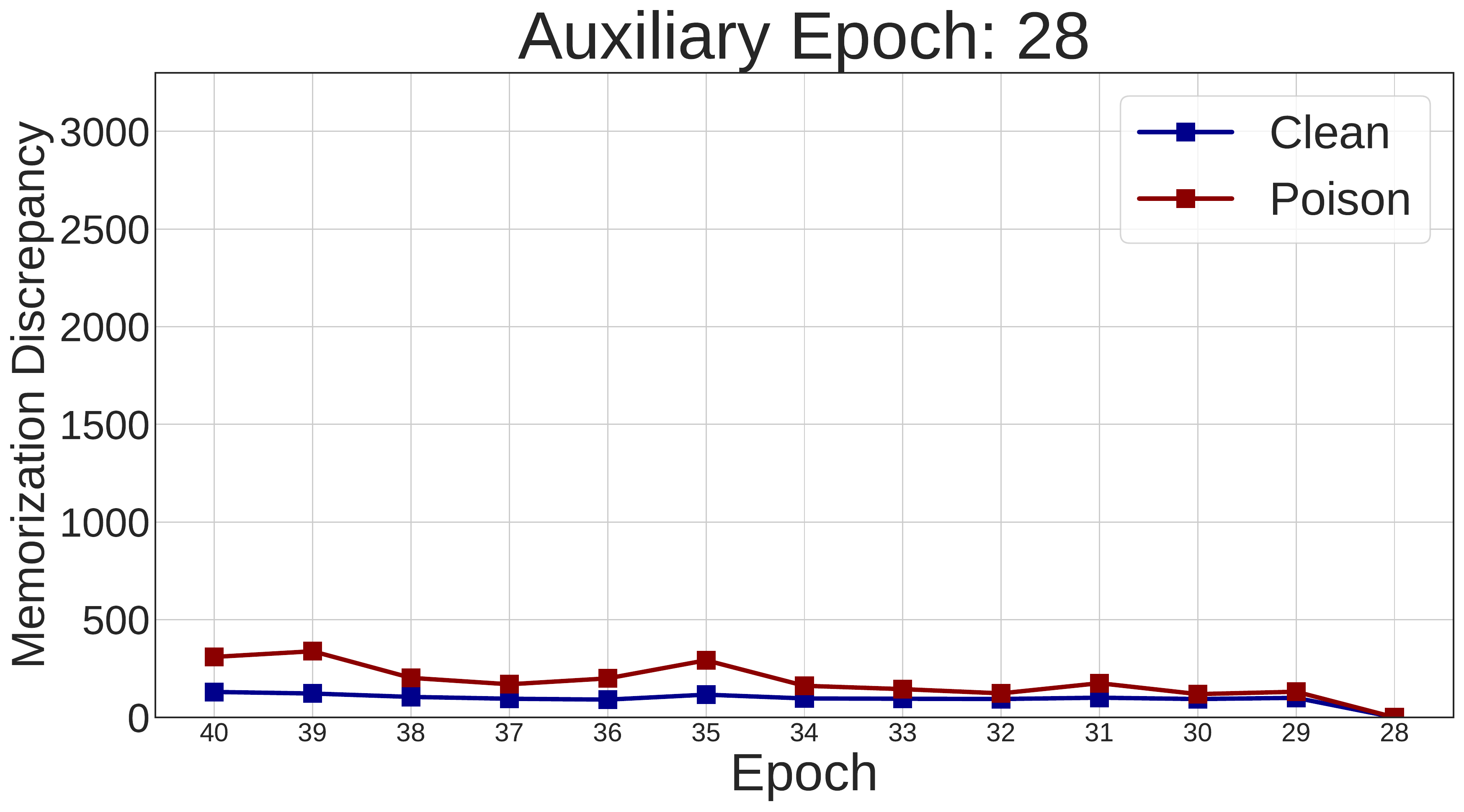}\\
    \caption{Dynamics of different auxiliary model on Memorization Difference in CIFAR-10.}
    \label{fig:reason_2_app}
\end{figure}

\subsection{Detailed Discussion about Attackers Being Aware of Memorization Discrepancy}
\label{app:adaptive_attack}

Considering the concern about adaptive attackers in conventional adversarial literatures~\citep{tramer2020adaptive}, we also present a further discussion about a stronger attacker being aware of our Memorization Discrepancy and trying to incorporate it into the poison sample generation~\citep{pang2021accumulative} with the auxiliary model that used in our experiments.

Before that, we also try different adversarial attacking objectives (e.g., PGD~\citep{Madry_adversarial_training}, KL-based method in TRADES~\citep{Zhang_trades}, and C\&W~\citep{Carlini017_CW}) in generating the poison samples. Our empirical results in Table~\ref{table:exp_other_attack} show that different adversarial generation methods in conventional adversarial literature have limited differences from each other. Then, we delve into the stronger attacker that is also optimized for Memorization Discrepancy.

However, unlike the previous adaptive adversarial attacks~\citep{Carlini017_CW,tramer2020adaptive} utilizing the extra search space to find a stronger adversarial example to satisfy the misclassification requirement, and meanwhile keep the imperceptibility. Our empirical results in Table~\ref{table:exp_adaptive_attack} show that keeping the constraint of Memorization Discrepancy can directly affect the poisoning effect induced by the generated poison samples, indicating the underlying difference between generating adversarial examples~\citep{Goodfellow14_Adversarial_examples} for misleading the model inference and generating adversarial poison samples for misleading the model training. In other words, the constraint on Memorization Discrepancy in poison generation will directly mitigate the poison effect on the target model.

To be specific, following the detailed optimization procedure of accumulative poisoning attack, we incorporate the constraint of Memorization Discrepancy into the original generation equation used in~\citep{pang2021accumulative}. Similar to the first constraint in Eq.~\eqref{eq:real_time_malicious_ob_re_accu_expand_v2} used for keeping the accuracy (which is targeted for escaping from a simple monitor based on accuracy statics), we add the second term in Eq.~\eqref{eq:extend_malicious_ob_re_accu_expand_v2} for Memorization Discrepancy, where $\mathcal{L}_\text{MD}=\mathbb{D}(f(\hat{x}(\theta^t); \theta^{*}), f(\hat{x}(\theta^t); \theta^{t}))$ and the auxiliary historical model $\theta^*$ are kept same as DSC. The whole generation objective is extended as follows,
\begin{align}
\label{eq:extend_malicious_ob_re_accu_expand_v2}
    \max_{\mathcal{P},\mathcal{A}_t}\nabla_{\theta}\mathcal{L}(\mathcal{A}_t(S_{t});\theta^t)^\top\left[ \underbrace{\nabla_{\theta}\mathcal{L}(S_{t};\theta^{t})}_\text{keep accuracy}+\underbrace{\beta\cdot\nabla_{\theta}\mathcal{L}_\text{MD}}_\text{keep imperceptibility}+
    \lambda\cdot\underbrace{\nabla_{\theta}(\nabla_{\theta}\mathcal{L}(S_{val};\theta^{T})^\top\nabla_{\theta}\mathcal{L}(\mathcal{P}(S_{T});\theta^{T}))}_\text{accumulating poisoning effects for the trigger batch}\right],
\end{align}
Intuitively, it is reasonable that the search space about generating a perturbation for adversarial examples may be easier to utilize for keeping the imperceptibility than constructing the adversarial poison samples in accumulative poisoning attacks or other delusive attacks. The above exploration verifies Memorization Discrepancy has significance in identifying accumulative poisoning attacks and also in increasing the difficulty of generating poison samples with satisfactory poisoning effects and better statistical unawareness.

\section{Further Discussion}




As for the underlying mechanism of Memorization Discrepancy, it has no special assumption on the types of poisoning generation but reflects the target-level discrepancy (i.e., the differences between poisoning target $\max\mathcal{L}(S,\theta)$ and the original target $\min\mathcal{L}(S, \theta)$) by exploring model dynamics. Memorization Discrepancy is a characteristic of poisoned behavior that can be considered in different defensive methods or detection strategies. We primarily focus on this problem set in our study since the delusive attack and its corresponding defense are important and of great interest in the related literature~\citep{newsome2006paragraph, fowl2021adversarial, pang2021accumulative}. One possible strategy to extend our work to different types of poisoning is to explore indications in the nature of the specific poisoning objective using model dynamics. However, since the poisons have distinct targets~\citep{fowl2021adversarial,geiping2021doesn,pang2021accumulative} and various different objectives, we would leave expanding our approaches to be one major work in future.


Here we also discuss the potential limitations of our work, there are two points that need to be improved in the future. First, as our work mainly focuses on defending against the accumulative poison attack on real-time data streaming, currently, there is a certain gap in generalizing our method to an offline setting (e.g., training with the poisoned samples from scratch). To be specific, utilizing the Memorization Discrepancy in other settings may require more improvement or adjustment. Second, regarding the proposed DSC, the current method still requires carefully checking the model dynamics to set the threshold P. The predefined threshold may increase the extra analytical workload for adopting the method in practice.

Regarding the future directions, there are two directions corresponding to previously discussed limitations. First, the learning dynamics revealed by the Memorization Discrepancy capture the relationship between the natural objective and the poisoning objective, which can be extended to or explored in other settings like the offline poisoning defense. Second, it can be found that all the current methods still suffer from performance degradation induced by the accumulative poisoning attack. Considering the practical and special scenarios, how to enhance the defense or detection method is also a worthwhile topic to explore further.

\end{document}